\PassOptionsToPackage{cleveref}{jmlrutils}
\PassOptionsToPackage{capitalise,nameinlink,noabbrev}{cleveref}
\documentclass[final]{alt2027} 
\usepackage{times}
\usepackage[T1]{fontenc}
\usepackage[utf8]{inputenc}
\usepackage{mathtools}
\usepackage{microtype}
\usepackage{booktabs,array}
\usepackage{enumitem}
\usepackage{needspace}
\hypersetup{
  pdftitle={The Cost of Replicability in Stochastic Bandits},
  pdfauthor={Anonymous authors},
  pdfsubject={Anonymous ALT 2027 submission},
  pdfkeywords={replicability, stochastic bandits, regret, successive elimination},
  hidelinks
}

\crefname{theorem}{Theorem}{Theorems}
\crefname{lemma}{Lemma}{Lemmas}
\crefname{corollary}{Corollary}{Corollaries}
\crefname{proposition}{Proposition}{Propositions}
\crefname{definition}{Definition}{Definitions}
\crefname{section}{Section}{Sections}
\crefname{appendix}{Appendix}{Appendices}
\crefname{algocf}{Algorithm}{Algorithms}
\crefname{table}{Table}{Tables}
\Crefname{algocf}{Algorithm}{Algorithms}

\usepackage[dvipsnames,table]{xcolor}
\usepackage{algorithm}
\usepackage{algorithmic}
\usepackage{hyperref}
\usepackage{newtxtext}
\makeatletter
\def\@email{\hfill\small\mdseries\upshape}
\renewcommand*{\Email}[1]{{\@email\href{mailto:#1}{\textcolor{black}{#1}}}}
\makeatother

\makeatletter
\long\def\@makecaption#1#2{\vskip 10pt\setbox\@tempboxa\hbox{#1: #2}\ifdim\wd\@tempboxa>\hsize{\leftskip\z@\rightskip\z@\parfillskip\z@ plus 1fil\noindent#1: #2\par}\else\hbox to\hsize{\hfil\box\@tempboxa\hfil}\fi}

\renewcommand\ps@jmlrtps{%
  \let\@mkboth\@gobbletwo
  \def\@oddhead{}%
  \let\@evenhead\@oddhead
  \def\@oddfoot{\hfill\small\rmfamily\thepage\hfill}%
  \let\@evenfoot\@oddfoot}
\makeatother
\editor{}

\usepackage{booktabs}

\hypersetup{
    colorlinks=true,       
    linkcolor=Red,         
    citecolor=blue,        
    filecolor=magenta,     
    urlcolor=Blue          
}

\newcommand{\beq}{\begin{equation}}
\newcommand{\eeq}{\end{equation}}

\renewcommand\P{\mathbb{P}}
\newcommand\I{\mathbb{I}}

\newcommand{\x}{\mathbf{x}}
\newcommand{\y}{\mathbf{y}}
\newcommand{\z}{\mathbf{z}}

\renewcommand{\I}{\mathbb{I}}

\newcommand{\cC}{{\cal C}}
\newcommand{\cH}{{\cal H}}

\newcommand{\cP}{{\cal P}}

\newcommand{\cS}{{\cal S}}

\newcommand{\cF}{{\cal F}}

\newcommand{\cA}{{\cal A}}
\newcommand{\cB}{{\cal B}}
\newcommand{\cD}{{\cal D}}

\newcommand{\cE}{{\cal E}}

\DeclareMathOperator{\tr}{Tr}

\newtheorem{innerdef}{Definition}[section]

\newtheorem{theo}{Theorem}[section]
\newtheorem{prop}{Proposition}[section]
\newtheorem{lemm}{Lemma}[section]

\newtheorem{asmp}{Assumption}
\newtheorem{defn}[innerdef]{Definition}

\newcommand {\commentout}[1] {}

\def\ints{{{\rm Z} \kern -.35em {\rm Z} }}  
\def\smallints{{{\rm Z} \kern -.3em {\rm Z} }}  
\def\pints{{{\rm I} \kern -.15em {\rm N} }}      
\newcommand{\reals}{\mathbb R}

\def\cplx{{{\rm I} \kern -.45em {\rm C} }}       
\def\l2{\rm {\mathcal L}^{2}(\reals)}            

\newcommand{\be}{\begin{eqnarray}}
\newcommand{\ee}{\end{eqnarray}}
\newcommand{\bea}{\begin{eqnarray}}
\newcommand{\eea}{\end{eqnarray}}
\newcommand{\beaa}{\begin{eqnarray*}}
\newcommand{\eeaa}{\end{eqnarray*}}
\newcommand{\bnad}{\begin{nad}}
\newcommand{\enad}{\end{nad}}

\newcommand{\cO}{\mathcal{O}}

\newcommand{\IGNORE}[1]{}

\newcommand{\Reg}{\mathtt{Reg}}

{\begin{array}{l@{\hspace{8mm}}l@{\hspace{8mm}}l}}%
{\end{array}}
\newlength{\lplb}
\DeclareMathSymbol{\mhyphen}{\mathord}{AMSa}{"39}

\newcommand{\MaxInPwBC}{\ensuremath{\color{Green}\mathtt{MaxInP \mhyphen BC}}}

\newcommand{\BTIGWBC}{\ensuremath{\color{Green}\mathtt{BTIGW \mhyphen BC}}}

\newcommand{\Ncons}{N_{\mathrm{cons}}}

\providecommand{\dif}{\mathrm{d}}
\providecommand{\Ss}{\mathbf{S}}
\providecommand{\dd}{\mathbf{u}}
\providecommand{\I}{\mathbb{I}}

\usepackage{amsmath,amsfonts,bm}

\def\1{\bm{1}}

\DeclareMathAlphabet{\mathsfit}{\encodingdefault}{\sfdefault}{m}{sl}
\SetMathAlphabet{\mathsfit}{bold}{\encodingdefault}{\sfdefault}{bx}{n}

\newcommand{\E}{\mathbb{E}}

\newcommand{\R}{\mathbb{R}}

\DeclareMathOperator*{\argmax}{arg\,max}
\DeclareMathOperator*{\argmin}{arg\,min}

\title[Conservative Dueling Bandits: Preference Learning under Baseline Constraints]{ Conservative Dueling Bandits: \\
Preference Learning under Baseline Constraints}

\author[Deb \, Ghavamzadeh \, Banerjee]{%
  \Name{Rohan Deb} \Email{rd22@illinois.edu}\\
  \addr University of Illinois, Urbana Champaign
  \AND
  \Name{Mohammad Ghavamzadeh} \Email{mohammad.ghavamzadeh51@gmail.com}\\
  \addr Qualcomm AI Research
  \AND
  \Name{Arindam Banerjee} \Email{arindamb@illinois.edu}\\
  \addr University of Illinois, Urbana Champaign%
}

\begin{document}
\maketitle

\begin{abstract}
Contextual dueling bandits learn from pairwise comparisons, the same feedback used to fine-tune language models from human preferences. Standard algorithms explore aggressively and can show users two poor options. We study conservative contextual dueling bandits, which require that, at every round, the cumulative score of the better displayed item stay above a $(1-\alpha)$ fraction of a trusted baseline's. An item's score is its probability of beating the best available item, so the constraint depends only on utility differences and needs no absolute rewards. This also changes how safety can be certified: we show that transferring the safety test of the conservative bandit algorithm CLUCB unchanged can incur linear regret, even when the baseline's utilities are known exactly. For linear utilities, we develop $\MaxInPwBC$, which certifies scores through utility gaps and, when it must fall back, pairs the baseline with an informative partner so that fallback rounds still drive learning. It is safe with high probability at all times, with regret $\tilde{\mathcal O}(d\sqrt{T}/\kappa)$ plus an additive cost of conservatism that does not grow with the horizon. For general utility classes, we develop $\BTIGWBC$, which turns the guarantee of an online logistic-regression oracle into cumulative safety certificates. It is safe at all times, makes finitely many fallbacks, and has sublinear regret whenever the oracle's logistic regret is sublinear. In experiments, both satisfy the constraint, which their unconstrained counterparts often violate, at a moderate cost in regret.
\end{abstract}

\section{Introduction}
\label{sec:intro}
Many learning systems that interact with people cannot rely on numerical rewards. When users are asked to rate a search result, a recommended item, or a piece of generated text on an absolute scale, the scores they give are noisy, poorly calibrated across individuals, and expensive to collect at scale. Asking the same users which of two options they prefer is easier for them and yields a more consistent signal \citep{jones2011comparisons}. This has motivated a large body of work on learning from pairwise preferences \citep{BradleyTerry,furnkranz2010preference,yue2012karmed,wirth2017survey,bengs2021preference}. Its applications include online evaluation and optimization of search rankers \citep{joachims2002optimizing,radlinski2008does,yue2009interactively,chapelle2012large,hofmann2016online}, recommender systems \citep{jones2011comparisons,kalloori2016pairwise}, human-in-the-loop design \citep{brochu2007active,gonzalez2017preferential,tucker2020preference}, and learning from human comparisons \citep{christiano2017deep}, which has more recently become a standard way to fine-tune language models, either by reinforcement learning against a learned reward model \citep{ziegler2019fine,stiennon2020learning,ouyang2022training,bai2022training} or by optimizing the policy directly on the preference data \citep{rafailov2023direct}.

The dueling bandit problem gives a clean abstraction of this interaction and has been studied for finitely many actions \citep{yue2012karmed}. In each round, the learner selects a pair of actions, presents them to a user, and observes only a noisy binary outcome indicating which one was preferred. Because no single observation reveals how good an action is, the learner must infer quality from patterns of wins and losses. A common way to add structure is to posit a probabilistic choice model. Under the Bradley--Terry--Luce (BTL) model \citep{BradleyTerry,hunter2004mm,luce2005individual}, each action has a latent utility, and the probability that one action beats another is a logistic function of their utility difference (cf. Assumption~\ref{asmp:BTL-model}). The Plackett--Luce model extends this to comparisons among more than two items \citep{soufiani2014plackett}. Contextual dueling bandits \citep{saha2021maxinp} let the utility also depend on observed side information, such as a query, a user profile, or a prompt. They have led to algorithms with sublinear regret guarantees for linear and more general utility classes, the latter typically through reductions to online regression oracles \citep{sekhari2023aurora}, an approach that extends to neural function classes in contextual bandits with cardinal rewards \citep{deb2024contextual}. Constraints are far less understood under preference feedback. For dueling bandits with knapsack constraints, \citet{pmlr-v238-deb24a} show that relative feedback makes the problem harder than its bandit counterpart: without further assumptions on the budget, it is not learnable from a regret-minimization perspective.

However, all of this work measures performance by regret against the best comparison in hindsight. The resulting algorithms therefore explore as aggressively as they can. This sits uneasily with how preference-based learners are deployed in practice, because every comparison is shown to a real person. A user who is repeatedly offered two unappealing recommendations, or two unhelpful answers, has little reason to keep trusting the system, however informative those comparisons are for the learner. Moreover, a new learning system rarely starts from scratch. It usually replaces or augments an existing policy whose behavior is considered acceptable. In such settings, it is natural to require that algorithms avoid aggressive exploration, and ensure that the deployed policy performs at least comparably to an existing, trusted \emph{baseline strategy}. Concretely, we ask that at every step, the better of the two displayed items be, in aggregate, at least a $(1-\alpha)$ fraction as good as what the baseline would have offered.

For bandits with numerical rewards, this kind of safe exploration is captured by conservative bandits \citep{conservativeMAB,conservative_linear_bandits,Garcelon2020IA}. Conservative algorithms are now available for linear rewards under both optimism and Thompson sampling \citep{conservative_linear_bandits,deb2025thompson}, and for general reward classes accessed through online regression oracles \citep{deb2025conservative}. These methods keep a budget of permissible loss relative to the baseline: they lower-bound the cumulative reward of the actions played so far, credit the baseline's reward on rounds where it is played, and explore only when this bound stays above the target. That bookkeeping requires knowing how good the learner's actions are in absolute terms, and preference feedback does not provide it. A comparison reveals only the difference between two utilities, so a confidence set built from comparisons can leave the overall level of utility uncertain no matter how much data arrives. A lower bound on cumulative utility computed from such a set may then stay loose indefinitely. Indeed, we show that transferring CLUCB's safety test unchanged can incur linear regret, even when the baseline's utilities are known exactly. As with knapsack constraints \citep{pmlr-v238-deb24a}, relative feedback changes what a constrained learner can certify. This shapes both the constraint and how it is enforced. We measure each item by its probability of beating the best available item, which depends only on utility differences, and we certify safety through utility gaps rather than estimated payoffs. A second difficulty arises when the budget runs out. A conservative bandit then plays the baseline, but the dueling analogue, showing the baseline against itself, yields a fair coin flip and no information. A learner that falls back this way stops learning exactly when it is uncertain enough to need the fallback. Because our constraint scores a pair by its better item, any pair containing the baseline is safe whatever its partner. We therefore pair the baseline with an informative partner, so that fallback rounds still sharpen the certificates that govern future exploration.

In this work, we study \emph{Preference Learning under Baseline Constraints}. Our technical contributions are as summarized below.

\begin{enumerate}
    \item \textbf{Framework:} We formulate conservative contextual dueling bandits under Bradley--Terry feedback, requiring the cumulative preference score of the better displayed item to remain above a $(1-\alpha)$ fraction of the baseline's at every prefix. This score-based constraint is invariant to utility shifts and requires no absolute reward observations.

    \item \textbf{Linear utilities:} We show that a direct CLUCB transfer can incur linear regret and develop \MaxInPwBC{}, which certifies scores through utility gaps and uses informative baseline fallbacks. We prove anytime high-probability safety and regret $\tilde{\cO}\big(d\sqrt{T}/\kappa+\Delta_h d^2/(\kappa^2\alpha^2\underline{s}^{\,2})\big)$, where $\underline{s}=\sigma(-\Delta_h)$. The additive conservatism term is independent of the horizon, and known score bounds suffice in place of exact baseline scores.

    \item \textbf{General utility functions:} We develop \BTIGWBC{}, which converts online logistic-regression guarantees into cumulative score certificates and updates the oracle on both candidate and fallback comparisons. For realizable classes with oracle regret $\mathcal R_{\log}(T)=o(T)$, we prove anytime high-probability safety and regret $\tilde{\cO}\big(\sqrt{KT[\mathcal R_{\log}(T)/\kappa+1/\kappa^2]}+\Delta_hK[\mathcal R_{\log}(T)/\kappa+1/\kappa^2]/(\alpha^2\underline{s}^{\,2})\big)$.
    \item \textbf{Experiments:} On synthetic linear and general-utility instances, \MaxInPwBC{} and \BTIGWBC{} empirically satisfy the baseline constraint at every round. Their unconstrained counterparts violate it in early rounds in nearly all settings. Conservatism costs some extra regret over the unconstrained learners, concentrated where the constraint is tight, but far less than the linear regret of always playing the baseline.
\end{enumerate}

\section{Problem Formulation}
\label{sec:prelim}
We consider a stochastic \emph{Contextual Dueling Bandit} (CDB) problem. At each round $t \in [T]$, the learner is presented with a context set $\cS_t = \{\x_{t,1}, \ldots, \x_{t,K}\} \subset \cD$ of $K$ items (or arms), where each $\x_{t,i}$ is drawn i.i.d.\ from an unknown distribution $\cP_{\cD}$ over a decision space $\cD \subset \R^d$. The agent selects an \emph{ordered} pair $(\x_t, \y_t) \in \cS_t \times \cS_t$ (repetitions are allowed), and receives noisy binary feedback $o_t \in \{0,1\}$ indicating which item is preferred. In the context of LLMs, $(\x_t,\y_t)$ correspond to the two responses shown to the user, and $o_t$ to the user's recorded preference. Let $\cH_t^{-}$ denote all information available immediately before observing $o_t$, including all completed rounds, the current contexts and baseline information, and the selected pair. As in \cite{saha2021maxinp, bengs2022stochastic,dueling_feel_good}, we assume preferences are generated by a latent utility function $f : \cD \to \R$, through the Bradley--Terry--Luce (BTL) model.

\begin{asmp}[\textbf{BTL Model}]
\label{asmp:BTL-model}
    There exists a utility function $f:\cD\to\R$ such that the conditional preference probability satisfies
    \begin{align}
        \P(o_t=1\mid\cH_t^{-})
        &=
        \frac{\exp(f(\x_t))}
        {\exp(f(\x_t))+\exp(f(\y_t))}
        =
        \sigma\big(f(\x_t)-f(\y_t)\big),
        \label{eq:BTL-model}
    \end{align}
    where $\sigma(z)=1/(1+e^{-z})$ is the sigmoid function.
\end{asmp}
Two structural properties of \eqref{eq:BTL-model} will be useful. The comparison law depends only on utility differences and is invariant to adding a constant to every utility. A self-duel has success probability $1/2$, so its outcome provides no information about $f$.

\begin{asmp}[\textbf{Bounded utility range}]
\label{asmp:bounded}
    There is a known constant $B>0$ such that the absolute difference $|f(\x)-f(\y)| \le B$ for all $\x,\y \in \cD$. In the linear case $f(\x)=\langle\theta^\star,\x\rangle$ with $\|\theta^\star\|_2\le S$ and $\|\x\|_2\le L$, one may take $B=2LS$. We define the minimum slope of the logistic link on $[-B,B]$, $$\kappa := \inf_{|z|\le B}\sigma'(z) = \sigma(B)\big(1-\sigma(B)\big) > 0.$$

\end{asmp}

Let $\x_t^\star \in \argmax_{\x \in \cS_t} f(\x)$ denote a best item in round $t$, ties broken by a fixed rule. Next we define the regret of an algorithm following \cite{saha2021maxinp,verma2025neural}. It captures the cumulative gap between the best possible utility and the average utility of the learner's choices over $T$ rounds.

\textbf{Regret.} Suppose an algorithm chooses the pair $(\x_t,\y_t)$ at time $t\in[T]$. The instantaneous regret incurred at round $t$ is the gap between the utility of the best item and the average utility of the compared pair, denoted $R_t := f(\x_t^\star) - \tfrac{1}{2} \big( f(\x_t) + f(\y_t) \big)$, and the cumulative $T$-round regret is
\begin{align}
    \Reg(T) := \sum_{t=1}^T R_t = \sum_{t=1}^T \left( f(\x_t^\star) - \tfrac{1}{2} \big( f(\x_t) + f(\y_t) \big) \right).
    \label{eq:regret}
\end{align}
Next we define a performance constraint that the algorithm must satisfy so that the overall user experience stays above a threshold set by an existing trusted baseline policy.

\begin{defn}[\textbf{Baseline Performance Constraint}]
\label{def:performance}
    The \emph{preference score} of an item $\x \in \cS_t$ is its probability of beating the best item of the round: $s_t(\x) := \P\{\x \succ \x_t^\star\} = \sigma\big(f(\x)-f(\x_t^\star)\big) \in (0,\tfrac12].$
    Suppose the agent has access to baseline items $\x_{b_t} \in \cS_t$, $t\in[T]$, with known scores $s_t(\x_{b_t})$. Given a risk parameter $\alpha \in (0,1)$, a sequence of displayed pairs $(\x_i,\y_i), i \in [t]$ is \emph{$\alpha$-conservative} if
    \begin{align}
        \sum_{i=1}^t \max \left\{ s_i(\x_i), s_i(\y_i) \right\} \geq (1 - \alpha) \sum_{i=1}^t s_i(\x_{b_i}), \quad \text{for all } t \in [T].
        \label{eq:performance}
    \end{align}
\end{defn}
\begin{asmp}[\textbf{Baseline Gap and Bounds}]
\label{asmp:gapBounds}
    Let $\Delta_{t,b_t} := f(\x_t^\star) - f(\x_{b_t})$ denote the baseline gap. There are known constants $0\le \Delta_l \le \Delta_h \le B$ with $\Delta_{l} \leq \Delta_{t,b_t} \leq \Delta_h$ for all $t\in[T]$; equivalently
    \begin{align*}
        \underline{s} := \sigma(-\Delta_h) \;\le\; s_t(\x_{b_t}) \;\le\; \sigma(-\Delta_l) =: \bar{s} \;\le\; \tfrac12 .
    \end{align*}
\end{asmp}

\begin{remark}[\textbf{Why the maximum and why scores}]
    The constraint in \eqref{eq:performance} compares cumulative preference probabilities along the learner's trajectory, rather than realized comparison labels. Using the maximum ensures that any pair containing the baseline has score at least that of the baseline, allowing its partner to be selected for information acquisition.
\end{remark}

\section{Linear Utilities: Conservative Maximum-Informative-Pair}
\label{sec:linear}
\begin{asmp}[\textbf{Linear Utility}]
\label{asmp:linear}
    There is an unknown $\theta^\star\in\R^d$ with $\|\theta^\star\|_2\le S$ such that $f(\x)=\langle\theta^\star,\x\rangle$, and $\|\x\|_2\le L$, for every $\x\in\cD$. The learner optimises over $\Theta=\{\theta:\|\theta\|_2\le S\}$, so Assumption~\ref{asmp:bounded} holds with $B=2LS$ and $\kappa=\sigma(B)(1-\sigma(B))$.
\end{asmp}

In the unconstrained setting, \textsc{MaxInP} \citep{saha2021maxinp} maintains a set of items that may still be optimal and displays the most uncertain pair among them. Such a pair is informative, but nothing prevents both of its items from being poor. We therefore design \MaxInPwBC{} around a safety budget. The algorithm fits $\theta^\star$ by regularized maximum likelihood on the observed comparisons and uses the resulting confidence ellipsoid to identify promising items. It then proposes the most informative pair and displays it only if the baseline constraint remains certifiably satisfied. Otherwise, it falls back to a pair that contains the baseline item and is still informative for learning.

\textbf{Estimator.}  Each round the learner is given $\cS_t$, the baseline item $\x_{b_t}\in\cS_t$ and its score $s_t(\x_{b_t})$. The learner displays $(\x_t,\y_t)\in\cS_t\times\cS_t$, and observes $o_t$. Write $\z_t:=\x_t-\y_t$ for the displayed comparison vector, so $\|\z_t\|_2\le 2L$. We define
\begin{align}
    V_t &= \lambda I_d + \sum_{i<t}\z_i\z_i^\top,
    \\
    \hat\theta_t &\in \argmin_{\theta\in\Theta}\Big\{\sum_{i<t}\big[\log(1+e^{\z_i^\top\theta}) - o_i\z_i^\top\theta\big] + \tfrac\lambda2\|\theta\|_2^2\Big\},
    \label{eq:design}
\end{align}
with $\lambda>0$; the objective is strongly convex, so $\hat\theta_t$ is unique. A self-duel has $\z_t=0$ and contributes the constant $\log 2$ to the objective, leaving the design matrix and the minimizer unchanged.
For $\delta\in(0,1)$, set

\begin{align}
    \beta_t = \frac1\kappa\Big[\tfrac12\sqrt{d\log\big(1+\tfrac{4L^2(t-1)}{\lambda d}\big) + 2\log\tfrac1\delta} + \sqrt\lambda S\Big].
    \label{eq:beta}
\end{align}

Both the item selection and the safety certificate rely on $\theta^\star$ lying in the ellipsoid around $\hat\theta_t$ at every round. The following lemma guarantees this with high probability, uniformly over $t$.
\begin{lemm}[\textbf{Anytime confidence ellipsoid}]
\label{lem:ellipsoid}
    Under Assumptions~\ref{asmp:BTL-model} and \ref{asmp:linear}, $\P(\cE)\ge1-\delta$ for the event $\;\cE:=\{\|\hat\theta_t-\theta^\star\|_{V_t}\le\beta_t, \;\; \forall t\ge1\}$.
\end{lemm}

\emph{Proof sketch.} The residuals $o_t-\sigma(\z_t^\top\theta^\star)$ are conditionally $\tfrac12$-sub-Gaussian, so a mixture supermartingale with Ville's inequality gives an anytime self-normalised bound. First-order optimality of $\hat\theta_t$ and the curvature bound $\nabla^2\mathcal{L}_t\succeq\kappa V_t$ turn this into the stated parameter concentration (Appendix~\ref{app:ellipsoid}).

\textbf{Promising items and the regret certificate.} On $\cE$, we can discard any item that is certifiably worse than another, and bound the regret of the items that remain. Let
\begin{align}
    \cC_t = \big\{\x\in\cS_t : \langle\hat\theta_t,\x-\y\rangle + \beta_t\|\x-\y\|_{V_t^{-1}} \ge 0, \;\; \forall\y\in\cS_t\big\},
    \label{eq:promising}
\end{align}
and let the \emph{candidate pair} be the maximum-width pair of promising items,
\begin{align}
    (\tilde\x_t,\tilde\y_t) &\in\argmax_{\x,\y\in\cC_t}\|\x-\y\|_{V_t^{-1}},
    \\ w_t &= \|\tilde\x_t-\tilde\y_t\|_{V_t^{-1}},
    \\ c_t &= \min\{B,\,2\beta_t w_t\}.
    \label{eq:candidate}
\end{align}

Since $\x_t^\star$ is never eliminated on $\cE$, every promising item is close to it in the $V_t^{-1}$ geometry. The following lemma makes this precise: the width $w_t$ bounds the regret of any promising item, and in particular that of the candidate pair.
\begin{lemm}[\textbf{Regret certificate}]
\label{lem:certificate}
    $\cC_t\neq\emptyset$ for all $t$ and on $\cE$, $\x_t^\star\in\cC_t$ and $f(\x_t^\star)-f(\y)\le c_t$ for every $\y\in\cC_t$. In particular, the candidate pair has regret
    \[
        R_t^{\mathrm{cand}}
        := f(\x_t^\star)
        - \tfrac12\big(f(\tilde\x_t)+f(\tilde\y_t)\big)
        \le c_t.
    \]
\end{lemm}

We next convert this utility-gap certificate into a lower bound on the score performance in \eqref{eq:performance}.

\begin{lemm}[\textbf{Score certificate}]
\label{lem:score-certificate}
    If a pair $(\x,\y)$ satisfies $f(\x_t^\star)-\tfrac12(f(\x)+f(\y))\le c$ for some $c\ge0$, then $\max\{s_t(\x),s_t(\y)\}\ge\sigma(-c)\ge\tfrac12-\tfrac c4$.
\end{lemm}

Indeed, with $g_\x:=f(\x_t^\star)-f(\x)\ge0$ and $g_\y$ likewise, $\max\{s_t(\x),s_t(\y)\}=\sigma(-\min\{g_\x,g_\y\})$ and $\min\{g_\x,g_\y\}\le\tfrac12(g_\x+g_\y)\le c$; the second inequality is $\sigma'\le\tfrac14$. The anchor is $s_t(\x_t^\star)=\tfrac12$, available without estimating any utility.

\textbf{Informative fallback.} When the safety test rejects the candidate, we retain the baseline item and select the promising partner with the largest comparison width:
\begin{align}
    &\y_t^{\mathrm{fb}}
    \in
    \argmax_{\y\in\cC_t}
    \|\x_{b_t}-\y\|_{V_t^{-1}},
    \\
    &\text{displaying }(\x_{b_t},\y_t^{\mathrm{fb}}).
    \label{eq:fallback}
\end{align}
Including the baseline guarantees that $\max\{s_t(\x_{b_t}),s_t(\y_t^{\mathrm{fb}})\}\ge s_t(\x_{b_t})$. The maximizing choice of partner also connects the displayed comparison to the candidate width. Since $\tilde\x_t,\tilde\y_t\in\cC_t$, the triangle inequality gives
\begin{align}
    w_t
    &\le
    \|\tilde\x_t-\x_{b_t}\|_{V_t^{-1}}
    +
    \|\x_{b_t}-\tilde\y_t\|_{V_t^{-1}}
    \le
    2\|\x_{b_t}-\y_t^{\mathrm{fb}}\|_{V_t^{-1}}.
    \label{eq:width-transfer}
\end{align}
Thus, the displayed comparison retains at least half the candidate width, whether or not $\x_{b_t}\in\cC_t$. Note that because the fallback pair contains the baseline, its one-round surplus under \eqref{eq:performance} is at least $\alpha s_t(\x_{b_t})$, regardless of its partner. The partner can therefore be chosen for information, and by \eqref{eq:width-transfer} it retains at least half the candidate width. Its regret is also controlled: since $\y_t^{\mathrm{fb}}\in\cC_t$, on $\cE$,
    \[
        R_t
        =
        \tfrac12\Big(
            \Delta_{t,b_t}
            +f(\x_t^\star)-f(\y_t^{\mathrm{fb}})
        \Big)
        \le
        \tfrac12(\Delta_h+c_t).
    \]
    The baseline contributes at most $\Delta_h/2$, and the partner's contribution is controlled by the same certificate $c_t$ used for candidate decisions.


\textbf{Safety budget.} The resulting algorithm, \MaxInPwBC{}, is given in Algorithm~\ref{alg:cmaxinp}. It maintains a computable quantity $\Lambda_t$, initialized at $\Lambda_0=0$, that lower-bounds the cumulative safety surplus on $\cE$. By Lemma~\ref{lem:score-certificate}, the candidate increment is $\tfrac12-\tfrac{c_t}{4}-(1-\alpha)s_t(\x_{b_t})$. A fallback is credited with $\alpha s_t(\x_{b_t})$, which lower-bounds its actual surplus because the displayed pair contains the baseline. Each update is determined by the history, current contexts, baseline score, and chosen pair before the comparison outcome is observed. The candidate is accepted exactly when its proposed update leaves $\Lambda_t$ nonnegative; otherwise, the fallback preserves nonnegativity. On $\cE$, these per-round lower bounds therefore imply the cumulative safety constraint.
\begin{algorithm}[t]
\caption{\MaxInPwBC{} (Conservative Maximum-Informative-Pair)}
\label{alg:cmaxinp}
\begin{algorithmic}[1]
\REQUIRE risk level $\alpha$, confidence $\delta$, regularisation $\lambda$, bounds $S,L$
\STATE Initialise $\Lambda_0\gets0$, $V_1\gets\lambda I_d$
\FOR{$t=1,2,\dots$}
  \STATE Observe $\cS_t$, baseline $\x_{b_t}$ and its known score $s_t(\x_{b_t})$; compute $\hat\theta_t,\beta_t$ by \eqref{eq:design}--\eqref{eq:beta}
  \STATE Form $\cC_t$ by \eqref{eq:promising}; compute $(\tilde\x_t,\tilde\y_t)$, $w_t$, $c_t$ by \eqref{eq:candidate}
  \IF{$\Lambda_{t-1} + \tfrac12 - \tfrac{c_t}{4} - (1-\alpha)s_t(\x_{b_t}) \ge 0$}
     \STATE Display $(\x_t,\y_t)\gets(\tilde\x_t,\tilde\y_t)$ and set $\Lambda_t\gets\Lambda_{t-1}+\tfrac12-\tfrac{c_t}4-(1-\alpha)s_t(\x_{b_t})$ \COMMENT{candidate}
  \ELSE
     \STATE Display $(\x_t,\y_t)\gets(\x_{b_t},\y_t^{\mathrm{fb}})$ from \eqref{eq:fallback} and set $\Lambda_t\gets\Lambda_{t-1}+\alpha s_t(\x_{b_t})$ \COMMENT{fallback}
  \ENDIF
  \STATE Observe $o_t$; set $\z_t\gets\x_t-\y_t$ and $V_{t+1}\gets V_t+\z_t\z_t^\top$
\ENDFOR
\vspace{2ex}
\end{algorithmic}
\end{algorithm}

Let $\cA_T,\cB_T$ be the candidate and fallback rounds up to $T$, with $M_T=|\cA_T|$ and $N_T=|\cB_T|$. Since $\Lambda_{t-1}\ge0$ by construction and $s_t(\x_{b_t})\le\tfrac12$, the candidate increment is at least $\tfrac\alpha2-\tfrac{c_t}{4}$. Consequently, every round with $c_t\le2\alpha$ is accepted, and any rejection requires $c_t>2\alpha$. This criterion applies separately at each round: the certificate can change with the contexts, estimates, and confidence radius, so acceptance at one round does not imply acceptance at all subsequent rounds.

\textbf{Main result.} With $\underline{s}:=\sigma(-\Delta_h)$ and $\bar\beta_t:=\max\{B/2,\beta_t\}$, define certificate envelope and the fallback bound
\begin{align}
    G(t)
    &:=
    4\bar\beta_t
    \sqrt{
        2dt\log\Big(1+\frac{4L^2t}{\lambda d}\Big)
    },
    \\
    \Ncons
    &:=
    \sup\big\{
        t\in\mathbb N_{\ge1}:
        4\alpha\underline{s}\,t<G(t)
    \big\},
    \label{eq:envelope}
\end{align}
with $G(0)=0$ and $\sup\emptyset:=0$. The function $G$ is nondecreasing. For fixed problem parameters and confidence level, $G(t)=O(\sqrt t\log(et))=o(t)$, so $\Ncons<\infty$. 

\begin{theo}[\textbf{Anytime safety and regret}]
\label{thm:linear}
    Let Assumptions~\ref{asmp:BTL-model}, \ref{asmp:bounded}, \ref{asmp:gapBounds}, and \ref{asmp:linear} hold, and run Algorithm~\ref{alg:cmaxinp} with $\alpha,\delta\in(0,1)$ and $\lambda>0$. With probability at least $1-\delta$, simultaneously for every $T\ge1$, the displayed pairs are $\alpha$-conservative in the sense of \eqref{eq:performance}, the fallback count satisfies $N_T\le\min\{T,\Ncons\}$, and $\Reg(T) \le G(T)+\tfrac12\Delta_h\Ncons$.
    When $L,S,\lambda$ are treated as fixed constants, suppressing logarithmic factors in the remaining parameters gives
    \begin{align}
        \Reg(T)
        =
        \tilde\cO\left(
            \frac{d\sqrt T}{\kappa}
            +
            \frac{\Delta_h d^2}
            {\kappa^2\alpha^2\underline{s}^{\,2}}
        \right).
        \label{eq:regret_linear}
    \end{align}
\end{theo}

\emph{Proof sketch.} Work on $\cE$, which has probability at least $1-\delta$ by Lemma~\ref{lem:ellipsoid}; all guarantees hold simultaneously on this event.

\begin{enumerate}

\item[(a)] \textbf{Envelope.} On a candidate round, $c_t=\min\{B,2\beta_t\|\z_t\|_{V_t^{-1}}\}$. On a fallback round, \eqref{eq:width-transfer} gives $c_t\le\min\{B,4\beta_t\|\z_t\|_{V_t^{-1}}\}$. Thus, in both cases,
\(
    c_t
    \le
    4\bar\beta_t
    \min\{1,\|\z_t\|_{V_t^{-1}}\}.
\)
Cauchy--Schwarz and the elliptical-potential bound
\(
    \sum_{i=1}^{t}
    \min\{1,\|\z_i\|_{V_i^{-1}}^2\}
    \le
    2d\log\Big(1+\frac{4L^2t}{\lambda d}\Big)
\)
yield $\sum_{i=1}^{t}c_i\le G(t)$. This includes the certificate computed on a rejected round, since the displayed fallback comparison controls that certificate through \eqref{eq:width-transfer}.

\item[(b)] \textbf{Safety.} On candidate rounds, Lemma~\ref{lem:score-certificate} shows that the true one-round surplus is at least the budget increment. On fallback rounds, including the baseline guarantees a true surplus of at least $\alpha s_t(\x_{b_t})$, again the budget increment. Induction therefore shows that $\Lambda_t$ lower-bounds the true cumulative surplus. Candidate updates preserve $\Lambda_t\ge0$ by the acceptance test, and fallback updates add a positive quantity, proving safety.

\item[(c)] \textbf{Fallbacks.} Expanding the failed acceptance test at a fallback round $\tau$, and using the fact that $\tfrac12-(1-\alpha)s_i(\x_{b_i})\ge\alpha/2$ and $\alpha s_i(\x_{b_i})\ge\alpha\underline{s}$, gives
\[
    \frac{\alpha}{2}(M_{\tau-1}+1)
    +
    \alpha\underline{s}N_{\tau-1}
    <
    \frac14\left(
        \sum_{i\in\cA_{\tau-1}}c_i+c_\tau
    \right)
    \le
    \frac14G(\tau).
\]
Because $\underline{s}\le1/2$ and $M_{\tau-1}+N_{\tau-1}+1=\tau$, the left side is at least $\alpha\underline{s}\tau$. Hence every fallback time satisfies $4\alpha\underline{s}\tau<G(\tau)$ and is at most $\Ncons$, proving $N_T\le\min\{T,\Ncons\}$.

\item[(d)] \textbf{Regret.} Lemma~\ref{lem:certificate} gives $R_t\le c_t$ on candidate rounds. On fallback rounds, $\y_t^{\mathrm{fb}}\in\cC_t$ gives $R_t\le(\Delta_h+c_t)/2$. Therefore,
\[
    \Reg(T)
    \le
    \sum_{t=1}^{T}c_t+\frac{\Delta_h}{2}N_T
    \le
    G(T)+\frac{\Delta_h}{2}\Ncons.
\]
For fixed $L,S,\lambda$, the simplified rate follows from $\bar\beta_t=\tilde\cO(\sqrt d/\kappa)$ and the defining inequality for $\Ncons$. Details are given in Appendix~\ref{app:linear-proofs}. \hfill$\square$

\end{enumerate}

\begin{remark}[\textbf{An additive cost of conservatism}]
    The leading term in the regret bound in \eqref{eq:regret_linear} has the same dependence on $d$ and $T$ as unconstrained \textsc{MaxInP} \citep{saha2021maxinp}, while the additional term $\Delta_h\Ncons/2$ is independent of $T$ for fixed parameters. 
\end{remark}
\begin{remark}[\textbf{Why not transfer CLUCB directly}]
\label{rem:clucb}
    For bandits with numerical rewards, CLUCB \citep{conservative_linear_bandits} certifies safety by lower-bounding cumulative candidate rewards with a confidence set, and it credits known baseline rewards separately. Under preference feedback, the design matrix in \eqref{eq:design} accumulates comparison vectors $\z_t=\x_t-\y_t$, whereas average utilities depend on $(\x_t+\y_t)/2$. A comparison-based confidence set can therefore retain uncertainty in a common utility component that comparisons never reveal. Proposition~\ref{prop:obstruction} (Appendix~\ref{app:obstruction}) gives an i.i.d.\ two-dimensional instance on which this transfer incurs $\Reg(T)=\Omega(T)$ with high probability, even with exact baseline utilities and arbitrary candidate selection. The obstruction concerns the unchanged certification rule: known baseline utilities would also permit a baseline-centered cardinal certificate. \MaxInPwBC{} instead certifies scores through utility gaps (Lemmas~\ref{lem:certificate} and~\ref{lem:score-certificate}), which depend only on pairwise differences. Through \eqref{eq:width-transfer}, it also ties the uncertainty in this certificate to the comparisons displayed on both candidate and fallback rounds.
\end{remark}

\begin{remark}[\textbf{Relaxing baseline-score knowledge}]
\label{rem:baseline-score-bounds}
    For ease of presentation, we assume that baseline scores are known exactly. Appendix~\ref{app:baseline-score-bounds} relaxes this assumption by modifying the safety budget to use lower and upper score bounds together with the known range $[\underline{s},1/2]$. 
\end{remark}

\section{General Utility Functions}
\label{sec:general}

Linear utilities can be restrictive when items are described by rich features. We now let $f$ lie in a general class $\mathcal F$, accessed only through an online logistic-regression oracle. In the unconstrained setting, \citet{sekhari2023aurora} sample comparisons from an inverse-gap-weighted (IGW) distribution built from the oracle's predictions. Without a confidence ellipsoid, we cannot certify each candidate individually as in Lemma~\ref{lem:certificate}. We therefore develop \BTIGWBC{} (Algorithm~\ref{alg:gf-igw}) and certify safety cumulatively: the oracle guarantee yields a computable envelope on total regret, and a candidate is displayed only if the accumulated score credits cover it. Otherwise, the algorithm falls back to a baseline-containing pair whose partner is still drawn by IGW.

\begin{asmp}[\textbf{Realizability and logistic oracle}]
\label{asmp:gf-oracle}
    The utility $f$ belongs to a class $\mathcal F$ whose functions satisfy $|g(\x)-g(\y)|\le B$ on $\cD$. Before its $m$-th comparison $(\x_m,\y_m)$ is drawn, the oracle outputs $\hat f_m$ with the same range bound; $\hat f_m$ may use the current item set and need not lie in $\mathcal F$. With $\ell_m(g)=\log(1+e^{g(\x_m)-g(\y_m)})-o_m[g(\x_m)-g(\y_m)]$, every fresh instance satisfies, on every adaptive sequence and for all $n\ge1$,
    \begin{align}
        \sum_{m=1}^n\ell_m(\hat f_m)-\inf_{g\in\mathcal F}\sum_{m=1}^n\ell_m(g)\le\mathcal R_{\log}(n),
        \label{eq:gf-oracle}
    \end{align}
    where $\mathcal R_{\log}$ is known, nonnegative, nondecreasing, and $o(n)$.
\end{asmp}

\textbf{IGW with a fixed partner.} Given $\hat f_t$, let $\hat a_t\in\argmax_a\hat f_t(\x_{t,a})$ and
\begin{align}
    p_t(a)&=\frac{1}{K+\gamma_t[\hat f_t(\x_{t,\hat a_t})-\hat f_t(\x_{t,a})]}\quad(a\ne\hat a_t),
    \\
    p_t(\hat a_t)&=1-\sum_{a\ne\hat a_t}p_t(a).
    \label{eq:gf-igw}
\end{align}
Write $\Delta_t(\x)=f(\x_t^\star)-f(\x)$ and $\varepsilon_t(\x,\y)=(\hat f_t-f)(\x)-(\hat f_t-f)(\y)$ for the margin error.
\begin{lemm}[\textbf{IGW with a fixed partner}]
\label{lem:gf-anchor}
    For every $\gamma_t>0$ and $\mathbf u\in\cS_t$, the following holds
    \[
        \mathbb E_{a\sim p_t}\Delta_t(\x_{t,a})
        \le \frac{2K}{\gamma_t}+\frac{\gamma_t}{4}\mathbb E_{a\sim p_t}\varepsilon_t(\x_{t,a},\mathbf u)^2.
    \]
\end{lemm}
This is the standard IGW bound 
applied after subtracting the values at $\mathbf u$. The free partner lets the error be measured on the displayed pair: for a candidate $(\x_{t,I_t},\x_{t,J_t})$ with $I_t,J_t\sim p_t$ independent, take $\mathbf u$ to be the other item; for a fallback $(\x_{b_t},\x_{t,J_t})$, take $\mathbf u=\x_{b_t}$. This plays the role of \eqref{eq:width-transfer}: the fallback partner's gap is controlled exactly as a candidate's regret.

\begin{algorithm}[t]
\caption{\BTIGWBC{} (Bradley--Terry Inverse-Gap-Weighting Baseline-Conservative)}
\label{alg:gf-igw}
\begin{algorithmic}[1]
\REQUIRE $\alpha,\delta\in(0,1)$, $B$, $\mathcal R_{\log}$, and an oracle satisfying Assumption~\ref{asmp:gf-oracle}
\STATE Initialise $Q_0\gets0$
\FOR{$t=1,2,\ldots$}
  \STATE If $t=2^e$, restart the oracle and set $\iota_e,E_e,C_e$ by \eqref{eq:gf-parameters}
  \STATE Observe $\cS_t$, $\x_{b_t}$, $s_t(\x_{b_t})$; obtain $\hat f_t$; form $p_t$ with $\gamma_t$; set $h_t\gets\tfrac12-(1-\alpha)s_t(\x_{b_t})$
  \IF{$Q_{t-1}+h_t\ge G_{\mathrm{IGW}}(t)/4$}
     \STATE Draw $I_t,J_t\sim p_t$ independently; display $(\x_{t,I_t},\x_{t,J_t})$; $Q_t\gets Q_{t-1}+h_t$ \COMMENT{candidate}
  \ELSE
     \STATE Draw $J_t\sim p_t$; display $(\x_{b_t},\x_{t,J_t})$; $Q_t\gets Q_{t-1}+\alpha s_t(\x_{b_t})$ \COMMENT{fallback}
  \ENDIF
  \STATE Observe $o_t$ and update the oracle with this comparison
\ENDFOR
\vspace{1ex}
\end{algorithmic}
\end{algorithm}

\textbf{An anytime envelope.} We run the oracle in epochs $e=0,1,\ldots$ of length $n_e=2^e$, restarting it at each, with confidence $\delta_e=\delta/[(e+1)(e+2)]$. For local round $m$, set
\begin{align}
    \iota_e &=\log\frac{3n_e}{\delta_e},\quad
    E_e=\frac{8\mathcal R_{\log}(n_e)}{\kappa}+\frac{5\iota_e}{\kappa^2},\\
    \gamma_{e,m} &=\sqrt{\frac{8Km}{E_e}},\quad
    C_e=\frac{3}{\sqrt2}\sqrt{KE_e}+B\sqrt{2\iota_e}.
    \label{eq:gf-parameters}
\end{align}
At time $t$, let $e(t)=\lfloor\log_2t\rfloor$, $m(t)=t-2^{e(t)}+1$, $\gamma_t=\gamma_{e(t),m(t)}$, and $G_{\mathrm{IGW}}(t)=\sum_{j<e(t)}C_j\sqrt{2^j}+C_{e(t)}\sqrt{m(t)},
     G_{\mathrm{IGW}}(0)=0.$ Let $r_t^\circ=R_t$ on candidate rounds and $r_t^\circ=\Delta_t(\y_t)$ on fallback rounds.
\begin{lemm}[\textbf{Anytime regret envelope}]
\label{lem:gf-envelope}
    Under Assumptions~\ref{asmp:BTL-model}, \ref{asmp:bounded}, and \ref{asmp:gf-oracle}, with probability at least $1-\delta$, Algorithm~\ref{alg:gf-igw} satisfies $\sum_{i\le t}r_i^\circ\le G_{\mathrm{IGW}}(t)$ for all $t\ge1$. Moreover, with $$A(t):=\mathcal R_{\log}(t)/\kappa+1/\kappa^2, G_{\mathrm{IGW}}(t)=\tilde\cO(\sqrt{KtA(t)})=o(t).$$
\end{lemm}
The proof applies Lemma~\ref{lem:gf-anchor}, the oracle bound and martingale concentration within each epoch (Appendix~\ref{app:general-proofs}).

\textbf{Safety credits.} Instead of charging each candidate its own certificate, \BTIGWBC{} charges the envelope cumulatively. It keeps credits $Q_t$, with $Q_0=0$ and no reset between epochs: a candidate earns $h_t:=\tfrac12-(1-\alpha)s_t(\x_{b_t})$ and a fallback earns $\alpha s_t(\x_{b_t})$. By Lemma~\ref{lem:score-certificate} with $c=R_t$, a candidate's true surplus is at least $h_t-R_t/4$; a fallback's is at least $\alpha s_t(\x_{b_t})$ because it contains the baseline. The true cumulative surplus is therefore at least $Q_t-\tfrac14\sum_{i\in\cA_t}R_i$, where $\cA_t$ is the set of candidate rounds, and by Lemma~\ref{lem:gf-envelope} the charge is at most $G_{\mathrm{IGW}}(t)/4$, which the acceptance test in Algorithm~\ref{alg:gf-igw} requires the credits to cover.

\textbf{Main result.} Since $s_t(\x_{b_t})\in[\underline s,\tfrac12]$ with $\underline s=\sigma(-\Delta_h)$, every credit is at least $\alpha\underline s$, so a rejection at time $t$ requires $4\alpha\underline s\,t<G_{\mathrm{IGW}}(t)$. Let $N_T$ count fallbacks up to $T$, and let $N_{\mathrm{IGW}}:=\sup\{t\ge1:4\alpha\underline s\,t<G_{\mathrm{IGW}}(t)\}$, which is finite since $G_{\mathrm{IGW}}(t)=o(t)$.
\begin{theo}[\textbf{Anytime safety and regret}]
\label{thm:general}
    Under Assumptions~\ref{asmp:BTL-model}, \ref{asmp:bounded}, \ref{asmp:gapBounds}, and \ref{asmp:gf-oracle}, with probability at least $1-\delta$, simultaneously for every $T\ge1$, Algorithm~\ref{alg:gf-igw} is $\alpha$-conservative in the sense of \eqref{eq:performance}, $N_T\le\min\{T,N_{\mathrm{IGW}}\}$, and $\Reg(T)\le G_{\mathrm{IGW}}(T)+\tfrac12\Delta_hN_T$. Suppressing log factors in $T$ and $1/\delta$,
    \begin{align}
        \Reg(T)=\tilde\cO\!\left(\sqrt{KT\,A(T)}+\frac{\Delta_hK\,A(T)}{\alpha^2\underline{s}^{\,2}}\right).
        \label{eq:gf-main}
    \end{align}
\end{theo}
\emph{Proof sketch.} Work on the event of Lemma~\ref{lem:gf-envelope}. \emph{Safety:} if $\tau\le t$ is the last candidate round, then since $r_i^\circ\ge0$ and credits are positive, $$\sum_{i\in\cA_t}R_i\le\sum_{i\le\tau}r_i^\circ\le G_{\mathrm{IGW}}(\tau)\le4Q_\tau\le4Q_t,$$ 
so the true surplus is nonnegative. 

\emph{Fallbacks:} every fallback time satisfies $4\alpha\underline s\,\tau<G_{\mathrm{IGW}}(\tau)$, so $\tau\le N_{\mathrm{IGW}}$; at the last fallback $\tau\le T$, $G_{\mathrm{IGW}}(\tau)=\tilde\cO(\sqrt{K\tau A(T)})$ gives $N_T=\tilde\cO(KA(T)/(\alpha^2\underline s^{\,2}))$. 

\emph{Regret:} $R_t=r_t^\circ$ on candidate rounds and $R_t\le\tfrac12(\Delta_h+r_t^\circ)$ on fallbacks, so $$\Reg(T)\le G_{\mathrm{IGW}}(T)+\tfrac12\Delta_hN_T.$$ Details are in Appendix~\ref{app:gf-safety}. \hfill$\square$

\begin{remark}[\textbf{Finite classes}]
\label{rem:gf-finite}
    The oracle in Appendix~\ref{app:gf-finite} has $\mathcal R_{\log}(n)=e^B\log|\mathcal F|$, so $A(T)\equiv A_{\mathcal F}:=e^B\log|\mathcal F|/\kappa+1/\kappa^2$ and, as in the linear case, the cost of conservatism is independent of $T$.
\end{remark}
\begin{remark}[\textbf{From C-SquareCB to preference feedback}]
\label{rem:gf-csquarecb}
    C-SquareCB \citep{deb2025conservative} certifies safety by calibrating cumulative predictions of cardinal costs. Our oracle controls only comparison margins, so we certify scores through cumulative regret (Lemmas~\ref{lem:score-certificate} and~\ref{lem:gf-envelope}), and Lemma~\ref{lem:gf-anchor} lets fallbacks keep learning. Appendix~\ref{app:gf-csquarecb} explains why the calibration argument does not transfer.
\end{remark}

\begin{figure}[t]
\centering
\includegraphics[width=\textwidth]{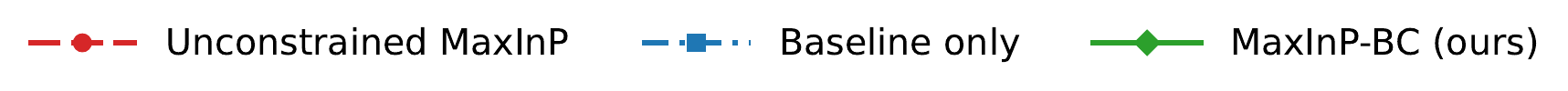}\\
\subfigure[$d=10$, $\alpha=0.05$]{%
  \includegraphics[width=0.32\textwidth]{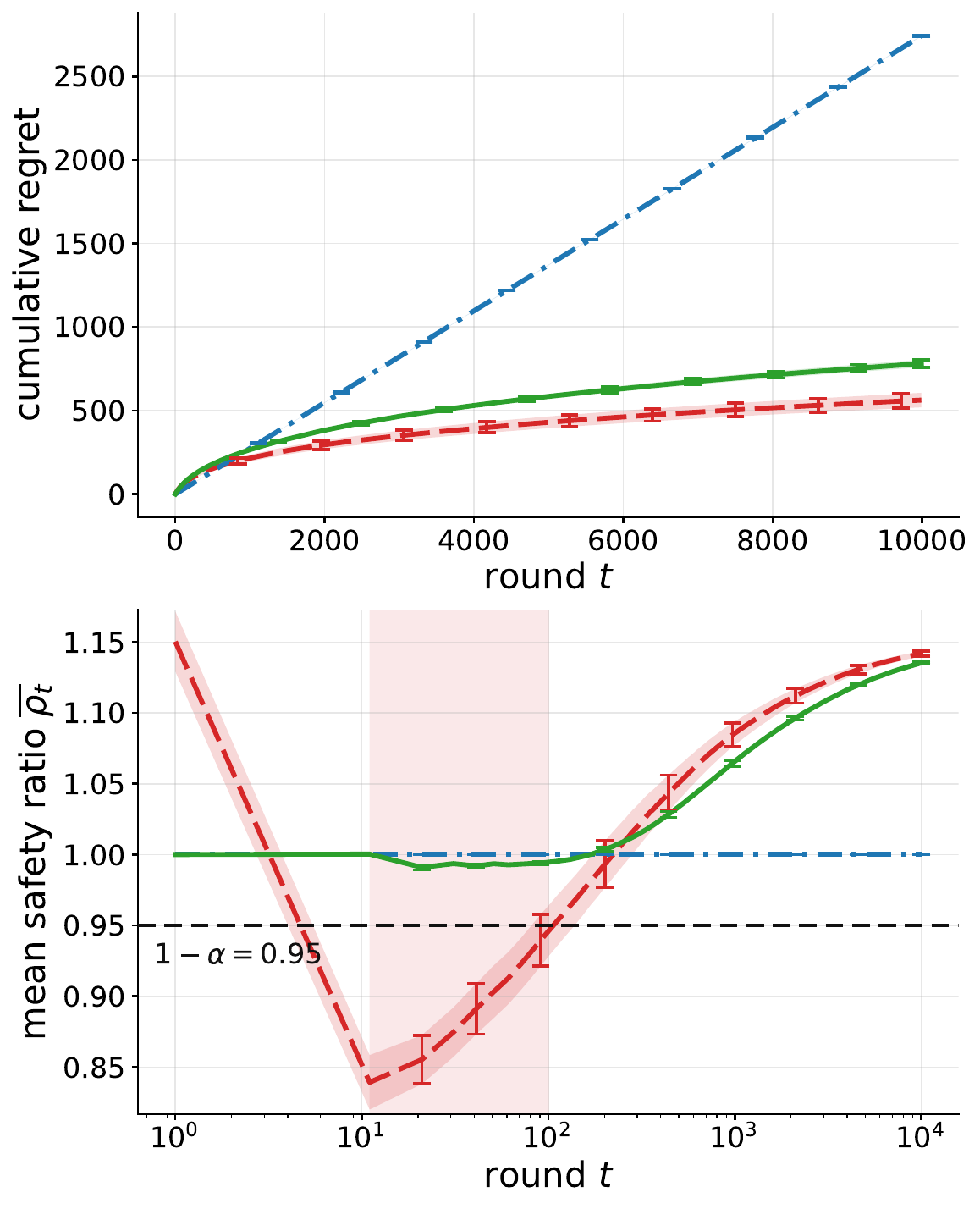}%
  \label{fig:linear-headline}}
\hfill
\subfigure[Varying $\alpha$ ($d=10$)]{%
  \includegraphics[width=0.32\textwidth]{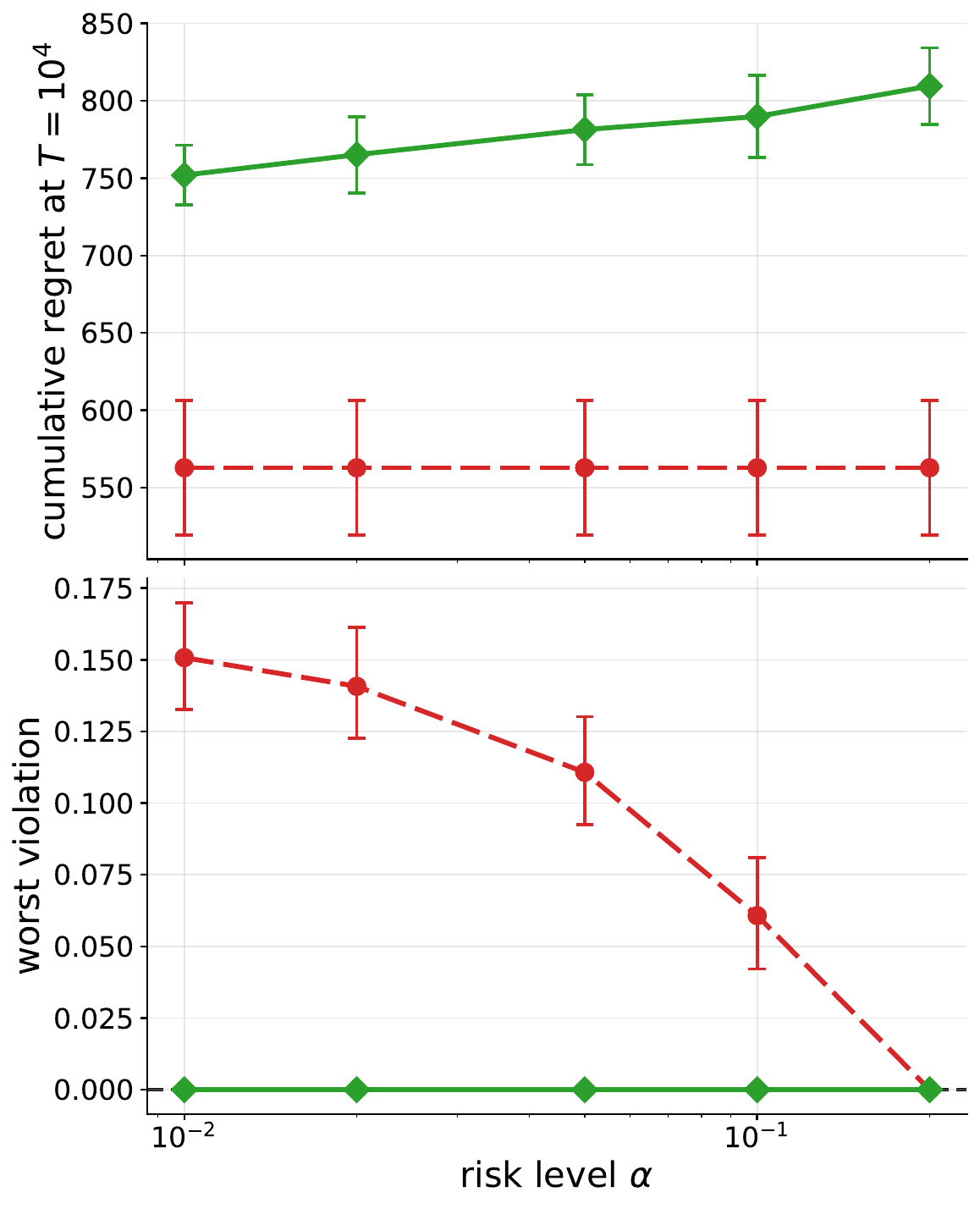}%
  \label{fig:linear-alpha}}
\hfill
\subfigure[Varying $d$ ($\alpha=0.05$)]{%
  \includegraphics[width=0.32\textwidth]{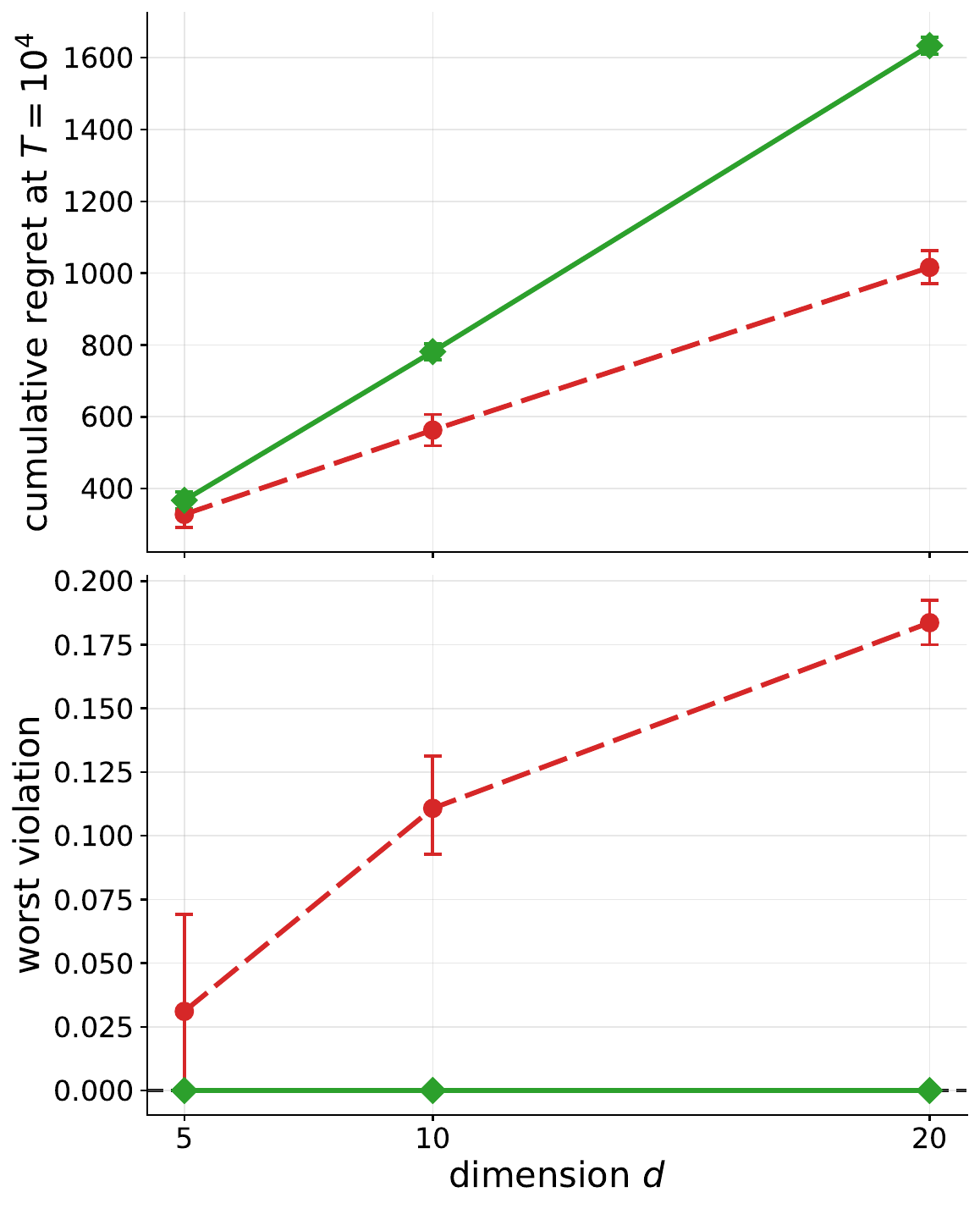}%
  \label{fig:linear-d}}
\caption{Linear utilities. Top: cumulative regret. Bottom: mean safety ratio $\bar\rho_t$ in (a) and worst violation in (b, c). Dashed lines mark the bar $1-\alpha$ in (a) and zero violation in (b, c); the vertical band in (a) marks the rounds where \textsc{MaxInP} is below the bar. Baseline-only is shown in (a) only. Intervals are $95\%$ confidence intervals over $50$ seeds. \MaxInPwBC{} has zero violation in every setting.}
\label{fig:linear}
\end{figure}

\section{Experiments}
\label{sec:exp}

We compare each conservative algorithm with its unconstrained counterpart, \textsc{MaxInP} \citep{saha2021maxinp} or \textsc{BT-IGW} \citep{sekhari2023aurora}, and with \emph{baseline-only}, which selects the baseline item.

\paragraph{Protocol.}
Each configuration runs for $T=10^4$ rounds on $50$ independent instances. We report cumulative regret and the \emph{safety ratio} $\rho_t=\sum_{i\le t}\max\{s_i(\x_i),s_i(\y_i)\}\big/\sum_{i\le t}s_i(\x_{b_i})$; the constraint \eqref{eq:performance} holds at prefix $t$ if and only if $\rho_t\ge1-\alpha$. Sweeps report the \emph{worst violation} $\max_t[(1-\alpha)-\bar\rho_t]_+$ of the cross-seed mean $\bar\rho_t$. Intervals are $95\%$ confidence intervals over seeds. The baseline is the item at fractional utility rank $q$, the fraction of items with higher utility. We tune only the multipliers on $\beta_t$, $\gamma_{e,m}$ and $G_{\mathrm{IGW}}$, by grid search on seeds disjoint from those reported: unconstrained learners minimise regret, and conservative learners minimise regret subject to the constraint. Tuned multipliers fall outside the theory, so the safety reported here is empirical.

\textbf{Linear utilities:} We take $d=10$, $K=50$, $\alpha=0.05$ and $q=0.14$, with $\beta$-multipliers $0.02$ for \textsc{MaxInP} and $0.05$ for \MaxInPwBC{}. \textsc{MaxInP} attains the lowest regret, but its mean safety ratio drops to $0.839$ and stays below the bar for roughly the first $100$ rounds, because its most-uncertain pair can contain two poor items. \MaxInPwBC{} never drops below $0.991$, at $1.39\times$ the regret of \textsc{MaxInP}, while baseline-only is safe but incurs linear regret, $3.5\times$ that of \MaxInPwBC{} (Figure~\ref{fig:linear-headline}). The sweeps confirm this across settings: \MaxInPwBC{} has zero violation for every $\alpha$ and $d$, whereas \textsc{MaxInP} violates the constraint in all but the loosest setting.

\textbf{General utilities:} Here $d=5$, $K=40$, $\alpha=0.02$ and $q=0.10$, with a realizable class of $|\mathcal F|=64$ random quadratics and an exponential-weights oracle, for which $\mathcal R_{\log}(n)=\log|\mathcal F|$. \textsc{BT-IGW} has the lowest regret ($464.5$), but its mean safety ratio starts at $0.862$ and stays below the bar for the first ${\approx}190$ rounds, whereas \BTIGWBC{} stays at or above $1.014$ at $2.69\times$ the regret, and baseline-only is a further $2.47\times$ worse (Figure~\ref{fig:general-headline}). Across the sweeps, \BTIGWBC{} has zero violation while \textsc{BT-IGW} violates the constraint in every setting. The cost of conservatism is largest when the constraint is tight (small $\alpha$, large $K$), and when it is loose ($\alpha\ge0.05$ or $K=10$), \BTIGWBC{} matches or beats \textsc{BT-IGW}.

\begin{figure}[!t]
\centering
\includegraphics[width=\textwidth]{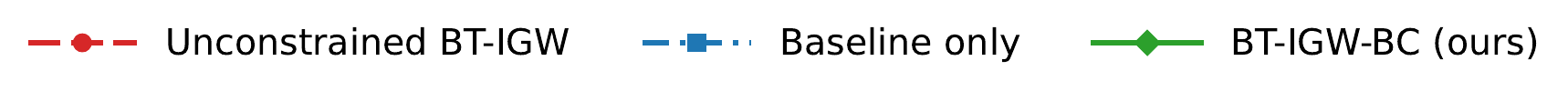}\\[-4pt]
\subfigure[$K=40$, $\alpha=0.02$]{%
  \includegraphics[width=0.30\textwidth]{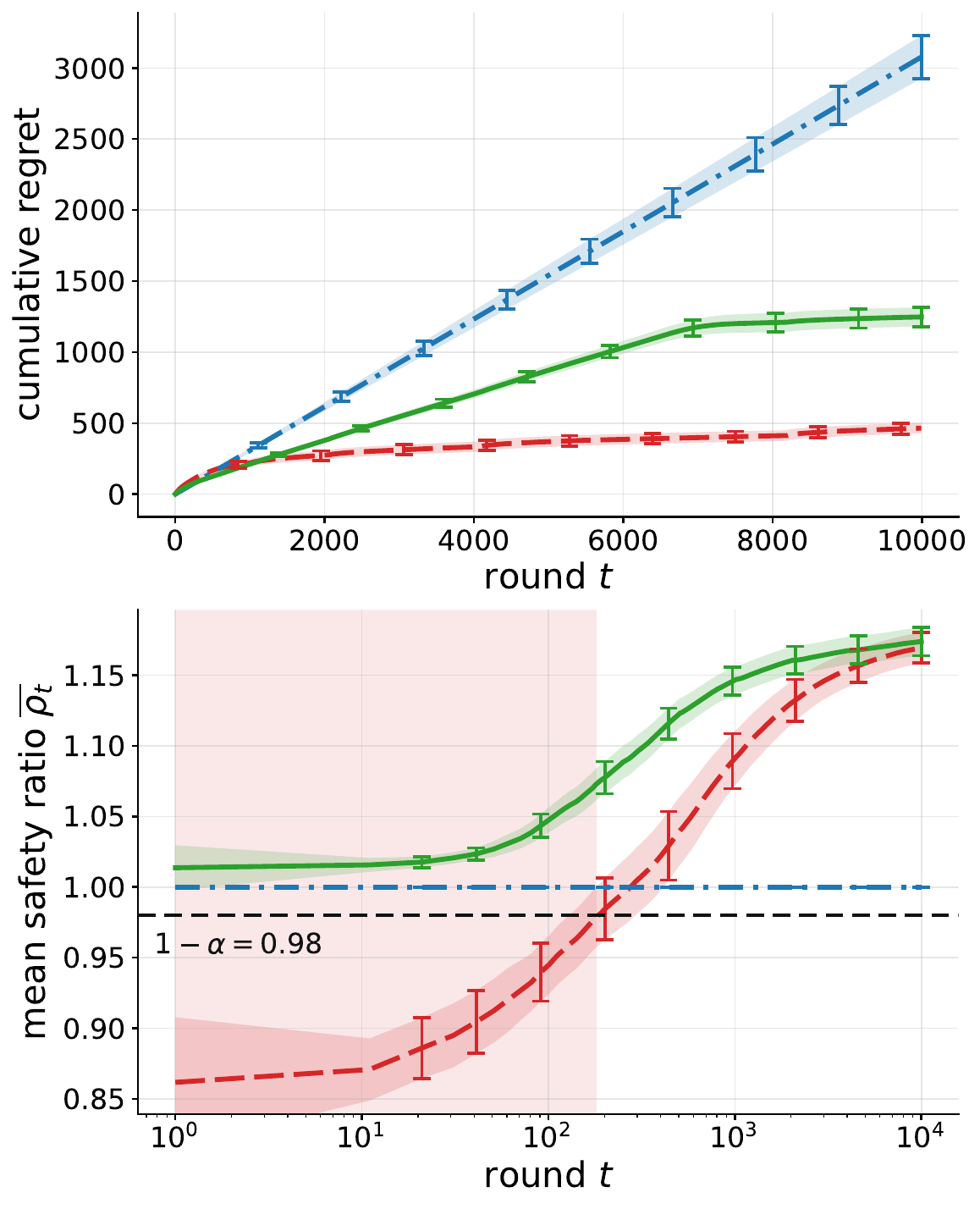}%
  \label{fig:general-headline}}
\hfill
\subfigure[Varying $\alpha$ ($K=40$)]{%
  \includegraphics[width=0.30\textwidth]{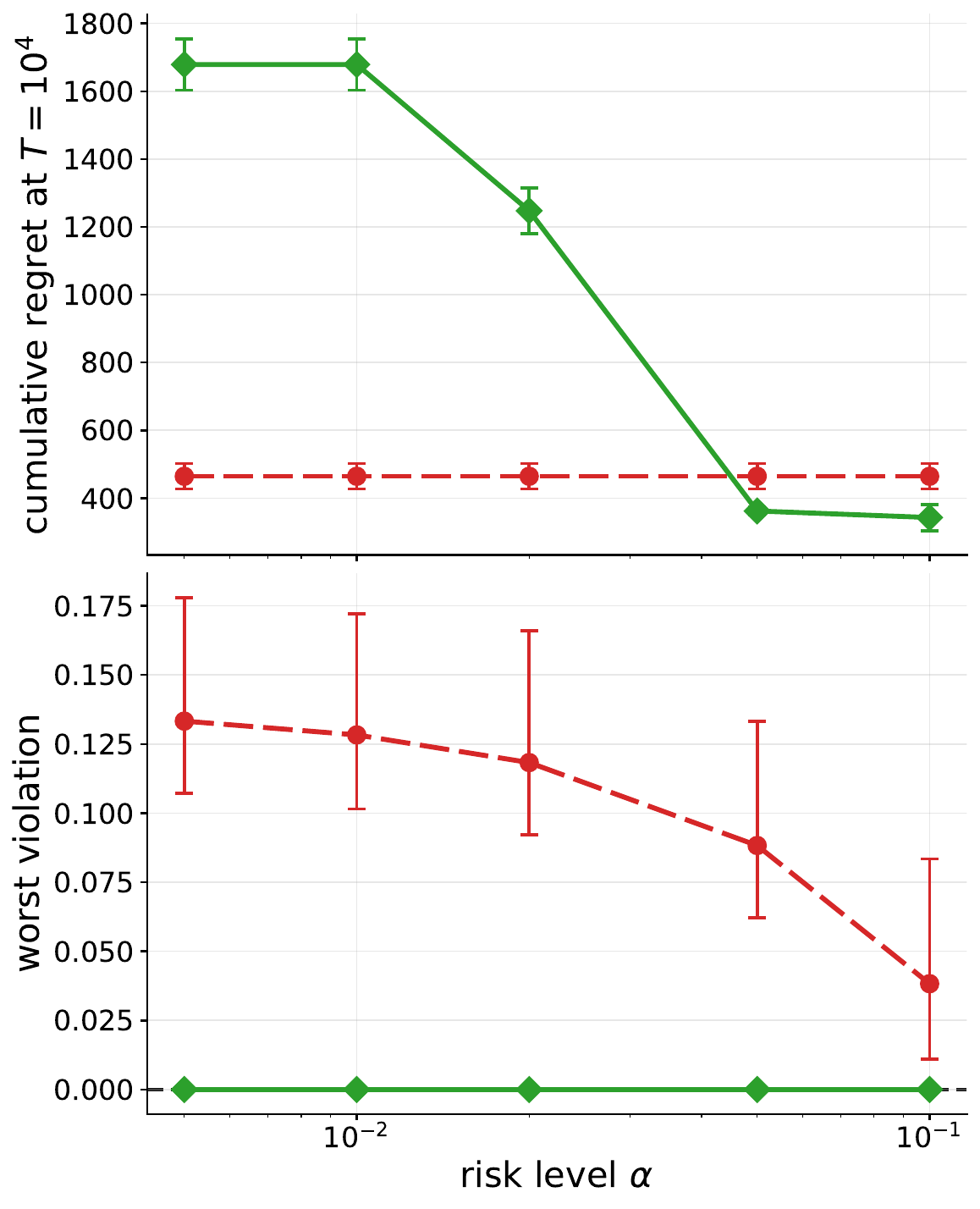}%
  \label{fig:general-alpha}}
\hfill
\subfigure[Varying $K$ ($\alpha=0.02$)]{%
  \includegraphics[width=0.30\textwidth]{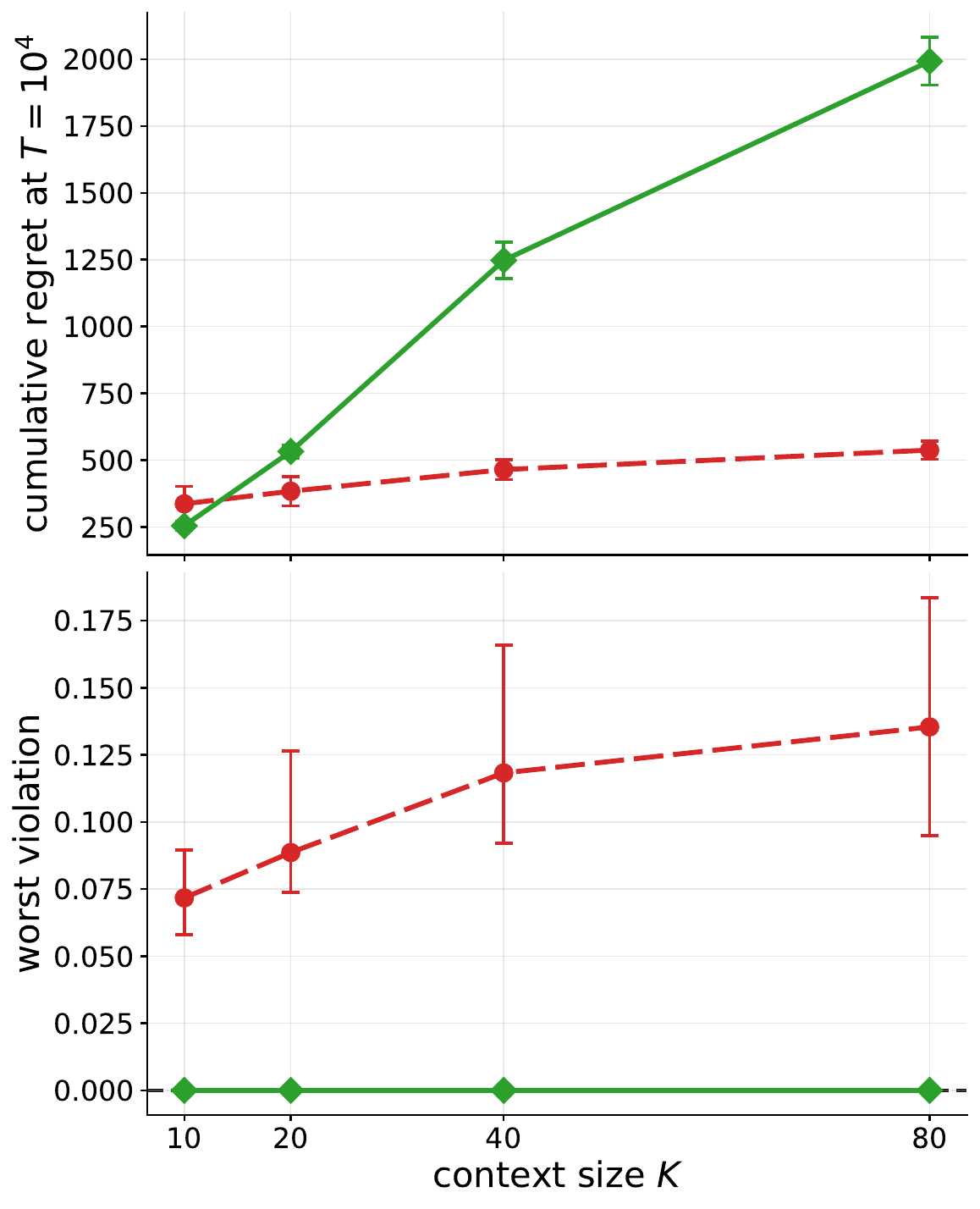}%
  \label{fig:general-K}}
\caption{General utilities ($d=5$, $|\mathcal F|=64$, $q=0.10$), laid out as in Figure~\ref{fig:linear} with \textsc{BT-IGW} in place of \textsc{MaxInP}. \BTIGWBC{} stays above the bar in (a), and its extra regret over \textsc{BT-IGW} arises mainly where \textsc{BT-IGW} violates the constraint.}
\label{fig:general}
\end{figure}

\section{Conclusion}
\label{sec:conclusion}
We introduced conservative contextual dueling bandits, which require the cumulative score of the better displayed item to stay above a $(1-\alpha)$ fraction of a trusted baseline's at all times. Because the constraint is stated through preference scores rather than rewards, it is invariant to shifts of the utility function and can be enforced from binary comparisons alone. We showed that this setting is not a routine extension of conservative bandits with numerical rewards, since transferring CLUCB directly can incur linear regret. \MaxInPwBC{} and \BTIGWBC{} avoid this failure by certifying the constraint through utility gaps instead of observed payoffs and by continuing to learn from the comparisons made when they fall back to the baseline. Both are safe at all times with high probability, for linear and general utility classes respectively. In both cases, conservatism adds only an additive term to the regret. Over long horizons, the price of safety is therefore negligible compared with the cost of learning itself.

Several directions remain open. One is conservative preference-based reinforcement learning \citep{garcelon2020conservativeRL}, where a policy learns from trajectory comparisons while its cumulative score stays close to that of a deployed baseline policy. Certificates there must account for the state distribution induced by the learner's own policy, not an exogenous context distribution. Another is replicable constrained dueling bandits \citep{impagliazzo2022reproducibility,deb2026replicablebanditsucbbased,deb2026replicableduelingbandits}, where the data-dependent acceptance test must give the same decisions across runs without compromising anytime safety. Closer to the present setting, fallback comparisons involve the baseline and are therefore informative about its score, so baseline scores could be estimated online rather than bounded in advance, as conservative bandits do for unknown baseline rewards \citep{conservativeMAB,conservative_linear_bandits}. A stricter, stagewise variant would require every displayed pair, rather than the cumulative sum, to meet the baseline bar, analogous to stagewise safety constraints for bandits with numerical rewards \citep{khezeli2020safe,moradipari2020stagewise}. Under non-transitive preferences beyond the Bradley--Terry--Luce model, scores would need to be defined against a von Neumann winner \citep{dudik2015contextual} rather than a best item. Finally, lower bounds on the cost of conservatism and feedback over more than two items remain open.


\bibliography{references}

@inproceedings{Garcelon2020IA,
  title={Improved Algorithms for Conservative Exploration in Bandits},
  author={Evrard Garcelon and Mohammad Ghavamzadeh and Alessandro Lazaric and Matteo Pirotta},
  booktitle={AAAI},
  year={2020}
}

@inproceedings{bengs2022stochastic,
  title={Stochastic Contextual Dueling Bandits under Linear Stochastic Transitivity Models},
  author={Bengs, Viktor and Saha, Aadirupa and H{\"u}llermeier, Eyke},
  booktitle={International Conference on Machine Learning},
  pages={1764--1786},
  year={2022},
  organization={PMLR}
}

@inproceedings{yue2009interactively,
  title={Interactively optimizing information retrieval systems as a dueling bandits problem},
  author={Yue, Yisong and Joachims, Thorsten},
  booktitle={Proceedings of the 26th Annual International Conference on Machine Learning},
  pages={1201--1208},
  year={2009}
}

@inproceedings{conservative_linear_bandits,
 author = {Kazerouni, Abbas and Ghavamzadeh, Mohammad and Abbasi Yadkori, Yasin and Van Roy, Benjamin},
 booktitle = {Advances in Neural Information Processing Systems},
 editor = {I. Guyon and U. Von Luxburg and S. Bengio and H. Wallach and R. Fergus and S. Vishwanathan and R. Garnett},
 pages = {},
 publisher = {Curran Associates, Inc.},
 title = {Conservative Contextual Linear Bandits},
 url = {https://proceedings.neurips.cc/paper_files/paper/2017/file/bdc4626aa1d1df8e14d80d345b2a442d-Paper.pdf},
 volume = {30},
 year = {2017}
}

@InProceedings{conservativeMAB,
  title = 	 {Conservative Bandits},
  author = 	 {Wu, Yifan and Shariff, Roshan and Lattimore, Tor and Szepesvari, Csaba},
  booktitle = 	 {Proceedings of The 33rd International Conference on Machine Learning},
  pages = 	 {1254--1262},
  year = 	 {2016},
  editor = 	 {Balcan, Maria Florina and Weinberger, Kilian Q.},
  volume = 	 {48},
  series = 	 {Proceedings of Machine Learning Research},
  address = 	 {New York, New York, USA},
  month = 	 {20--22 Jun},
  publisher =    {PMLR},
  url = 	 {https://proceedings.mlr.press/v48/wu16.html}
}

@inproceedings{christiano2017deep,
  title={Deep reinforcement learning from human preferences},
  author={Christiano, Paul and Leike, Jan and Brown, Tom and Martic, Miljan and Legg, Shane and Amodei, Dario},
  booktitle={NeurIPS},
  year={2017}
}

@misc{ziegler2019fine,
      title={Fine-Tuning Language Models from Human Preferences}, 
      author={Daniel M. Ziegler and Nisan Stiennon and Jeffrey Wu and Tom B. Brown and Alec Radford and Dario Amodei and Paul Christiano and Geoffrey Irving},
      year={2020},
      eprint={1909.08593},
      archivePrefix={arXiv},
      primaryClass={cs.CL},
      url={https://arxiv.org/abs/1909.08593}, 
}

@inproceedings{ouyang2022training,
 author = {Ouyang, Long and Wu, Jeffrey and Jiang, Xu and Almeida, Diogo and Wainwright, Carroll and Mishkin, Pamela and Zhang, Chong and Agarwal, Sandhini and Slama, Katarina and Ray, Alex and Schulman, John and Hilton, Jacob and Kelton, Fraser and Miller, Luke and Simens, Maddie and Askell, Amanda and Welinder, Peter and Christiano, Paul F and Leike, Jan and Lowe, Ryan},
 booktitle = {Advances in Neural Information Processing Systems},
 doi = {10.52202/068431-2011},
 editor = {S. Koyejo and S. Mohamed and A. Agarwal and D. Belgrave and K. Cho and A. Oh},
 pages = {27730--27744},
 publisher = {Curran Associates, Inc.},
 title = {Training language models to follow instructions with human feedback},
 url = {https://proceedings.neurips.cc/paper_files/paper/2022/file/b1efde53be364a73914f58805a001731-Paper-Conference.pdf},
 volume = {35},
 year = {2022}
}

@misc{soufiani2014plackett,
      title={Learning Mixtures of Plackett-Luce Models}, 
      author={Zhibing Zhao and Peter Piech and Lirong Xia},
      year={2020},
      eprint={1603.07323},
      archivePrefix={arXiv},
      primaryClass={cs.LG},
      url={https://arxiv.org/abs/1603.07323}, 
}

@article{hunter2004mm,
  title={MM algorithms for generalized Bradley–Terry models},
  author={Hunter, David R},
  journal={Annals of Statistics},
  year={2004},
  volume={32},
  number={1},
  pages={384--406}
}

@book{luce2005individual,
  title={Individual Choice Behavior: A Theoretical Analysis},
  author={Luce, R. Duncan},
  publisher={Courier Corporation},
  year={2005}
}

@article{yue2012karmed,
author = {Yue, Yisong and Broder, Josef and Kleinberg, Robert and Joachims, Thorsten},
title = {The K-armed dueling bandits problem},
year = {2012},
issue_date = {September, 2012},
publisher = {Academic Press, Inc.},
address = {USA},
volume = {78},
number = {5},
issn = {0022-0000},
url = {https://doi.org/10.1016/j.jcss.2011.12.028},
doi = {10.1016/j.jcss.2011.12.028},
journal = {J. Comput. Syst. Sci.},
month = sep,
pages = {1538–1556},
numpages = {19}
}

@inproceedings{saha2021maxinp,
 author = {Saha, Aadirupa},
 booktitle = {Advances in Neural Information Processing Systems},
 editor = {M. Ranzato and A. Beygelzimer and Y. Dauphin and P.S. Liang and J. Wortman Vaughan},
 pages = {30050--30062},
 publisher = {Curran Associates, Inc.},
 title = {Optimal Algorithms for Stochastic Contextual Preference Bandits },
 url = {https://proceedings.neurips.cc/paper_files/paper/2021/file/fc3cf452d3da8402bebb765225ce8c0e-Paper.pdf},
 volume = {34},
 year = {2021}
}

@article{BradleyTerry,
 ISSN = {00063444, 14643510},
 URL = {http://www.jstor.org/stable/2334029},
 author = {Ralph Allan Bradley and Milton E. Terry},
 journal = {Biometrika},
 number = {3/4},
 pages = {324--345},
 publisher = {[Oxford University Press, Biometrika Trust]},
 title = {Rank Analysis of Incomplete Block Designs: I. The Method of Paired Comparisons},
 urldate = {2025-04-25},
 volume = {39},
 year = {1952}
}

@InProceedings{dueling_feel_good,
  title = 	 {Feel-Good Thompson Sampling for Contextual Dueling Bandits},
  author =       {Li, Xuheng and Zhao, Heyang and Gu, Quanquan},
  booktitle = 	 {Proceedings of the 41st International Conference on Machine Learning},
  pages = 	 {29406--29426},
  year = 	 {2024},
  editor = 	 {Salakhutdinov, Ruslan and Kolter, Zico and Heller, Katherine and Weller, Adrian and Oliver, Nuria and Scarlett, Jonathan and Berkenkamp, Felix},
  volume = 	 {235},
  series = 	 {Proceedings of Machine Learning Research},
  month = 	 {21--27 Jul},
  publisher =    {PMLR},
  url = 	 {https://proceedings.mlr.press/v235/li24co.html}
}

@inproceedings{
verma2025neural,
title={Neural Dueling Bandits: Preference-Based Optimization with Human Feedback},
author={Arun Verma and Zhongxiang Dai and Xiaoqiang Lin and Patrick Jaillet and Bryan Kian Hsiang Low},
booktitle={The Thirteenth International Conference on Learning Representations},
year={2025},
url={https://openreview.net/forum?id=VELhv9BBfn}
}

@inproceedings{deb2025conservative,
  title     = {Conservative Contextual Bandits: Beyond Linear Representations},
  author    = {Deb, Rohan and Ghavamzadeh, Mohammad and Banerjee, Arindam},
  booktitle = {The Thirteenth International Conference on Learning Representations},
  year      = {2025},
  url       = {https://openreview.net/forum?id=SThJXvucjQ}
}

@inproceedings{sekhari2023aurora,
  title     = {Contextual Bandits and Imitation Learning with Preference-Based Active Queries},
  author    = {Sekhari, Ayush and Sridharan, Karthik and Sun, Wen and Wu, Runzhe},
  booktitle = {Advances in Neural Information Processing Systems},
  volume    = {36},
  year      = {2023},
  url       = {https://proceedings.neurips.cc/paper_files/paper/2023/file/2567c95fd41459a98a73ba893775d22a-Paper-Conference.pdf}
}

@book{furnkranz2010preference,
  editor    = {F{\"u}rnkranz, Johannes and H{\"u}llermeier, Eyke},
  title     = {Preference Learning},
  publisher = {Springer},
  year      = {2010}
}

@article{wirth2017survey,
  author  = {Wirth, Christian and Akrour, Riad and Neumann, Gerhard and F{\"u}rnkranz, Johannes},
  title   = {A Survey of Preference-Based Reinforcement Learning Methods},
  journal = {Journal of Machine Learning Research},
  volume  = {18},
  number  = {136},
  pages   = {1--46},
  year    = {2017}
}

@article{bengs2021preference,
  author  = {Bengs, Viktor and Busa-Fekete, R{\'o}bert and El Mesaoudi-Paul, Adil and H{\"u}llermeier, Eyke},
  title   = {Preference-based Online Learning with Dueling Bandits: A Survey},
  journal = {Journal of Machine Learning Research},
  volume  = {22},
  number  = {7},
  pages   = {1--108},
  year    = {2021}
}

@inproceedings{joachims2002optimizing,
  author    = {Joachims, Thorsten},
  title     = {Optimizing Search Engines Using Clickthrough Data},
  booktitle = {Proceedings of the 8th ACM SIGKDD International Conference on Knowledge Discovery and Data Mining},
  pages     = {133--142},
  year      = {2002}
}

@inproceedings{radlinski2008does,
  author    = {Radlinski, Filip and Kurup, Madhu and Joachims, Thorsten},
  title     = {How Does Clickthrough Data Reflect Retrieval Quality?},
  booktitle = {Proceedings of the 17th ACM Conference on Information and Knowledge Management (CIKM)},
  pages     = {43--52},
  year      = {2008}
}

@article{chapelle2012large,
  author  = {Chapelle, Olivier and Joachims, Thorsten and Radlinski, Filip and Yue, Yisong},
  title   = {Large-Scale Validation and Analysis of Interleaved Search Evaluation},
  journal = {ACM Transactions on Information Systems},
  volume  = {30},
  number  = {1},
  pages   = {6:1--6:41},
  year    = {2012}
}

@article{hofmann2016online,
  author  = {Hofmann, Katja and Li, Lihong and Radlinski, Filip},
  title   = {Online Evaluation for Information Retrieval},
  journal = {Foundations and Trends in Information Retrieval},
  volume  = {10},
  number  = {1},
  pages   = {1--117},
  year    = {2016}
}

@INPROCEEDINGS{jones2011comparisons,
  author={Jones, Nicolas and Brun, Armelle and Boyer, Anne},
  booktitle={2011 IEEE/WIC/ACM International Conferences on Web Intelligence and Intelligent Agent Technology}, 
  title={Comparisons Instead of Ratings: Towards More Stable Preferences}, 
  year={2011},
  volume={1},
  number={},
  pages={451-456},
  doi={10.1109/WI-IAT.2011.13}}

@inproceedings{kalloori2016pairwise,
  author    = {Kalloori, Saikishore and Ricci, Francesco and Tkalcic, Marko},
  title     = {Pairwise Preferences Based Matrix Factorization and Nearest Neighbor Recommendation Techniques},
  booktitle = {Proceedings of the 10th ACM Conference on Recommender Systems (RecSys)},
  year      = {2016}
}

@inproceedings{brochu2007active,
  author    = {Brochu, Eric and de Freitas, Nando and Ghosh, Abhijeet},
  title     = {Active Preference Learning with Discrete Choice Data},
  booktitle = {Advances in Neural Information Processing Systems 20 (NIPS)},
  year      = {2007}
}

@inproceedings{gonzalez2017preferential,
  author    = {Gonz{\'a}lez, Javier and Dai, Zhenwen and Damianou, Andreas and Lawrence, Neil D.},
  title     = {Preferential {B}ayesian Optimization},
  booktitle = {Proceedings of the 34th International Conference on Machine Learning (ICML)},
  series    = {PMLR},
  volume    = {70},
  pages     = {1282--1291},
  year      = {2017}
}

@inproceedings{tucker2020preference,
  author    = {Tucker, Maegan and Novoseller, Ellen and Kann, Claudia and Sui, Yanan and Yue, Yisong and Burdick, Joel W. and Ames, Aaron D.},
  title     = {Preference-Based Learning for Exoskeleton Gait Optimization},
  booktitle = {IEEE International Conference on Robotics and Automation (ICRA)},
  pages     = {2351--2357},
  year      = {2020}
}

@inproceedings{stiennon2020learning,
  author    = {Stiennon, Nisan and Ouyang, Long and Wu, Jeffrey and Ziegler, Daniel and Lowe, Ryan and Voss, Chelsea and Radford, Alec and Amodei, Dario and Christiano, Paul F.},
  title     = {Learning to Summarize with Human Feedback},
  booktitle = {Advances in Neural Information Processing Systems 33 (NeurIPS)},
  pages     = {3008--3021},
  year      = {2020}
}

@misc{bai2022training,
  author        = {Bai, Yuntao and Jones, Andy and Ndousse, Kamal and Askell, Amanda and Chen, Anna and DasSarma, Nova and Drain, Dawn and Fort, Stanislav and Ganguli, Deep and Henighan, Tom and others},
  title         = {Training a Helpful and Harmless Assistant with Reinforcement Learning from Human Feedback},
  year          = {2022},
  eprint        = {2204.05862},
  archivePrefix = {arXiv}
}

@inproceedings{rafailov2023direct,
 author = {Rafailov, Rafael and Sharma, Archit and Mitchell, Eric and Manning, Christopher D and Ermon, Stefano and Finn, Chelsea},
 booktitle = {Advances in Neural Information Processing Systems},
 doi = {10.52202/075280-2338},
 editor = {A. Oh and T. Naumann and A. Globerson and K. Saenko and M. Hardt and S. Levine},
 pages = {53728--53741},
 publisher = {Curran Associates, Inc.},
 title = {Direct Preference Optimization: Your Language Model is Secretly a Reward Model},
 url = {https://proceedings.neurips.cc/paper_files/paper/2023/file/a85b405ed65c6477a4fe8302b5e06ce7-Paper-Conference.pdf},
 volume = {36},
 year = {2023}
}

@InProceedings{pmlr-v238-deb24a,
  title = 	 {Think Before You Duel: Understanding Complexities of Preference Learning under Constrained Resources},
  author =       {Deb, Rohan and Saha, Aadirupa and Banerjee, Arindam},
  booktitle = 	 {Proceedings of The 27th International Conference on Artificial Intelligence and Statistics},
  pages = 	 {4546--4554},
  year = 	 {2024},
  editor = 	 {Dasgupta, Sanjoy and Mandt, Stephan and Li, Yingzhen},
  volume = 	 {238},
  series = 	 {Proceedings of Machine Learning Research},
  month = 	 {02--04 May},
  publisher =    {PMLR},
  url = 	 {https://proceedings.mlr.press/v238/deb24a.html}
}

@inproceedings{
deb2024contextual,
title={Contextual Bandits with Online Neural Regression},
author={Rohan Deb and Yikun Ban and Shiliang Zuo and Jingrui He and Arindam Banerjee},
booktitle={The Twelfth International Conference on Learning Representations},
year={2024},
url={https://openreview.net/forum?id=5ep85sakT3}
}

@inproceedings{
deb2025thompson,
title={Thompson Sampling for Constrained Bandits},
author={Rohan Deb and Mohammad Ghavamzadeh and Arindam Banerjee},
booktitle={Reinforcement Learning Conference},
year={2025},
url={https://openreview.net/forum?id=5HL4QaNQcf}
}

@misc{deb2026replicableduelingbandits,
      title={Replicable Dueling Bandits}, 
      author={Rohan Deb and Aadirupa Saha and Udaya Ghai and Karan Singh and Arindam Banerjee},
      year={2026},
      eprint={2610.20908},
      archivePrefix={arXiv},
      primaryClass={cs.LG},
      url={https://arxiv.org/abs/2610.20908}, 
}

@misc{deb2026replicablebanditsucbbased,
      title={Replicable Bandits with UCB based Exploration}, 
      author={Rohan Deb and Udaya Ghai and Karan Singh and Arindam Banerjee},
      year={2026},
      eprint={2604.20024},
      archivePrefix={arXiv},
      primaryClass={cs.LG},
      url={https://arxiv.org/abs/2604.20024}, 
}

@InProceedings{garcelon2020conservativeRL,
  title     = {Conservative Exploration in Reinforcement Learning},
  author    = {Garcelon, Evrard and Ghavamzadeh, Mohammad and Lazaric, Alessandro and Pirotta, Matteo},
  booktitle = {Proceedings of the Twenty Third International Conference on Artificial Intelligence and Statistics},
  pages     = {1431--1441},
  year      = {2020},
  editor    = {Chiappa, Silvia and Calandra, Roberto},
  volume    = {108},
  series    = {Proceedings of Machine Learning Research},
  publisher = {PMLR},
  url       = {https://proceedings.mlr.press/v108/garcelon20a.html}
}

@inproceedings{impagliazzo2022reproducibility,
  title     = {Reproducibility in Learning},
  author    = {Impagliazzo, Russell and Lei, Rex and Pitassi, Toniann and Sorrell, Jessica},
  booktitle = {Proceedings of the 54th Annual ACM SIGACT Symposium on Theory of Computing (STOC)},
  pages     = {818--831},
  year      = {2022},
  publisher = {Association for Computing Machinery},
  doi       = {10.1145/3519935.3519973}
}

@inproceedings{khezeli2020safe,
  title     = {Safe Linear Stochastic Bandits},
  author    = {Khezeli, Kia and Bitar, Eilyan},
  booktitle = {Proceedings of the AAAI Conference on Artificial Intelligence},
  volume    = {34},
  number    = {06},
  pages     = {10202--10209},
  year      = {2020},
  doi       = {10.1609/aaai.v34i06.6581}
}

@inproceedings{moradipari2020stagewise,
  title     = {Stage-wise Conservative Linear Bandits},
  author    = {Moradipari, Ahmadreza and Thrampoulidis, Christos and Alizadeh, Mahnoosh},
  booktitle = {Advances in Neural Information Processing Systems},
  volume    = {33},
  pages     = {11191--11201},
  year      = {2020}
}

@InProceedings{dudik2015contextual,
  title     = {Contextual Dueling Bandits},
  author    = {Dud{\'\i}k, Miroslav and Hofmann, Katja and Schapire, Robert E. and Slivkins, Aleksandrs and Zoghi, Masrour},
  booktitle = {Proceedings of The 28th Conference on Learning Theory},
  pages     = {563--587},
  year      = {2015},
  editor    = {Gr{\"u}nwald, Peter and Hazan, Elad and Kale, Satyen},
  volume    = {40},
  series    = {Proceedings of Machine Learning Research},
  publisher = {PMLR},
  url       = {https://proceedings.mlr.press/v40/Dudik15.html}
}

\appendix

\newpage
\section{Proofs for Section~\ref{sec:linear}}
\label{app:linear-proofs}
Throughout this appendix Assumptions~\ref{asmp:BTL-model}, \ref{asmp:bounded}, \ref{asmp:gapBounds} and \ref{asmp:linear} are in force, and $\lambda>0$, $\alpha,\delta\in(0,1)$ are the inputs of Algorithm~\ref{alg:cmaxinp}.

\subsection{Notation and filtration}
\label{app:filtration}

Recall $\z_t=\x_t-\y_t$ is the comparison vector of the pair actually displayed at round $t$, that $V_1=\lambda I_d$ and $V_{t+1}=V_t+\z_t\z_t^\top$, and write
\[
    \varepsilon_t \;:=\; o_t-\sigma(\z_t^\top\theta^\star),
    \qquad
    \Ss_t \;:=\; \sum_{i=1}^{t}\varepsilon_i\z_i,
    \qquad
    v_t \;:=\; \|\z_t\|_{V_t^{-1}} .
\]
Let $\cF_{t-1}$ be the $\sigma$-algebra generated by all contexts, displayed pairs and outcomes of rounds $1,\dots,t-1$, together with the round-$t$ context set $\cS_t$, the baseline item $\x_{b_t}$ and its score, and the round-$t$ displayed pair. Every quantity the algorithm computes at round $t$ --- $V_t$, $\hat\theta_t$, $\beta_t$, $\cC_t$, $(\tilde\x_t,\tilde\y_t)$, $w_t$, $c_t$, $\Lambda_{t-1}$, the acceptance decision, and hence $\z_t$ --- is $\cF_{t-1}$-measurable, and $(\cF_t)_{t\ge0}$ is a filtration with $\E[o_t\mid\cF_{t-1}]=\sigma(\z_t^\top\theta^\star)$ by Assumption~\ref{asmp:BTL-model}. Unlike analyses that discard uninformative baseline rounds, no random time change is needed here.

\begin{remark}[\textbf{Degenerate rounds}]
\label{rem:degenerate}
    A round with $\z_t=0$ contributes $\varepsilon_t\z_t=0$ to $\Ss_t$, leaves $V_{t+1}=V_t$, and adds the constant $\log 2$ to the objective in \eqref{eq:design}; it therefore affects neither the estimator nor any argument below. Such a round can only be a candidate round: by \eqref{eq:width-transfer}, if $\max_{\y\in\cC_t}\|\x_{b_t}-\y\|_{V_t^{-1}}=0$ then $w_t=0$, hence $c_t=0$, and the acceptance test passes because $\Lambda_{t-1}\ge0$ and $\tfrac12-(1-\alpha)s_t(\x_{b_t})\ge\tfrac\alpha2>0$ (Lemma~\ref{lem:surplus-floor}). Consequently every fallback round has $v_t\ge w_t/2>0$.
\end{remark}

\subsection{Elementary bounds on scores}

\begin{lemm}
\label{lem:surplus-floor}
    For every $t$ and every $\x\in\cS_t$, $\;\sigma(-B)\le s_t(\x)\le\tfrac12$, and $s_t(\x_t^\star)=\tfrac12$. Moreover $\underline{s}:=\sigma(-\Delta_h)$ satisfies $\underline{s}\le s_t(\x_{b_t})\le\tfrac12$, so that
    \[
        \tfrac12-(1-\alpha)s_t(\x_{b_t}) \;\ge\; \tfrac\alpha2
        \qquad\text{and}\qquad
        \alpha\,s_t(\x_{b_t}) \;\ge\; \alpha\underline{s}.
    \]
\end{lemm}

\begin{proof}
    By definition $s_t(\x)=\sigma\big(f(\x)-f(\x_t^\star)\big)$ with $f(\x)-f(\x_t^\star)\in[-B,0]$ by Assumption~\ref{asmp:bounded} and the definition of $\x_t^\star$; monotonicity of $\sigma$ and $\sigma(0)=\tfrac12$ give the first claim. For the baseline, $f(\x_t^\star)-f(\x_{b_t})=\Delta_{t,b_t}\in[\Delta_l,\Delta_h]$ by Assumption~\ref{asmp:gapBounds}, so $s_t(\x_{b_t})\in[\sigma(-\Delta_h),\sigma(-\Delta_l)]\subseteq[\underline{s},\tfrac12]$. Finally $\tfrac12-(1-\alpha)s_t(\x_{b_t})\ge\tfrac12-(1-\alpha)\tfrac12=\tfrac\alpha2$.
\end{proof}

Note for later use that $\underline{s}\le\tfrac12$, since $\Delta_h\ge0$.

\subsection{Anytime logistic confidence ellipsoid}
\label{app:ellipsoid}

This subsection proves Lemma~\ref{lem:ellipsoid}. The argument is self-contained.

\begin{lemm}[Hoeffding's lemma]
\label{lem:hoeffding}
    Let $X$ be a random variable supported on an interval of length $b-a$. Then $\E\exp\big(\eta(X-\E X)\big)\le\exp\big(\eta^2(b-a)^2/8\big)$ for every $\eta\in\R$.
\end{lemm}

\begin{proof}
    After the affine rescaling $a=0$, $b=1$ it suffices to treat $X\in[0,1]$; write $\mu=\E X$. Convexity of the exponential gives $e^{\eta x}\le(1-x)e^0+xe^\eta$ for $x\in[0,1]$, so taking expectations and centring,
    \begin{align}
        \E e^{\eta(X-\mu)} \;\le\; (1-\mu)e^{-\mu\eta}+\mu e^{(1-\mu)\eta}.
        \label{eq:hoeff-1}
    \end{align}
    Let $h(\eta)=\log\big((1-\mu)e^{-\mu\eta}+\mu e^{(1-\mu)\eta}\big)$ and $q(\eta)=\mu e^{(1-\mu)\eta}/\big((1-\mu)e^{-\mu\eta}+\mu e^{(1-\mu)\eta}\big)$. Differentiating, $h'(\eta)=q(\eta)-\mu$ and $h''(\eta)=q(\eta)(1-q(\eta))\le\tfrac14$, while $h(0)=0$ and $h'(0)=q(0)-\mu=0$. Taylor's formula with integral remainder, valid for either sign of $\eta$, gives
    \[
        h(\eta)=\eta^2\int_0^1(1-s)h''(s\eta)\,\dif s\;\le\;\frac{\eta^2}{8}.
    \]
    Exponentiating and combining with \eqref{eq:hoeff-1} proves the normalised case; for a general interval apply it to $(X-a)/(b-a)$ with parameter $\eta(b-a)$.
\end{proof}

\begin{lemm}[Conditional sub-Gaussianity of preference residuals]
\label{lem:subgaussian}
    For every $t\ge1$, $\E[\varepsilon_t\mid\cF_{t-1}]=0$ and $\E[e^{\eta\varepsilon_t}\mid\cF_{t-1}]\le e^{\eta^2/8}$ for all $\eta\in\R$. That is, $\varepsilon_t$ is conditionally $R$-sub-Gaussian with $R=\tfrac12$ in the convention $\E[e^{\eta\varepsilon}\mid\cF_{t-1}]\le e^{R^2\eta^2/2}$.
\end{lemm}

\begin{proof}
    Conditionally on $\cF_{t-1}$ the outcome $o_t$ is Bernoulli with mean $\sigma(\z_t^\top\theta^\star)$, since $\z_t$ is $\cF_{t-1}$-measurable; it is supported on an interval of length one. Apply Lemma~\ref{lem:hoeffding} conditionally.
\end{proof}

\begin{lemm}[Ville's inequality]
\label{lem:ville}
    Let $(M_n)_{n\ge0}$ be a nonnegative supermartingale with $\E M_0\le1$. Then $\P(\sup_{n\ge0}M_n\ge c)\le1/c$ for every $c>0$.
\end{lemm}

\begin{proof}
    Fix a deterministic $N$ and let $\tau_N=\min\{n\le N: M_n\ge c\}$, with $\tau_N=N$ if the set is empty. Since $\{\tau_N\ge n\}\in\cF_{n-1}$,
    \[
        \E M_{\tau_N}
        =\E M_0+\sum_{n=1}^{N}\E\big[\I\{\tau_N\ge n\}\,\E[M_n-M_{n-1}\mid\cF_{n-1}]\big]
        \;\le\;\E M_0\;\le\;1 .
    \]
    On the event $\max_{0\le n\le N}M_n\ge c$ the stopped value is at least $c$, so $c\,\P(\max_{n\le N}M_n\ge c)\le\E M_{\tau_N}\le1$. Let $N\to\infty$ and use continuity from below.
\end{proof}

\begin{lemm}[Method-of-mixtures self-normalised bound]
\label{lem:mixtures}
    Let $(\varepsilon_i,z_i)_{i\ge1}$ be adapted with $z_i\in\R^d$ being $\cF_{i-1}$-measurable and $\E[e^{\eta\varepsilon_i}\mid\cF_{i-1}]\le e^{R^2\eta^2/2}$ for all $\eta\in\R$. Let $V_0\succ0$, $S_n=\sum_{i\le n}\varepsilon_iz_i$ and $V_n'=V_0+\sum_{i\le n}z_iz_i^\top$. Then with probability at least $1-\delta$, simultaneously for every $n\ge0$,
    \begin{align}
        \|S_n\|_{(V_n')^{-1}}^2 \;\le\; 2R^2\log\!\left(\frac{\det(V_n')^{1/2}}{\det(V_0)^{1/2}\,\delta}\right).
        \label{eq:selfnorm}
    \end{align}
\end{lemm}

\begin{proof}
    Fix $q\in\R^d$. Conditional sub-Gaussianity with $\eta=q^\top z_i$ gives $\E\big[\exp\big(q^\top z_i\varepsilon_i-\tfrac{R^2}{2}(q^\top z_i)^2\big)\mid\cF_{i-1}\big]\le1$, so
    \[
        M_n(q)=\exp\Big(q^\top S_n-\tfrac{R^2}{2}q^\top(V_n'-V_0)q\Big)
    \]
    is a nonnegative supermartingale with $M_0(q)=1$. Let $Q$ be an independent Gaussian vector with density $f(q)=R^d\det(V_0)^{1/2}(2\pi)^{-d/2}\exp(-\tfrac{R^2}{2}q^\top V_0q)$, which integrates to one because $\int e^{-\frac12 q^\top Aq}\dif q=(2\pi)^{d/2}\det(A)^{-1/2}$ with $A=R^2V_0$. By Tonelli's theorem and the conditional supermartingale property for each fixed $q$, the mixture $M_n=\int M_n(q)f(q)\dif q$ is a nonnegative supermartingale with $M_0=1$, and
    \[
        M_n=\frac{R^d\det(V_0)^{1/2}}{(2\pi)^{d/2}}\int_{\R^d}\exp\Big(q^\top S_n-\tfrac{R^2}{2}q^\top V_n'q\Big)\dif q .
    \]
    Completing the square with $a=R^{-2}(V_n')^{-1}S_n$ gives
    \[
        q^\top S_n-\tfrac{R^2}{2}q^\top V_n'q
        = -\tfrac{R^2}{2}(q-a)^\top V_n'(q-a)+\tfrac{1}{2R^2}\|S_n\|^2_{(V_n')^{-1}},
    \]
    and evaluating the Gaussian integral yields
    \begin{align}
        M_n=\left(\frac{\det V_0}{\det V_n'}\right)^{1/2}\exp\left(\frac{1}{2R^2}\|S_n\|^2_{(V_n')^{-1}}\right).
        \label{eq:mixture-value}
    \end{align}
    Lemma~\ref{lem:ville} with $c=1/\delta$ bounds $\sup_n M_n<1/\delta$ with probability at least $1-\delta$; taking logarithms in \eqref{eq:mixture-value} and rearranging gives \eqref{eq:selfnorm}.
\end{proof}

\begin{lemm}[Trace--determinant bound]
\label{lem:tracedet}
    If $\|z_i\|_2\le2L$ for all $i$ and $W_m=\lambda I_d+\sum_{i=1}^{m-1}z_iz_i^\top$, then
    \[
        \log\frac{\det W_m}{\det(\lambda I_d)}\;\le\;d\log\!\Big(1+\frac{4L^2(m-1)}{\lambda d}\Big).
    \]
\end{lemm}

\begin{proof}
    Let $A=\lambda^{-1}\sum_{i<m}z_iz_i^\top\succeq0$ with eigenvalues $\rho_1,\dots,\rho_d\ge0$, so that $\det W_m/\det(\lambda I_d)=\prod_j(1+\rho_j)$. By the arithmetic--geometric mean inequality, $\prod_j(1+\rho_j)\le\big(1+\tr(A)/d\big)^d$, and $\tr(A)=\lambda^{-1}\sum_{i<m}\|z_i\|_2^2\le 4L^2(m-1)/\lambda$. Take logarithms.
\end{proof}

\begin{lemm}[Global curvature of the regularised likelihood]
\label{lem:curvature}
    Let $F_t(\theta)=\sum_{i<t}\big[\log(1+e^{\z_i^\top\theta})-o_i\z_i^\top\theta\big]+\tfrac\lambda2\|\theta\|_2^2$. For all $\theta,\theta'\in\Theta$,
    \[
        \big\langle\nabla F_t(\theta)-\nabla F_t(\theta'),\,\theta-\theta'\big\rangle\;\ge\;\kappa\,\|\theta-\theta'\|^2_{V_t}.
    \]
\end{lemm}

\begin{proof}
    For $\vartheta\in\Theta$ we have $\nabla^2F_t(\vartheta)=\sum_{i<t}\sigma'(\z_i^\top\vartheta)\z_i\z_i^\top+\lambda I_d$. Since $\|\z_i\|_2\le2L$ and $\|\vartheta\|_2\le S$, $|\z_i^\top\vartheta|\le2LS=B$, so $\sigma'(\z_i^\top\vartheta)\ge\kappa$ and the fact that $\sigma'$ is even and decreasing in $|s|$. As $0<\kappa\le\tfrac14<1$,
    \[
        \nabla^2F_t(\vartheta)\;\succeq\;\kappa\sum_{i<t}\z_i\z_i^\top+\lambda I_d\;\succeq\;\kappa\Big(\sum_{i<t}\z_i\z_i^\top+\lambda I_d\Big)=\kappa V_t .
    \]
    Writing $\mathbf{u}=\theta-\theta'$ and noting $\Theta$ is convex, the fundamental theorem of calculus gives $\langle\nabla F_t(\theta)-\nabla F_t(\theta'),\mathbf{u}\rangle=\int_0^1\mathbf{u}^\top\nabla^2F_t(\theta'+s\mathbf{u})\mathbf{u}\,\dif s\ge\kappa\,\mathbf{u}^\top V_t\mathbf{u}$.
\end{proof}

\noindent\textbf{Lemma~\ref{lem:ellipsoid} (restated).} \emph{With probability at least $1-\delta$, the event $\cE=\{\|\hat\theta_t-\theta^\star\|_{V_t}\le\beta_t\ \forall t\ge1\}$ holds, with $\beta_t$ as in \eqref{eq:beta}.}

\begin{proof}
    Set $\dd_t=\hat\theta_t-\theta^\star$. Since $\hat\theta_t$ minimises the differentiable convex $F_t$ over the closed convex $\Theta$ and $\theta^\star\in\Theta$, first-order optimality gives $\langle\nabla F_t(\hat\theta_t),\theta^\star-\hat\theta_t\rangle\ge0$, i.e.\ $\langle\nabla F_t(\hat\theta_t),\dd_t\rangle\le0$. Lemma~\ref{lem:curvature} with $\theta=\hat\theta_t$, $\theta'=\theta^\star$ then yields
    \begin{align}
        \kappa\|\dd_t\|_{V_t}^2\;\le\;\big\langle\nabla F_t(\hat\theta_t)-\nabla F_t(\theta^\star),\dd_t\big\rangle\;\le\;-\big\langle\nabla F_t(\theta^\star),\dd_t\big\rangle .
        \label{eq:ellip-1}
    \end{align}
    The gradient at the truth is $\nabla F_t(\theta^\star)=\sum_{i<t}\big(\sigma(\z_i^\top\theta^\star)-o_i\big)\z_i+\lambda\theta^\star=-\Ss_{t-1}+\lambda\theta^\star$, so by Cauchy--Schwarz in the $V_t$ geometry,
    \[
        -\big\langle\nabla F_t(\theta^\star),\dd_t\big\rangle
        =\langle \Ss_{t-1},\dd_t\rangle-\lambda\langle\theta^\star,\dd_t\rangle
        \;\le\;\Big(\|\Ss_{t-1}\|_{V_t^{-1}}+\lambda\|\theta^\star\|_{V_t^{-1}}\Big)\|\dd_t\|_{V_t}.
    \]
    Since $V_t\succeq\lambda I_d$ we have $V_t^{-1}\preceq\lambda^{-1}I_d$, hence $\lambda\|\theta^\star\|_{V_t^{-1}}\le\sqrt\lambda\,\|\theta^\star\|_2\le\sqrt\lambda\,S$. If $\dd_t=0$ the claim is trivial; otherwise divide \eqref{eq:ellip-1} by $\kappa\|\dd_t\|_{V_t}$ to obtain
    \begin{align}
        \|\dd_t\|_{V_t}\;\le\;\frac{1}{\kappa}\Big(\|\Ss_{t-1}\|_{V_t^{-1}}+\sqrt\lambda\,S\Big).
        \label{eq:ellip-2}
    \end{align}
    Apply Lemma~\ref{lem:mixtures} with $R=\tfrac12$, $V_0=\lambda I_d$, $z_i=\z_i$ and $n=t-1$, so that $V_{t-1}'=V_t$: with probability at least $1-\delta$, simultaneously for every $t\ge1$,
    \[
        \|\Ss_{t-1}\|_{V_t^{-1}}
        \;\le\;\frac12\sqrt{2\log\!\left(\frac{\det(V_t)^{1/2}}{\det(\lambda I_d)^{1/2}\delta}\right)}
        \;=\;\frac12\sqrt{\log\frac{\det V_t}{\det(\lambda I_d)}+2\log\frac1\delta}\,.
    \]
    Bounding the determinant ratio by Lemma~\ref{lem:tracedet} with $m=t$ and substituting into \eqref{eq:ellip-2} gives exactly $\beta_t$ of \eqref{eq:beta}.
\end{proof}

For the remainder of the appendix we work on $\cE$, which by Lemma~\ref{lem:ellipsoid} has probability at least $1-\delta$. All statements below are deterministic consequences of $\cE$; in particular no further probabilistic argument is attached to the safety budget.

\subsection{The promising set and the regret certificate}
\label{app:certificate}

Write $U_t(\x,\y)=\langle\hat\theta_t,\x-\y\rangle+\beta_t\|\x-\y\|_{V_t^{-1}}$, so that $\cC_t=\{\x\in\cS_t: U_t(\x,\y)\ge0\ \forall\y\in\cS_t\}$ as in \eqref{eq:promising}.

\begin{lemm}[Nonemptiness]
\label{lem:nonempty}
    For every round $t$ and every value of $\hat\theta_t$, $\cC_t\neq\emptyset$.
\end{lemm}

\begin{proof}
    Let $\hat\x_t\in\argmax_{\x\in\cS_t}\langle\hat\theta_t,\x\rangle$. For every $\y\in\cS_t$ we have $\langle\hat\theta_t,\hat\x_t-\y\rangle\ge0$, and the confidence bonus is nonnegative, so $U_t(\hat\x_t,\y)\ge0$. Hence $\hat\x_t\in\cC_t$.
\end{proof}

\begin{lemm}[The optimal item is promising]
\label{lem:opt-promising}
    On $\cE$, $\x_t^\star\in\cC_t$ for every $t$.
\end{lemm}

\begin{proof}
    Fix $\y\in\cS_t$ and put $\mathbf{z}=\x_t^\star-\y$. On $\cE$, Cauchy--Schwarz gives $\langle\hat\theta_t-\theta^\star,\mathbf{z}\rangle\ge-\beta_t\|\mathbf{z}\|_{V_t^{-1}}$, whence
    \[
        U_t(\x_t^\star,\y)=\langle\hat\theta_t,\mathbf{z}\rangle+\beta_t\|\mathbf{z}\|_{V_t^{-1}}\;\ge\;\langle\theta^\star,\mathbf{z}\rangle=f(\x_t^\star)-f(\y)\;\ge\;0 .
    \]
\end{proof}

\begin{lemm}[Gap of a promising item]
\label{lem:gap-promising}
    On $\cE$, every $\x\in\cC_t$ satisfies $f(\x_t^\star)-f(\x)\le2\beta_t\|\x_t^\star-\x\|_{V_t^{-1}}$.
\end{lemm}

\begin{proof}
    Since $\x\in\cC_t$ and $\x_t^\star\in\cS_t$, the definition of $\cC_t$ with $\y=\x_t^\star$ gives $U_t(\x,\x_t^\star)\ge0$, i.e.
    \begin{align}
        \langle\hat\theta_t,\x_t^\star-\x\rangle\;\le\;\beta_t\|\x_t^\star-\x\|_{V_t^{-1}}.
        \label{eq:gp-1}
    \end{align}
    On $\cE$, Cauchy--Schwarz also gives
    \begin{align}
        \langle\theta^\star-\hat\theta_t,\x_t^\star-\x\rangle\;\le\;\beta_t\|\x_t^\star-\x\|_{V_t^{-1}}.
        \label{eq:gp-2}
    \end{align}
    Adding \eqref{eq:gp-1} and \eqref{eq:gp-2} yields $\langle\theta^\star,\x_t^\star-\x\rangle\le2\beta_t\|\x_t^\star-\x\|_{V_t^{-1}}$.
\end{proof}

\noindent\textbf{Lemma~\ref{lem:certificate} (restated).} \emph{$\cC_t\neq\emptyset$ always. On $\cE$, $\x_t^\star\in\cC_t$ and $f(\x_t^\star)-f(\y)\le c_t$ for every $\y\in\cC_t$; in particular the candidate pair satisfies $R_t=f(\x_t^\star)-\tfrac12\big(f(\tilde\x_t)+f(\tilde\y_t)\big)\le c_t$.}

\begin{proof}
    Nonemptiness is Lemma~\ref{lem:nonempty} and $\x_t^\star\in\cC_t$ is Lemma~\ref{lem:opt-promising}. Fix $\y\in\cC_t$. Both $\x_t^\star$ and $\y$ lie in $\cC_t$, and $(\tilde\x_t,\tilde\y_t)$ maximises $\|\cdot-\cdot\|_{V_t^{-1}}$ over $\cC_t\times\cC_t$ by \eqref{eq:candidate}, so $\|\x_t^\star-\y\|_{V_t^{-1}}\le w_t$. Lemma~\ref{lem:gap-promising} then gives $f(\x_t^\star)-f(\y)\le2\beta_tw_t$. Moreover $f(\x_t^\star)-f(\y)\le B$ by Assumption~\ref{asmp:bounded}. Hence $f(\x_t^\star)-f(\y)\le\min\{B,2\beta_tw_t\}=c_t$. Applying this to $\y=\tilde\x_t$ and $\y=\tilde\y_t$, both of which lie in $\cC_t$, and averaging the two inequalities gives $R_t\le c_t$.
\end{proof}

\noindent\textbf{Lemma~\ref{lem:score-certificate} (restated).} \emph{If a pair $(\x,\y)$ satisfies $f(\x_t^\star)-\tfrac12(f(\x)+f(\y))\le c$ for some $c\ge0$, then $\max\{s_t(\x),s_t(\y)\}\ge\sigma(-c)\ge\tfrac12-\tfrac c4$.}

\begin{proof}
    Write $g_\x=f(\x_t^\star)-f(\x)\ge0$ and $g_\y=f(\x_t^\star)-f(\y)\ge0$, so that $s_t(\x)=\sigma(-g_\x)$ and $s_t(\y)=\sigma(-g_\y)$. Since $\sigma$ is increasing,
    \[
        \max\{s_t(\x),s_t(\y)\}=\sigma\big(-\min\{g_\x,g_\y\}\big),
        \qquad
        \min\{g_\x,g_\y\}\le\tfrac12(g_\x+g_\y)\le c,
    \]
    which gives the first inequality. For the second, $\tfrac12-\sigma(-c)=\sigma(0)-\sigma(-c)=\int_{-c}^{0}\sigma'(u)\,\dif u\le c/4$ because $\sigma'\le\tfrac14$ everywhere.
\end{proof}

\noindent\textbf{Width transfer \eqref{eq:width-transfer}.} \emph{For every round $t$, $\;w_t\le2\|\x_{b_t}-\y_t^{\mathrm{fb}}\|_{V_t^{-1}}$.}

\begin{proof}
    $V_t\succ0$, so $\|\cdot\|_{V_t^{-1}}$ is a norm. Since $\tilde\x_t,\tilde\y_t\in\cC_t$,
    \[
        w_t=\|\tilde\x_t-\tilde\y_t\|_{V_t^{-1}}
        \le\|\tilde\x_t-\x_{b_t}\|_{V_t^{-1}}+\|\x_{b_t}-\tilde\y_t\|_{V_t^{-1}}
        \le2\max_{\y\in\cC_t}\|\x_{b_t}-\y\|_{V_t^{-1}}
        =2\|\x_{b_t}-\y_t^{\mathrm{fb}}\|_{V_t^{-1}},
    \]
    the last step by the definition \eqref{eq:fallback} of $\y_t^{\mathrm{fb}}$. Note that $\x_{b_t}\in\cC_t$ is not required.
\end{proof}

\subsection{Elliptical potential and the certificate envelope}
\label{app:envelope}

\begin{lemm}[Matrix determinant recursion]
\label{lem:detlemma}
    If $W'=W+zz^\top$ with $W\succ0$, then $\det W'=\det W\,\big(1+\|z\|^2_{W^{-1}}\big)$.
\end{lemm}

\begin{proof}
    Factoring $W^{1/2}$ from both sides, $\det(W+zz^\top)=\det(W)\det(I_d+ww^\top)$ with $w=W^{-1/2}z$. The rank-one matrix $ww^\top$ has eigenvalues $\|w\|_2^2,0,\dots,0$, so $\det(I_d+ww^\top)=1+\|w\|_2^2=1+z^\top W^{-1}z$.
\end{proof}

\begin{lemm}[Elliptical potential]
\label{lem:ellpot}
    Let $W_1=\lambda I_d$ and $W_{i+1}=W_i+z_iz_i^\top$ with $\|z_i\|_2\le2L$. Then for every $m\ge1$,
    \[
        \sum_{i=1}^{m}\min\big\{1,\|z_i\|^2_{W_i^{-1}}\big\}\;\le\;2d\log\!\Big(1+\frac{4L^2m}{\lambda d}\Big).
    \]
\end{lemm}

\begin{proof}
    First, $\min\{1,x\}\le2\log(1+x)$ for every $x\ge0$: if $0\le x\le1$ then $\log(1+x)\ge x-x^2/2\ge x/2$; if $x\ge1$ then $2\log(1+x)\ge2\log2>1$. Writing $\omega_i^2=\|z_i\|^2_{W_i^{-1}}$ and applying Lemma~\ref{lem:detlemma},
    \[
        \sum_{i=1}^m\min\{1,\omega_i^2\}\le2\sum_{i=1}^m\log(1+\omega_i^2)=2\log\frac{\det W_{m+1}}{\det W_1}.
    \]
    Lemma~\ref{lem:tracedet} applied with $m+1$ in place of $m$ bounds the right side by $2d\log(1+4L^2m/(\lambda d))$.
\end{proof}

\begin{lemm}[Pointwise certificate bound]
\label{lem:pointwise}
    Let $\bar\beta_t=\max\{B/2,\beta_t\}$. For every round $t$ of Algorithm~\ref{alg:cmaxinp},
    \[
        c_t\;\le\;4\bar\beta_t\min\{1,v_t\},\qquad v_t=\|\z_t\|_{V_t^{-1}} .
    \]
    The same bound holds for the certificate $c_\tau$ of a rejected proposal with $v_\tau$ replaced by $w_\tau=\|\tilde\x_\tau-\tilde\y_\tau\|_{V_\tau^{-1}}$.
\end{lemm}

\begin{proof}
    Recall $c_t=\min\{B,2\beta_tw_t\}$ and $2\bar\beta_t\ge B$, $\bar\beta_t\ge\beta_t$. We distinguish the two round types.

    \emph{Candidate round.} Here $\z_t=\tilde\x_t-\tilde\y_t$, so $v_t=w_t$. If $v_t\le1$ then $c_t\le2\beta_tv_t\le2\bar\beta_tv_t\le4\bar\beta_t\min\{1,v_t\}$. If $v_t>1$ then $c_t\le B\le2\bar\beta_t\le4\bar\beta_t=4\bar\beta_t\min\{1,v_t\}$.

    \emph{Fallback round.} Here $\z_t=\x_{b_t}-\y_t^{\mathrm{fb}}$, and \eqref{eq:width-transfer} gives $w_t\le2v_t$, so $c_t\le\min\{B,4\beta_tv_t\}$. If $v_t\le1$ then $c_t\le4\beta_tv_t\le4\bar\beta_t\min\{1,v_t\}$; if $v_t>1$ then $c_t\le B\le2\bar\beta_t\le4\bar\beta_t$.

    The final claim is the candidate-round computation with $w_\tau$ in place of $v_\tau$, which does not use that the pair is displayed.
\end{proof}

\begin{lemm}[Cumulative certificate envelope]
\label{lem:envelope}
    On any trajectory of Algorithm~\ref{alg:cmaxinp}, for every $t$,
    \begin{align}
        \sum_{i=1}^{t}c_i\;\le\;G(t).
        \label{eq:env-1}
    \end{align}
    Moreover, if the proposal at round $\tau$ is rejected, then
    \begin{align}
        \sum_{i=1}^{\tau-1}c_i+c_\tau\;\le\;G(\tau).
        \label{eq:env-2}
    \end{align}
\end{lemm}

\begin{proof}
    Since $t\mapsto\beta_t$ is nondecreasing, so is $t\mapsto\bar\beta_t$, and hence $\bar\beta_i\le\bar\beta_t$ for $i\le t$. By Lemma~\ref{lem:pointwise} and Cauchy--Schwarz,
    \[
        \sum_{i=1}^{t}c_i\;\le\;4\bar\beta_t\sum_{i=1}^{t}\min\{1,v_i\}
        \;\le\;4\bar\beta_t\sqrt{t\sum_{i=1}^{t}\min\{1,v_i\}^2}
        \;=\;4\bar\beta_t\sqrt{t\sum_{i=1}^{t}\min\{1,v_i^2\}},
    \]
    using $\min\{1,v\}^2=\min\{1,v^2\}$ for $v\ge0$. Applying Lemma~\ref{lem:ellpot} with $z_i=\z_i$, which satisfies $\|\z_i\|_2\le2L$ by Assumption~\ref{asmp:linear}, bounds the inner sum by $2d\log(1+4L^2t/(\lambda d))$, giving \eqref{eq:env-1} by the definition \eqref{eq:envelope} of $G$.

    For \eqref{eq:env-2}, apply the same computation to the length-$\tau$ vector sequence $\z_1,\dots,\z_{\tau-1},\tilde\z_\tau$ with $\tilde\z_\tau:=\tilde\x_\tau-\tilde\y_\tau$. Its associated matrices are $V_1,\dots,V_\tau$, which are exactly those the algorithm formed, since the first $\tau-1$ vectors are the displayed ones; the $\tau$-th width in this sequence is $\|\tilde\z_\tau\|_{V_\tau^{-1}}=w_\tau$, and $\|\tilde\z_\tau\|_2\le2L$. Lemma~\ref{lem:pointwise} bounds the first $\tau-1$ certificates and, by its final claim, also $c_\tau$; Lemma~\ref{lem:ellpot} with $m=\tau$ needs only $V_1,\dots,V_\tau$ together with $\|\tilde\z_\tau\|_2\le2L$. Hence the argument applies without $\tilde\z_\tau$ ever being displayed or added to $V$.
\end{proof}

\subsection{Safety, fallback count, and regret}
\label{app:safety}

Let $\cA_t,\cB_t$ be the candidate and fallback rounds up to $t$, with $M_t=|\cA_t|$ and $N_t=|\cB_t|$, so $M_t+N_t=t$. Iterating the update in Algorithm~\ref{alg:cmaxinp},
\begin{align}
    \Lambda_t=\sum_{i\in\cA_t}\Big[\tfrac12-\tfrac{c_i}{4}-(1-\alpha)s_i(\x_{b_i})\Big]+\sum_{i\in\cB_t}\alpha\,s_i(\x_{b_i}),
    \label{eq:budget-expanded}
\end{align}
and we write the true cumulative surplus as
\begin{align}
    \Sigma_t=\sum_{i=1}^{t}\Big[\max\{s_i(\x_i),s_i(\y_i)\}-(1-\alpha)s_i(\x_{b_i})\Big].
    \label{eq:true-surplus}
\end{align}

\begin{lemm}[Budget invariant]
\label{lem:budget}
    On $\cE$, $\;\Sigma_t\ge\Lambda_t\ge0$ for every $t\ge0$.
\end{lemm}

\begin{proof}
    Induction on $t$; both quantities vanish at $t=0$. Suppose $\Sigma_{t-1}\ge\Lambda_{t-1}\ge0$.

    \emph{Candidate round.} On $\cE$, Lemma~\ref{lem:certificate} gives $R_t\le c_t$ for the displayed pair, so Lemma~\ref{lem:score-certificate} with $c=c_t$ yields $\max\{s_t(\x_t),s_t(\y_t)\}\ge\tfrac12-\tfrac{c_t}{4}$. Hence the true one-round increment of \eqref{eq:true-surplus} is at least $\tfrac12-\tfrac{c_t}{4}-(1-\alpha)s_t(\x_{b_t})$, which is exactly the budget increment, and $\Sigma_t\ge\Lambda_t$. Nonnegativity of $\Lambda_t$ is precisely the acceptance test.

    \emph{Fallback round.} The displayed pair is $(\x_{b_t},\y_t^{\mathrm{fb}})$, so $\max\{s_t(\x_t),s_t(\y_t)\}\ge s_t(\x_{b_t})$ and the true increment is at least $s_t(\x_{b_t})-(1-\alpha)s_t(\x_{b_t})=\alpha s_t(\x_{b_t})$, again the budget increment; so $\Sigma_t\ge\Lambda_t$. Since $s_t(\x_{b_t})>0$ and $\Lambda_{t-1}\ge0$, we get $\Lambda_t\ge0$.
\end{proof}

\begin{lemm}[Fallback rounds are early]
\label{lem:fallback-count}
    On $\cE$, every fallback round $\tau$ satisfies $4\alpha\underline{s}\,\tau<G(\tau)$. Consequently $N_T\le\min\{T,\Ncons\}$ for every $T\ge1$, with $\Ncons$ as in \eqref{eq:envelope}.
\end{lemm}

\begin{proof}
    Let $\tau$ be a fallback round, so the acceptance test failed:
    \[
        \Lambda_{\tau-1}+\tfrac12-\tfrac{c_\tau}{4}-(1-\alpha)s_\tau(\x_{b_\tau})<0 .
    \]
    Expanding $\Lambda_{\tau-1}$ by \eqref{eq:budget-expanded} and rearranging,
    \[
        \sum_{i\in\cA_{\tau-1}}\Big[\tfrac12-(1-\alpha)s_i(\x_{b_i})\Big]
        +\sum_{i\in\cB_{\tau-1}}\alpha s_i(\x_{b_i})
        +\Big[\tfrac12-(1-\alpha)s_\tau(\x_{b_\tau})\Big]
        \;<\;\frac14\Big(\sum_{i\in\cA_{\tau-1}}c_i+c_\tau\Big).
    \]
    By Lemma~\ref{lem:surplus-floor} the first bracket is at least $\tfrac\alpha2$ for each $i$, the second summand is at least $\alpha\underline{s}$ for each $i$, and the third bracket is at least $\tfrac\alpha2$. On the right, certificates are nonnegative, so $\sum_{i\in\cA_{\tau-1}}c_i+c_\tau\le\sum_{i=1}^{\tau-1}c_i+c_\tau\le G(\tau)$ by \eqref{eq:env-2}. Therefore
    \[
        \tfrac\alpha2 M_{\tau-1}+\alpha\underline{s}\,N_{\tau-1}+\tfrac\alpha2\;<\;\tfrac14G(\tau).
    \]
    Since $\underline{s}\le\tfrac12$ we have $\tfrac\alpha2\ge\alpha\underline{s}$, so the left side is at least $\alpha\underline{s}\,(M_{\tau-1}+N_{\tau-1}+1)=\alpha\underline{s}\,\tau$. Hence $4\alpha\underline{s}\,\tau<G(\tau)$.

    Every fallback time therefore belongs to $\{t\ge1:4\alpha\underline{s}\,t<G(t)\}$ and so is at most the supremum $\Ncons$ of that set. Fallback times are distinct positive integers in $\{1,\dots,\Ncons\}$, giving $N_T\le\Ncons$; and $N_T\le T$ trivially.
\end{proof}

\begin{lemm}[Finiteness of $\Ncons$]
\label{lem:ncons-finite}
    $G(t)=o(t)$, and consequently $\Ncons<\infty$.
\end{lemm}

\begin{proof}
    From \eqref{eq:beta}, $\beta_t=O\big(\sqrt{d\log t+\log(1/\delta)}/\kappa+\sqrt\lambda S/\kappa\big)$, so $\bar\beta_t=O(\sqrt{d\log(et)})$ for fixed $d,L,S,\lambda,\delta,\kappa$, and hence
    \[
        G(t)=4\bar\beta_t\sqrt{2dt\log\!\Big(1+\frac{4L^2t}{\lambda d}\Big)}=O\big(\sqrt t\,\log(et)\big)=o(t).
    \]
    Thus there is $t_0$ with $G(t)<4\alpha\underline{s}\,t$ for all $t\ge t_0$, so the set in \eqref{eq:envelope} is contained in $\{1,\dots,t_0-1\}$ and its supremum is finite. Note $\underline{s}\ge\sigma(-B)>0$ by Assumption~\ref{asmp:bounded} and $\Delta_h\le B$, so the threshold is strictly positive on every instance.
\end{proof}

\noindent\textbf{Theorem~\ref{thm:linear} (restated).} \emph{With probability at least $1-\delta$, simultaneously for every $T\ge1$: the displayed pairs are $\alpha$-conservative in the sense of \eqref{eq:performance}; $N_T\le\min\{T,\Ncons\}$; and $\Reg(T)\le G(T)+\tfrac12\Delta_h\Ncons$.}

\begin{proof}
    By Lemma~\ref{lem:ellipsoid}, $\P(\cE)\ge1-\delta$; all three claims are established on $\cE$, hence hold simultaneously on a single event.

    \emph{Safety.} By Lemma~\ref{lem:budget}, $\Sigma_T\ge0$ for every $T$, and expanding \eqref{eq:true-surplus} this is exactly
    $\sum_{t\le T}\max\{s_t(\x_t),s_t(\y_t)\}\ge(1-\alpha)\sum_{t\le T}s_t(\x_{b_t})$,
    i.e.\ \eqref{eq:performance}. Since the invariant holds at every $t$, the guarantee is anytime.

    \emph{Fallback count.} Lemma~\ref{lem:fallback-count}, with $\Ncons<\infty$ by Lemma~\ref{lem:ncons-finite}.

    \emph{Regret.} Split by round type. On a candidate round, $R_t\le c_t$ by Lemma~\ref{lem:certificate}. On a fallback round the displayed pair is $(\x_{b_t},\y_t^{\mathrm{fb}})$ with $\y_t^{\mathrm{fb}}\in\cC_t$, so Lemma~\ref{lem:certificate} gives $f(\x_t^\star)-f(\y_t^{\mathrm{fb}})\le c_t$, while $f(\x_t^\star)-f(\x_{b_t})=\Delta_{t,b_t}\le\Delta_h$ by Assumption~\ref{asmp:gapBounds}; averaging,
    \[
        R_t=\tfrac12\Big[\big(f(\x_t^\star)-f(\x_{b_t})\big)+\big(f(\x_t^\star)-f(\y_t^{\mathrm{fb}})\big)\Big]\;\le\;\tfrac12\big(\Delta_h+c_t\big).
    \]
    Summing and using $c_t\ge0$,
    \[
        \Reg(T)\;\le\;\sum_{t\in\cA_T}c_t+\sum_{t\in\cB_T}\frac{c_t}{2}+\frac{\Delta_h}{2}N_T
        \;\le\;\sum_{t=1}^{T}c_t+\frac{\Delta_h}{2}N_T
        \;\le\;G(T)+\frac{\Delta_h}{2}\Ncons,
    \]
    by \eqref{eq:env-1} and Lemma~\ref{lem:fallback-count}.
\end{proof}

\subsection{The simplified rate}
\label{app:rate}

\begin{lemm}
\label{lem:rate}
    Let $H(t):=G(t)/\sqrt t=4\bar\beta_t\sqrt{2d\log\big(1+4L^2t/(\lambda d)\big)}$, which is nondecreasing. Then
    \begin{align}
        \Ncons\;\le\;\frac{H(\Ncons)^2}{16\,\alpha^2\underline{s}^{\,2}}
        \;=\;\frac{2d\,\bar\beta_{\Ncons}^2\,\log\big(1+4L^2\Ncons/(\lambda d)\big)}{\alpha^2\underline{s}^{\,2}},
        \label{eq:ncons-explicit}
    \end{align}
    and, suppressing logarithmic factors in $T,d,L,S,\lambda,1/\delta$,
    \[
        \Ncons=\tilde\cO\!\Big(\frac{d^2}{\kappa^2\alpha^2\underline{s}^{\,2}}\Big),
        \qquad
        G(T)=\tilde\cO\!\Big(\frac{d\sqrt T}{\kappa}\Big),
        \qquad
        \Reg(T)=\tilde\cO\!\Big(\frac{d\sqrt T}{\kappa}+\frac{\Delta_h d^2}{\kappa^2\alpha^2\underline{s}^{\,2}}\Big).
    \]
\end{lemm}

\begin{proof}
    If $\Ncons=0$ there is nothing to prove. Otherwise $\Ncons$ belongs to the set in \eqref{eq:envelope}, so $4\alpha\underline{s}\,\Ncons<G(\Ncons)=H(\Ncons)\sqrt{\Ncons}$, whence $\sqrt{\Ncons}<H(\Ncons)/(4\alpha\underline{s})$ and \eqref{eq:ncons-explicit} follows by squaring and substituting the definition of $H$.

    For the orders, \eqref{eq:beta} gives $\bar\beta_t=\tilde\cO(\sqrt d/\kappa)$, so $H(t)=\tilde\cO(d/\kappa)$ and $G(T)=H(T)\sqrt T=\tilde\cO(d\sqrt T/\kappa)$. Substituting $H(\Ncons)=\tilde\cO(d/\kappa)$ into \eqref{eq:ncons-explicit} gives $\Ncons=\tilde\cO\big(d^2/(\kappa^2\alpha^2\underline{s}^{\,2})\big)$, where the logarithmic factors involving $\Ncons$ are themselves $\tilde\cO\big(\log(d/(\kappa\alpha\underline{s}))\big)$ by back-substitution. Combining with Theorem~\ref{thm:linear} gives the regret bound. The bound in Theorem~\ref{thm:linear} is the exact nonasymptotic statement; the display above is an order statement only.
\end{proof}

\subsection{Equivalence of latent and preference regret}
\label{app:equivalence}

\begin{prop}
\label{prop:equivalence}
    Define the preference regret of a displayed pair as $R^{\mathrm{BT}}_t:=s_t(\x_t^\star)-\tfrac12\big(s_t(\x_t)+s_t(\y_t)\big)$. Then for every pair,
    \[
        \kappa\,R_t\;\le\;R^{\mathrm{BT}}_t\;\le\;\tfrac14 R_t .
    \]
\end{prop}

\begin{proof}
    With $g_\x=f(\x_t^\star)-f(\x_t)$ and $g_\y=f(\x_t^\star)-f(\y_t)$, both in $[0,B]$ by Assumption~\ref{asmp:bounded}, we have $R_t=\tfrac12(g_\x+g_\y)$ and, using $s_t(\x_t^\star)=\sigma(0)=\tfrac12$,
    \[
        R^{\mathrm{BT}}_t=\tfrac12\Big[\big(\tfrac12-\sigma(-g_\x)\big)+\big(\tfrac12-\sigma(-g_\y)\big)\Big].
    \]
    For $g\in[0,B]$, $\;\tfrac12-\sigma(-g)=\int_{-g}^{0}\sigma'(u)\dif u$, and $\sigma'$ on $[-B,0]$ satisfies $\kappa\le\sigma'\le\tfrac14$; hence $\kappa g\le\tfrac12-\sigma(-g)\le\tfrac g4$. Apply this to $g_\x$ and $g_\y$, add, and divide by two.
\end{proof}

\subsection{Failure of Direct CLUCB Certification under Preference Feedback}
\label{app:obstruction}

We examine a direct transfer of CLUCB's cumulative safety test
\citep{conservative_linear_bandits} to preference feedback.
The result concerns the use of this test with confidence sets constructed
from comparisons. Known baseline utilities permit alternative,
baseline-centered certificates, which we discuss below.

Throughout this subsection, the learner is additionally supplied with the
exact baseline utilities
\[
    \mu_t^0 := f(\x_{b_t}).
\]
These are distinct from the baseline preference scores used in the main
formulation. Providing this additional information makes the comparison
with CLUCB explicit.

\begin{defn}[\textbf{Direct cumulative CLUCB certification}]
\label{def:avg-test}
    Let $\mathcal A_t$ and $\mathcal B_t$ denote the admitted candidate
    rounds and fallback rounds through time $t$, respectively, and write
    \[
        M_t := |\mathcal A_t|,
        \qquad
        N_t := |\mathcal B_t|,
        \qquad
        \mathbf{m}_i := \tfrac12(\x_i+\y_i).
    \]
    At round $t$, let
    $(\tilde\x_t,\tilde\y_t)\in\cS_t\times\cS_t$ be an arbitrary
    proposed pair, with
    \[
        \widetilde{\mathbf{m}}_t
        := \tfrac12(\tilde\x_t+\tilde\y_t).
    \]
    Construct the confidence set
    \begin{align}
        \mathcal K_t
        :=
        \left\{
            \theta\in\Theta:
            \|\theta-\hat\theta_t\|_{V_t}\le\beta_t
        \right\},
        \label{eq:obs-confidence}
    \end{align}
    where $\hat\theta_t$, $V_t$, and $\beta_t$ are defined by
    \eqref{eq:design}--\eqref{eq:beta}, using the observed comparisons.
    The algorithm admits the candidate only if
    \begin{align}
        L_t
        &:=
        \inf_{\theta\in\mathcal K_t}
        \left\langle
            \theta,\,
            \sum_{i\in\mathcal A_{t-1}}\mathbf{m}_i
            +\widetilde{\mathbf{m}}_t
        \right\rangle,
        \label{eq:obs-cumulative-lcb}
        \\
        L_t+\sum_{i\in\mathcal B_{t-1}}\mu_i^0
        &\ge
        (1-\alpha)\sum_{i=1}^{t}\mu_i^0.
        \label{eq:obs-cumulative-test}
    \end{align}
    Otherwise, it displays a pair containing $\x_{b_t}$ and credits
    $\mu_t^0$ for that round in subsequent applications of
    \eqref{eq:obs-cumulative-test}.
    Candidate pairs and fallback partners may be selected by arbitrary
    nonanticipating rules.

    The baseline utilities are credited separately in
    \eqref{eq:obs-cumulative-test}; the confidence set
    \eqref{eq:obs-confidence} is not intersected with the equalities
    $\langle\theta,\x_{b_i}\rangle=\mu_i^0$.
\end{defn}

This definition retains CLUCB's use of the current confidence set to
certify cumulative candidate performance, while replacing scalar-reward
regression by the comparison estimator.
The baseline self-duel $(\x_{b_t},\x_{b_t})$ reproduces the usual baseline
fallback and justifies crediting $\mu_t^0$.
We also allow other baseline-containing fallbacks: on the instance
constructed below, their average utility is always at least $\mu_t^0$,
so the same credit is valid.
On a general instance, such a credit for a non-self fallback would
require a separate justification.

\begin{prop}[\textbf{Failure of the direct cumulative transfer}]
\label{prop:obstruction}
    Fix $\lambda>0$ and $\alpha,\delta\in(0,1)$, and write
    $p:=1-\alpha$.
    There is a two-dimensional instance with $K=2$, satisfying the
    itemwise i.i.d.\ context assumption,
    Assumptions~\ref{asmp:BTL-model}, \ref{asmp:bounded},
    \ref{asmp:gapBounds}, and \ref{asmp:linear}, and the normalization
    $f(\cdot)\in[0,1]$, with the following properties.

    Every possible displayed pair satisfies both the cardinal
    average-utility constraint and the score constraint
    \eqref{eq:performance}.
    Nevertheless, every algorithm satisfying
    Definition~\ref{def:avg-test} obeys, on its confidence event $\cE$,
    \begin{align}
        N_T\ge\frac{p}{2}T
        \qquad\text{for every }T\ge1.
        \label{eq:obs-fallback-lower}
    \end{align}
    Moreover, for every $\eta\in(0,1)$ and every fixed integer
    \[
        T\ge\frac{32}{p^2}\log\frac1\eta,
    \]
    \begin{align}
        \mathbb P\left\{
            \Reg(T)\ge\frac{p^3}{512}T
        \right\}
        \ge 1-\delta-\eta.
        \label{eq:obs-regret-lower}
    \end{align}
    If all fallbacks are baseline self-duels, the denominator in
    \eqref{eq:obs-regret-lower} improves from $512$ to $256$.

    On the same instance, \MaxInPwBC{} satisfies, with probability at
    least $1-\delta$, simultaneously for every $T\ge1$,
    \[
        N_T\le\Ncons,
        \qquad
        \Reg(T)\le G(T)+\frac{p}{8}\Ncons
        =\tilde\cO(\sqrt T),
    \]
    where the rate treats the instance parameters and
    $\alpha,\delta,\lambda$ as fixed.
\end{prop}

\begin{proof}
    \emph{The instance.}
    Set
    \[
        g:=\frac p4,
        \qquad
        \theta^\star:=
        \begin{pmatrix}
            1/2\\
            g
        \end{pmatrix},
        \qquad
        \cD:=
        \left\{
            (1,u)^\top:u\in[0,1]
        \right\}.
    \]
    Independently over rounds and items, draw
    \[
        U_{t,1},U_{t,2}
        \sim\operatorname{Unif}[0,1],
        \qquad
        \x_{t,i}:=(1,U_{t,i})^\top.
    \]
    Thus each item is sampled independently from the same distribution
    on $\cD$, and the two items are distinct almost surely.

    Let the baseline select the item with the smaller second
    coordinate, breaking ties arbitrarily.
    Since $g>0$, the optimal item has the larger second coordinate.
    Consequently,
    \begin{align*}
        f(\x_{t,i})
        &=\frac12+gU_{t,i},
        \\
        \mu_t^0
        &=\frac12+g\min\{U_{t,1},U_{t,2}\},
        \\
        \Delta_{t,b_t}
        &=gQ_t,
        \qquad
        Q_t:=|U_{t,1}-U_{t,2}|.
    \end{align*}
    In particular,
    \[
        \frac12\le\mu_t^0\le\frac12+g,
        \qquad
        0\le\Delta_{t,b_t}\le g,
        \qquad
        \frac12\le f(\x)\le\frac34
        \quad\text{for every }\x\in\cD.
    \]
    The boundedness and linearity assumptions hold with
    $S=1$, $L=\sqrt2$, and $B=2LS=2\sqrt2$.
    The baseline-gap assumption holds with
    $\Delta_l=0$ and $\Delta_h=g$.
    Generate comparison outcomes according to the BTL model with
    parameter $\theta^\star$.
    The exact baseline utility is supplied as stipulated above;
    the baseline score is
    \[
        s_t(\x_{b_t})=\sigma(-gQ_t).
    \]

    Because the baseline is a least-utility item in $\cS_t$, every
    pair $(\x,\y)\in\cS_t\times\cS_t$ satisfies
    \begin{align}
        \frac{f(\x)+f(\y)}2
        &\ge\mu_t^0,
        &
        \max\{s_t(\x),s_t(\y)\}
        &\ge s_t(\x_{b_t}).
        \label{eq:obs-all-pairs-safe}
    \end{align}
    Summing these inequalities proves the asserted safety properties
    at every prefix.
    It also justifies the fallback credit for every possible
    baseline-containing pair on this instance.

    \emph{The confidence set retains a lower utility offset.}
    Every comparison vector has zero first coordinate.
    Hence, regardless of the algorithm's choices,
    \[
        V_t=
        \begin{pmatrix}
            \lambda&0\\
            0&\lambda+W_t
        \end{pmatrix},
        \qquad
        W_t:=\sum_{i<t}(\z_i)_2^2\ge0.
    \]
    Every possible pair average is $\mathbf{m}=(1,v)^\top$ with
    $v\in[0,1]$, so
    \begin{align}
        \|\mathbf{m}\|_{V_t^{-1}}^2
        =
        \frac1\lambda+\frac{v^2}{\lambda+W_t}
        \ge\frac1\lambda.
        \label{eq:obs-width}
    \end{align}

    The likelihood term in \eqref{eq:design} does not depend on the
    first coordinate of $\theta$.
    Setting that coordinate to zero preserves feasibility in the
    Euclidean ball $\Theta$ and reduces the regularization penalty
    unless the coordinate was already zero.
    Therefore,
    \[
        \hat\theta_t=(0,\hat g_t)^\top
        \qquad\text{for every }t.
    \]
    Consider the alternative parameter
    \[
        \theta^\circ:=(0,g)^\top.
    \]
    On the event $\cE$ of Lemma~\ref{lem:ellipsoid},
    \begin{align*}
        \|\theta^\circ-\hat\theta_t\|_{V_t}^2
        &=(\lambda+W_t)(g-\hat g_t)^2
        \\
        &\le
        \frac{\lambda}{4}
        +(\lambda+W_t)(g-\hat g_t)^2
        \\
        &=\|\theta^\star-\hat\theta_t\|_{V_t}^2
        \le\beta_t^2.
    \end{align*}
    Since $\|\theta^\circ\|_2=g<1$, this proves
    $\theta^\circ\in\mathcal K_t$ for every $t$ on $\cE$.
    Moreover,
    \[
        0\le\langle\theta^\circ,\x\rangle\le g
        \qquad\text{for every }\x\in\cD.
    \]
    Thus the inclusion also survives intersecting $\mathcal K_t$
    with the utility-normalization constraints
    $0\le\langle\theta,\x\rangle\le1$ for all $\x\in\cD$.

    For every possible pair average $\mathbf{m}=(1,v)^\top$,
    \[
        \langle\theta^\circ,\mathbf{m}\rangle=gv\le g.
    \]
    Consequently, on $\cE$,
    \begin{align}
        L_t
        &\le
        \left\langle
            \theta^\circ,\,
            \sum_{i\in\mathcal A_{t-1}}\mathbf{m}_i
            +\widetilde{\mathbf{m}}_t
        \right\rangle
        \le g(M_{t-1}+1).
        \label{eq:obs-mean}
    \end{align}
    This bound holds for arbitrary proposed pairs, including
    optimal self-duels.

    \emph{A constant fraction of rounds must be fallbacks.}
    At an admitted round $t$, we have
    $M_t=M_{t-1}+1$ and $N_t=N_{t-1}$.
    Combining \eqref{eq:obs-cumulative-test},
    \eqref{eq:obs-mean}, and the bounds on $\mu_i^0$ gives
    \begin{align}
        gM_t+\left(\frac12+g\right)N_t
        \ge\frac p2t.
        \label{eq:obs-count-invariant}
    \end{align}
    This inequality also persists through a fallback round:
    its left side increases by $\frac12+g$, while its right side
    increases by $p/2$, and
    \[
        \frac12+g-\frac p2
        =\frac\alpha2+g>0.
    \]
    Since both sides are zero at time zero,
    \eqref{eq:obs-count-invariant} holds for every $t$ on $\cE$.

    Substituting $M_t=t-N_t$ and $g=p/4$ yields
    \[
        \frac p4t+\frac12N_t\ge\frac p2t,
    \]
    which proves \eqref{eq:obs-fallback-lower}.
    Lemma~\ref{lem:ellipsoid} gives
    $\mathbb P(\cE)\ge1-\delta$.

    \emph{A linear number of fallbacks implies linear regret here.}
    The baseline gaps are not uniformly bounded away from zero,
    so we must account for the possibility that the algorithm
    selects small-gap rounds for fallback.

    Set $\varepsilon:=p/8$ and define
    \[
        H_T:=
        \sum_{t=1}^{T}
        \mathbf{1}\{Q_t<\varepsilon\}.
    \]
    The random variables $Q_t$ are i.i.d., with
    \[
        \mathbb P\{Q_t<\varepsilon\}
        =2\varepsilon-\varepsilon^2
        \le2\varepsilon
        =\frac p4.
    \]
    Hoeffding's inequality therefore gives, for each fixed $T$,
    \[
        \mathbb P\left\{
            H_T\le
            \frac p4T+
            \sqrt{\frac T2\log\frac1\eta}
        \right\}
        \ge1-\eta.
    \]
    If $T\ge32p^{-2}\log(1/\eta)$, the square-root term is at
    most $pT/8$.
    On this event and $\cE$, the number of fallback rounds with
    $Q_t\ge\varepsilon$ is at least
    \[
        N_T-H_T
        \ge
        \frac p2T-\frac{3p}{8}T
        =\frac p8T.
    \]

    Any pair containing the baseline incurs regret at least
    $\Delta_{t,b_t}/2=gQ_t/2$.
    Since all other rounds have nonnegative regret, we obtain
    \[
        \Reg(T)
        \ge
        \frac p8T\cdot\frac{g\varepsilon}{2}
        =
        \frac{p^3}{512}T.
    \]
    A union bound proves \eqref{eq:obs-regret-lower};
    independence between $\cE$ and the context-counting event
    is not required.
    A baseline self-duel incurs regret $gQ_t$, which doubles
    the lower bound.

    \emph{Performance of \MaxInPwBC{}.}
    The same instance satisfies every assumption of
    Theorem~\ref{thm:linear}, with
    \[
        \Delta_h=g=\frac p4,
        \qquad
        \underline{s}=\sigma(-g)>0.
    \]
    Applying that theorem to the trajectory of \MaxInPwBC{} gives,
    with probability at least $1-\delta$, simultaneously for
    every $T\ge1$,
    \[
        N_T\le\Ncons,
        \qquad
        \Reg(T)
        \le
        G(T)+\frac g2\Ncons
        =
        G(T)+\frac p8\Ncons.
    \]
    For fixed parameters, $G(T)=O(\sqrt T\log(eT))=o(T)$ and
    $\Ncons<\infty$ by \eqref{eq:envelope}.
    This proves the stated sublinear regret bound without
    requiring linear growth of $W_t$ or pointwise convergence
    of $c_t$.
\end{proof}

\begin{remark}[\textbf{Scope of the separation}]
\label{rem:obs-scope}
    Proposition~\ref{prop:obstruction} concerns the specific
    certification rule in Definition~\ref{def:avg-test}.
    The parameters $\theta^\star$ and $\theta^\circ$ induce
    identical comparison probabilities, but they do not agree
    with the same supplied absolute baseline utilities.
    The latter information is not imposed on the confidence set
    \eqref{eq:obs-confidence}.

    Known baseline utilities permit the identity
    \[
        f(\mathbf{m}_i)
        =
        \mu_i^0+
        \langle\theta^\star,\mathbf{m}_i-\x_{b_i}\rangle.
    \]
    Consequently, for baseline self-duel fallbacks, a valid
    alternative cumulative test is
    \begin{align}
        \alpha\sum_{i=1}^{t}\mu_i^0
        +
        \inf_{\theta\in\mathcal K_t}
        \left\langle
            \theta,\,
            \sum_{i\in\mathcal A_{t-1}}
                (\mathbf{m}_i-\x_{b_i})
            +\widetilde{\mathbf{m}}_t-\x_{b_t}
        \right\rangle
        \ge0.
        \label{eq:obs-centered-test}
    \end{align}
    On $\cE$, the left side lower-bounds the cumulative cardinal
    surplus after admitting the candidate.
    The same conclusion holds for other fallbacks whose average
    utilities are certified to be at least their baseline utilities.

    On the constructed instance, every centered vector in
    \eqref{eq:obs-centered-test} has zero first coordinate, so
    the common utility component disappears.
    Incorporating the known-baseline equalities into the confidence
    set is another possible modification.
    The proposition therefore establishes a limitation of the
    direct cumulative transfer, rather than an impossibility of
    conservative cardinal certification under preference feedback.
\end{remark}

\begin{remark}[\textbf{Score certification through comparisons}]
\label{rem:obs-score-certification}
    Scores can also be certified directly from pairwise differences.
    On $\cE$, for every $\y\in\cS_t$,
    \begin{align}
        s_t(\y)
        \ge
        \sigma\!\left(
            -\max_{\x\in\cS_t}
            \left\{
                \langle\hat\theta_t,\x-\y\rangle
                +\beta_t\|\x-\y\|_{V_t^{-1}}
            \right\}
        \right).
        \label{eq:obs-direct-score-bound}
    \end{align}
    Indeed, each expression inside the maximum upper-bounds
    $f(\x)-f(\y)$, and
    \[
        s_t(\y)
        =
        \sigma\!\left(
            -\max_{\x\in\cS_t}\{f(\x)-f(\y)\}
        \right).
    \]
    No estimate of an absolute utility is needed.

    Our construction uses the uniform optimality-gap certificate
    $c_t$ from Lemma~\ref{lem:certificate} and converts it into
    a score guarantee through Lemma~\ref{lem:score-certificate}.
    The candidate selection and informative fallback ensure,
    by \eqref{eq:width-transfer}, that
    \[
        c_t
        \le
        4\bar\beta_t
        \min\{1,\|\z_t\|_{V_t^{-1}}\},
        \qquad
        \bar\beta_t:=\max\{B/2,\beta_t\}.
    \]
    Thus its cumulative certification error is bounded by
    $G(T)$ along the comparisons actually observed.
    This connection between certification and information
    acquisition yields the finite fallback bound in
    Theorem~\ref{thm:linear}.
\end{remark}

\subsection{Relaxing knowledge of baseline scores}
\label{app:baseline-score-bounds}

The baseline constraint~\eqref{eq:performance} depends only on
utility differences and is therefore invariant to additive
shifts of $f$.
Indeed, for every $a\in\R$,
\[
    \sigma\big((f(\x)+a)-(f(\x_t^\star)+a)\big)
    =
    \sigma\big(f(\x)-f(\x_t^\star)\big).
\]
Consequently, exact knowledge of $s_t(\x_{b_t})$ is an
optional source of information, rather than a requirement
for resolving an additive utility offset.
This subsection proves that valid lower and upper bounds on
the baseline score suffice.

\paragraph{Available baseline-score bounds.}
Retain Assumptions~\ref{asmp:BTL-model},
\ref{asmp:bounded}, \ref{asmp:gapBounds}, and
\ref{asmp:linear}, but replace the exact-score input to
Algorithm~\ref{alg:cmaxinp} by two bounds $L_t,U_t$.
These bounds must be available before the round-$t$ pair
is selected and satisfy, on every realization,
\begin{equation}
    \underline{s}\le L_t\le U_t\le\tfrac12,
    \qquad
    \underline{s}=\sigma(-\Delta_h)>0.
    \label{eq:baseline-bound-range}
\end{equation}
The BTL assumption is understood conditionally on the full
pre-feedback information, including any side information
used to construct these bounds.

For some $\delta_b\in[0,1)$, assume simultaneous coverage:
\begin{equation}
    \P(\cE_b)\ge 1-\delta_b,
    \qquad
    \cE_b:=
    \bigcap_{t\ge1}
    \left\{
        L_t\le s_t(\x_{b_t})\le U_t
    \right\}.
    \label{eq:baseline-coverage}
\end{equation}
Deterministically valid bounds correspond to $\delta_b=0$.
No independence between $\cE_b$ and the model-confidence
event is required.

\paragraph{Modified safety budget.}
Compute $\hat\theta_t$, $\beta_t$, $\cC_t$,
$(\tilde\x_t,\tilde\y_t)$, $w_t$, and $c_t$
using~\eqref{eq:design}--\eqref{eq:candidate}, and define
the fallback partner by~\eqref{eq:fallback}.
Initialize $\Lambda_0=0$ and set
\begin{equation}
    q_t:=
    \frac12-\frac{c_t}{4}-(1-\alpha)U_t.
    \label{eq:baseline-bound-candidate-credit}
\end{equation}
The displayed pair and budget update are
\begin{align}
    (\x_t,\y_t)
    &=
    \begin{cases}
        (\tilde\x_t,\tilde\y_t),
        & \Lambda_{t-1}+q_t\ge0,\\
        (\x_{b_t},\y_t^{\mathrm{fb}}),
        & \Lambda_{t-1}+q_t<0,
    \end{cases}
    \label{eq:baseline-bound-pair}\\
    \Lambda_t
    &=
    \begin{cases}
        \Lambda_{t-1}+q_t,
        & \Lambda_{t-1}+q_t\ge0,\\
        \Lambda_{t-1}+\alpha L_t,
        & \Lambda_{t-1}+q_t<0.
    \end{cases}
    \label{eq:baseline-bound-budget}
\end{align}
After displaying the pair, observe $o_t$, set
$\z_t=\x_t-\y_t$, and update the design matrix by
$V_{t+1}=V_t+\z_t\z_t^\top$.
The estimator continues to use all displayed comparisons.

The upper bound $U_t$ appears in the subtracted baseline
term of the candidate test, whereas the lower bound $L_t$
determines the guaranteed credit from a fallback round.
The unknown quantity $s_t(\x_{b_t})$ is not used by
this decision rule.

Let $\cA_T,\cB_T$ denote the candidate and fallback rounds
of this variant, and write
$M_T=|\cA_T|$ and $N_T=|\cB_T|$.
Use $\beta_t$ from~\eqref{eq:beta}, with model-confidence
parameter $\delta$, and $G$ and $\Ncons$
from~\eqref{eq:envelope}.

\begin{theo}[\textbf{Guarantees with baseline-score bounds}]
\label{thm:baseline-score-bounds}
Suppose the assumptions above hold, with
$\alpha,\delta\in(0,1)$, $\lambda>0$, and
$\delta+\delta_b<1$.
Run the variant
\eqref{eq:baseline-bound-candidate-credit}--
\eqref{eq:baseline-bound-budget}.
With probability at least $1-\delta-\delta_b$,
simultaneously for every $T\ge1$,
\begin{align}
    \sum_{t=1}^T
    \max\{s_t(\x_t),s_t(\y_t)\}
    &\ge
    (1-\alpha)\sum_{t=1}^T s_t(\x_{b_t}),
    \label{eq:baseline-bound-safety}\\
    N_T
    &\le \min\{T,\Ncons\},
    \label{eq:baseline-bound-count}\\
    \Reg(T)
    &\le
    G(T)+\frac{\Delta_h}{2}\Ncons.
    \label{eq:baseline-bound-regret}
\end{align}
In particular, deterministically valid bounds preserve
the exact guarantees and confidence level of
Theorem~\ref{thm:linear}.
\end{theo}

\begin{proof}
We first establish the relevant model-confidence event.
The proof of Lemma~\ref{lem:ellipsoid} applies to any
adaptively selected comparison vectors that are measurable
before their feedback is observed.
The modified safety test preserves this property.
Hence
\[
    \P(\cE)\ge1-\delta,
    \qquad
    \cE:=
    \left\{
        \|\hat\theta_t-\theta^\star\|_{V_t}
        \le\beta_t
        \text{ for every }t\ge1
    \right\}.
\]
The union bound gives
\begin{equation}
    \P(\cE\cap\cE_b)
    \ge1-\delta-\delta_b.
    \label{eq:baseline-bound-joint-event}
\end{equation}
We prove all three conclusions on this single event.

\paragraph{Certificates and their cumulative envelope.}
On $\cE$, Lemma~\ref{lem:certificate} gives
\begin{equation}
    f(\x_t^\star)-f(\y)\le c_t
    \qquad
    \text{for every }\y\in\cC_t.
    \label{eq:baseline-bound-item-certificate}
\end{equation}
In particular, the candidate pair has instantaneous
regret at most $c_t$.

Write
\[
    v_t:=\|\z_t\|_{V_t^{-1}},
    \qquad
    \bar\beta_t:=\max\{B/2,\beta_t\},
\]
where $\z_t$ is the comparison vector actually displayed.
On a candidate round, $w_t=v_t$.
On a fallback round, the triangle inequality and the
definition of $\y_t^{\mathrm{fb}}$ give
\begin{align*}
    w_t
    &=
    \|\tilde\x_t-\tilde\y_t\|_{V_t^{-1}}\\
    &\le
    \|\tilde\x_t-\x_{b_t}\|_{V_t^{-1}}
    +
    \|\x_{b_t}-\tilde\y_t\|_{V_t^{-1}}\\
    &\le
    2\|\x_{b_t}-\y_t^{\mathrm{fb}}\|_{V_t^{-1}}
    =
    2v_t.
\end{align*}
Thus, on either type of round,
\begin{equation}
    c_t
    \le
    \min\{B,4\beta_t v_t\}
    \le
    4\bar\beta_t\min\{1,v_t\}.
    \label{eq:baseline-bound-pointwise-envelope}
\end{equation}
For the last inequality, if $v_t\le1$, use
$c_t\le4\beta_t v_t$.
If $v_t>1$, use $c_t\le B\le2\bar\beta_t$.

The matrix determinant lemma gives
\[
    \det(V_{t+1})
    =
    \det(V_t)(1+v_t^2).
\]
Since $\min\{1,u\}\le2\log(1+u)$ for every $u\ge0$,
\begin{align}
    \sum_{t=1}^T\min\{1,v_t^2\}
    &\le
    2\sum_{t=1}^T\log(1+v_t^2)\notag\\
    &=
    2\log\frac{\det(V_{T+1})}{\det(V_1)}\notag\\
    &\le
    2d\log\left(
        1+\frac{4L^2T}{\lambda d}
    \right).
    \label{eq:baseline-bound-potential}
\end{align}
The final inequality follows from the
arithmetic--geometric mean inequality applied to the
eigenvalues of $V_{T+1}$, together with
\[
    \operatorname{tr}(V_{T+1})
    =
    \lambda d+\sum_{t=1}^T\|\z_t\|_2^2
    \le \lambda d+4L^2T.
\]

Because $\bar\beta_t$ is nondecreasing,
Cauchy--Schwarz and
$\min\{1,v\}^2=\min\{1,v^2\}$ yield
\begin{align}
    \sum_{t=1}^T c_t
    &\le
    4\bar\beta_T
    \sum_{t=1}^T\min\{1,v_t\}\notag\\
    &\le
    4\bar\beta_T
    \sqrt{
        T\sum_{t=1}^T\min\{1,v_t^2\}
    }\notag\\
    &\le
    4\bar\beta_T
    \sqrt{
        2dT\log\left(
            1+\frac{4L^2T}{\lambda d}
        \right)
    }
    =
    G(T).
    \label{eq:baseline-bound-cumulative-envelope}
\end{align}
This bound includes $c_t$ on fallback rounds, even though
it is the certificate of the rejected candidate.
The width-transfer inequality controls that certificate
through the comparison vector actually displayed.

\paragraph{Safety.}
Define the true cumulative surplus by
\begin{equation}
    \Sigma_t:=
    \sum_{i=1}^t
    \left[
        \max\{s_i(\x_i),s_i(\y_i)\}
        -(1-\alpha)s_i(\x_{b_i})
    \right],
    \qquad
    \Sigma_0=0.
    \label{eq:baseline-bound-true-surplus}
\end{equation}
We show that
\begin{equation}
    \Sigma_t\ge\Lambda_t\ge0
    \qquad
    \text{for every }t\ge0.
    \label{eq:baseline-bound-invariant}
\end{equation}

On a candidate round, Lemmas~\ref{lem:certificate}
and~\ref{lem:score-certificate} imply
\[
    \max\{s_t(\x_t),s_t(\y_t)\}
    \ge\frac12-\frac{c_t}{4}.
\]
Since $s_t(\x_{b_t})\le U_t$ on $\cE_b$,
\begin{align*}
    \Sigma_t-\Sigma_{t-1}
    &\ge
    \frac12-\frac{c_t}{4}
    -(1-\alpha)s_t(\x_{b_t})\\
    &\ge
    \frac12-\frac{c_t}{4}
    -(1-\alpha)U_t\\
    &=q_t
    =\Lambda_t-\Lambda_{t-1}.
\end{align*}
The acceptance test ensures $\Lambda_t\ge0$.

On a fallback round, the displayed pair contains
$\x_{b_t}$, so
\begin{align*}
    \Sigma_t-\Sigma_{t-1}
    &\ge
    s_t(\x_{b_t})-(1-\alpha)s_t(\x_{b_t})\\
    &=\alpha s_t(\x_{b_t})\\
    &\ge\alpha L_t
    =\Lambda_t-\Lambda_{t-1}.
\end{align*}
Moreover, $L_t\ge\underline{s}>0$, so a fallback
preserves nonnegativity of the budget.
Starting from $\Sigma_0=\Lambda_0=0$, induction proves
\eqref{eq:baseline-bound-invariant}.
Expanding $\Sigma_T\ge0$ gives
\eqref{eq:baseline-bound-safety}.

\paragraph{Fallback count.}
Expanding the modified budget gives
\begin{equation}
    \Lambda_t
    =
    \sum_{i\in\cA_t}
    \left[
        \frac12-\frac{c_i}{4}-(1-\alpha)U_i
    \right]
    +
    \sum_{i\in\cB_t}\alpha L_i.
    \label{eq:baseline-bound-budget-expansion}
\end{equation}
Let $\tau$ be a fallback round.
By the rejection condition,
\[
    \Lambda_{\tau-1}
    +\frac12-\frac{c_\tau}{4}
    -(1-\alpha)U_\tau
    <0.
\]
Substituting
\eqref{eq:baseline-bound-budget-expansion}
and rearranging yields
\begin{align}
    &\sum_{i\in\cA_{\tau-1}}
    \left[
        \frac12-(1-\alpha)U_i
    \right]
    +
    \sum_{i\in\cB_{\tau-1}}\alpha L_i
    +
    \frac12-(1-\alpha)U_\tau
    \notag\\
    &\hspace{2em}<
    \frac14
    \left(
        \sum_{i\in\cA_{\tau-1}}c_i+c_\tau
    \right).
    \label{eq:baseline-bound-rejection}
\end{align}
The endpoint restrictions
\eqref{eq:baseline-bound-range} imply
\[
    \frac12-(1-\alpha)U_i
    \ge\frac{\alpha}{2},
    \qquad
    \alpha L_i\ge\alpha\underline{s}.
\]
Since all certificates are nonnegative,
\eqref{eq:baseline-bound-cumulative-envelope} gives
\[
    \sum_{i\in\cA_{\tau-1}}c_i+c_\tau
    \le
    \sum_{i=1}^{\tau}c_i
    \le G(\tau).
\]
Therefore,
\begin{equation}
    \frac{\alpha}{2}M_{\tau-1}
    +\alpha\underline{s}N_{\tau-1}
    +\frac{\alpha}{2}
    <
    \frac14G(\tau).
    \label{eq:baseline-bound-count-intermediate}
\end{equation}
Using $\underline{s}\le\tfrac12$ and
$M_{\tau-1}+N_{\tau-1}=\tau-1$, the left-hand side
is at least $\alpha\underline{s}\tau$.
Hence every fallback time satisfies
\[
    4\alpha\underline{s}\tau<G(\tau).
\]
By the definition of $\Ncons$ in~\eqref{eq:envelope},
every fallback time is at most $\Ncons$.
Because fallback times are distinct positive integers,
\[
    N_T\le\min\{T,\Ncons\}.
\]

For completeness, $\Ncons$ is finite.
For fixed problem parameters and $\delta$,
\eqref{eq:beta} gives
$\bar\beta_t=O(\sqrt{\log(et)})$, and consequently
\[
    G(t)=O\big(\sqrt{t}\log(et)\big)=o(t).
\]
Since $\alpha\underline{s}>0$, the inequality
$4\alpha\underline{s}t<G(t)$ can hold only for
bounded $t$.

\paragraph{Regret.}
On a candidate round,
Lemma~\ref{lem:certificate} gives $R_t\le c_t$.
On a fallback round,
$\y_t^{\mathrm{fb}}\in\cC_t$, so
\eqref{eq:baseline-bound-item-certificate} and
Assumption~\ref{asmp:gapBounds} give
\begin{align*}
    R_t
    &=
    \frac12
    \left[
        f(\x_t^\star)-f(\x_{b_t})
        +
        f(\x_t^\star)-f(\y_t^{\mathrm{fb}})
    \right]\\
    &\le\frac12(\Delta_h+c_t).
\end{align*}
It follows that
\begin{align*}
    \Reg(T)
    &\le
    \sum_{t\in\cA_T}c_t
    +
    \frac12\sum_{t\in\cB_T}c_t
    +
    \frac{\Delta_h}{2}N_T\\
    &\le
    \sum_{t=1}^T c_t
    +
    \frac{\Delta_h}{2}N_T\\
    &\le
    G(T)+\frac{\Delta_h}{2}\Ncons.
\end{align*}
All conclusions hold simultaneously for every $T$
on $\cE\cap\cE_b$.
Equation~\eqref{eq:baseline-bound-joint-event}
therefore proves the stated probability guarantee.
\end{proof}

\paragraph{A variant requiring no baseline-score estimates.}
Assumption~\ref{asmp:gapBounds} and the definition of
the preference score give the deterministic bounds
\[
    \underline{s}\le s_t(\x_{b_t})\le\tfrac12.
\]
Thus one may set
\begin{equation}
    L_t=\underline{s},
    \qquad
    U_t=\tfrac12
    \qquad
    \text{for every }t.
    \label{eq:baseline-bound-deterministic-choice}
\end{equation}
In this case, the candidate is accepted precisely when
\[
    \Lambda_{t-1}
    +\frac{\alpha}{2}-\frac{c_t}{4}
    \ge0,
\]
and the budget increments are
\begin{equation}
    \Lambda_t-\Lambda_{t-1}
    =
    \begin{cases}
        \displaystyle
        \frac{\alpha}{2}-\frac{c_t}{4},
        & t\in\cA_T,\\[5pt]
        \alpha\underline{s},
        & t\in\cB_T,
    \end{cases}
    \qquad t\le T.
    \label{eq:baseline-bound-no-estimation-budget}
\end{equation}
Here $\cE_b$ holds surely, so $\delta_b=0$ and
Theorem~\ref{thm:baseline-score-bounds} gives the same
exact bounds and confidence level as
Theorem~\ref{thm:linear}.
The learner needs the baseline item and the known
gap bound $\Delta_h$, but no exact or estimated
baseline score.

If the lower gap bound $\Delta_l$ is available,
the tighter deterministic upper bound
$U_t=\sigma(-\Delta_l)$ may instead be used.
It satisfies $U_t\le\tfrac12$ and therefore preserves
the same guarantees.

\paragraph{Statistical bounds and confidence accounting.}
Statistical intervals must cover the actual
context-dependent scores $s_t(\x_{b_t})$
simultaneously over time.
A confidence interval for a marginal average win rate
does not by itself establish
\eqref{eq:baseline-coverage}.

For example, suppose a procedure produces intervals
$[\widetilde L_t,\widetilde U_t]$ before the
round-$t$ decision and guarantees, under the adaptive
interaction protocol,
\[
    \P\left(
        s_t(\x_{b_t})
        \notin[\widetilde L_t,\widetilde U_t]
    \right)
    \le\delta_{b,t},
    \qquad
    \sum_{t=1}^{\infty}\delta_{b,t}\le\delta_b.
\]
The union bound then gives simultaneous coverage
with probability at least $1-\delta_b$.
One possible allocation is
\[
    \delta_{b,t}
    =
    \frac{6\delta_b}{\pi^2t^2}.
\]
Alternatively, a procedure may provide a simultaneous
confidence sequence directly.

To enforce~\eqref{eq:baseline-bound-range}, intersect
each supplied interval with
$[\underline{s},\tfrac12]$.
If the intersection is empty, use
$[\underline{s},\tfrac12]$ instead.
Whenever the supplied interval contains the true score,
its intersection with this known range is nonempty
and still contains that score.
Thus this operation preserves the simultaneous
coverage guarantee while ensuring that the algorithm
is well-defined on every realization.

The model-confidence and baseline-coverage failure
probabilities add:
\[
    \P(\cE\cap\cE_b)
    \ge1-\delta-\delta_b.
\]
For a desired total failure probability
$\delta_{\mathrm{tot}}\in(0,1)$, one may use
$\delta=\delta_b=\delta_{\mathrm{tot}}/2$.
In that case, $\beta_t$, $G$, and $\Ncons$ are evaluated
using the model-confidence parameter
$\delta_{\mathrm{tot}}/2$.
For the deterministic choice
\eqref{eq:baseline-bound-deterministic-choice},
no split is necessary: take
$\delta=\delta_{\mathrm{tot}}$ and $\delta_b=0$.

\section{Analysis for General Utility Functions}
\label{app:general-proofs}

This appendix proves Theorem~\ref{thm:general}, constructs a finite-class oracle, and treats baseline-score bounds. Throughout, $B>0$, $K\ge2$, and $\kappa=\sigma(B)(1-\sigma(B))$. All contexts and baseline information are available before pair sampling, and the BTL law is conditioned on all information available before the comparison outcome. The analysis allows arbitrary predictable context sets, and therefore includes the i.i.d. model of the main paper. No absolute baseline utility is used.

\subsection{Adapting C-SquareCB to preference feedback}
\label{app:gf-csquarecb}

This subsection explains the changes to certification and fallback sampling relative to C-SquareCB \citep{deb2025conservative}. The concentration bounds and proof of Theorem~\ref{thm:general} are given in the subsequent subsections.

\textbf{The quantity controlled by the oracle changes.} C-SquareCB observes noisy cardinal costs and uses a square-loss oracle to calibrate cumulative predictions of those costs; its safety test adds a correction for their cumulative estimation error. In our setting, the conditional mean of the label is $\sigma(f(\x_t)-f(\y_t))$. Accordingly, the logistic oracle controls the margin error
\begin{align}
\varepsilon_t(\x_t,\y_t)
=\big[\hat f_t(\x_t)-\hat f_t(\y_t)\big]-\big[f(\x_t)-f(\y_t)\big],
\label{eq:gf-comparison-error}
\end{align}
as established in Appendix~\ref{app:gf-concentration}. This does not control individual utility errors: adding the same constant to every prediction leaves both the logistic loss and \eqref{eq:gf-comparison-error} unchanged. Thus C-SquareCB's cardinal calibration cannot be invoked with these predictions. Scores themselves are shift invariant, however, so this observation is not an impossibility of score certification.

\textbf{Pairwise fit alone does not calibrate plug-in scores.} A separate example shows why simply substituting estimated scores into a cardinal test also requires an additional argument. For any utility function $g$, define its plug-in score on a context set $\cS$ by
\[
s^g(\x):=\sigma\!\left(g(\x)-\max_{\y\in\cS}g(\y)\right).
\]
Consider a realized context set containing four distinct items $\cS=\{\x^{(1)},\x^{(2)},\x_b,\x^\star\}$ and two utility functions with values
\[
\begin{array}{c|cccc}
 & \x^{(1)} & \x^{(2)} & \x_b & \x^\star\\ \hline
f & 0 & 0 & \log 2 & \log 3\\
g & \log 2 & \log 2 & \log 2 & \log 3
\end{array}
\]
Both have utility range at most $B=\log 3$, the same best item, and the same baseline score $s^f(\x_b)=s^g(\x_b)=2/5$. On the displayed pair $(\x^{(1)},\x^{(2)})$, both margins are zero. Hence the predicted and true comparison probabilities coincide, the margin error vanishes, and the two logistic losses agree for either outcome. Nevertheless,
\begin{align}
\max_{j\in\{1,2\}}s^f(\x^{(j)})=\frac14,
\qquad
\max_{j\in\{1,2\}}s^g(\x^{(j)})=\frac25.
\label{eq:gf-score-calibration-example}
\end{align}
For $\alpha=1/4$, the required baseline fraction is $3/10$, so the plug-in score exceeds this threshold while the true pair score falls below it. Repeating this pair $n$ times leaves the observed margin errors zero but gives a cumulative score overestimate of $3n/20$. This is an algebraic failure of calibration from observed margin errors alone, conditional on the displayed pairs; it is not a regret lower bound for an IGW algorithm or a claim about every possible comparison. Indeed, comparing $\x^{(1)}$ against the baseline distinguishes the models: its win probabilities are $1/3$ under $f$ and $1/2$ under $g$. The example shows why a score guarantee must also account for which comparisons the learner collects.

\textbf{The replacement couples certification to the sampling rule.} Algorithm~\ref{alg:gf-igw} uses $p_t\otimes p_t$ on candidate rounds and a baseline paired with one draw from $p_t$ on fallback rounds. The branch is selected before sampling, so these are the conditional sampling laws. Let $\mathbb E_t$ condition on all information before pair sampling, including the branch decision, and put $v_t:=\mathbb E_t[\varepsilon_t(\x_t,\y_t)^2]$. For $r_t^\circ$ defined in \eqref{eq:gf-auxiliary}, Lemma~\ref{lem:gf-anchor} gives on both branches
\begin{align}
\mathbb E_t[r_t^\circ]\le\frac{2K}{\gamma_t}+\frac{\gamma_t}{4}v_t.
\label{eq:gf-comparison-transfer}
\end{align}
On candidates, the fixed-partner inequality is averaged over an independent IGW partner; on fallbacks, the partner is fixed to the baseline. In both cases its error term is that of the comparison actually observed. The centering argument underlying this step appears in the proof of Lemma~12 of \citet{sekhari2023aurora}; Appendix~\ref{app:gf-anchor} verifies the constants used here. Our construction applies it to both conservative branches and updates the oracle with both types of comparison. A fallback may still be a self-duel; informative feedback is not guaranteed on every individual round.

\textbf{Cumulative regret replaces cardinal calibration.} Lemma~\ref{lem:gf-epoch} and Appendix~\ref{app:gf-anytime} combine \eqref{eq:gf-comparison-transfer} with the oracle guarantee to establish the envelope \eqref{eq:gf-global-auxiliary}. The score inequality \eqref{eq:gf-score-lower} then yields
\begin{align}
\sum_{i=1}^t\left[\max\{s_i(\x_i),s_i(\y_i)\}-(1-\alpha)s_i^b\right]
\ge Q_t-\frac14\sum_{i\in\mathcal A_t}R_i,
\label{eq:gf-comparison-budget}
\end{align}
where $\mathcal A_t$ contains the candidate rounds and $Q_t$ is the credit total in \eqref{eq:gf-credit-identity}. This explains the charge $G_{\mathrm{IGW}}(t)/4$ when considering a candidate. Including the baseline makes each fallback safe directly, while \eqref{eq:gf-comparison-transfer} also controls its sampled partner's gap. Appendix~\ref{app:gf-safety} proves prefix safety, bounds the fallback count, and derives the regret guarantee. These are changes to the certificate and its connection to observed feedback; they do not assert that a regret certificate is the only possible way to enforce the score constraint.

\subsection{The fixed-partner IGW inequality}
\label{app:gf-anchor}

We first establish an elementary IGW inequality, including its constants.

\begin{lemm}[\textbf{Single-action IGW inequality}]
\label{lem:gf-single}
Let $v,\hat v\in\mathbb R^K$, $a^\star\in\argmax_a v_a$, $\hat a\in\argmax_a\hat v_a$, and $\gamma>0$. Define $p_a=(K+\gamma(\hat v_{\hat a}-\hat v_a))^{-1}$ for $a\ne\hat a$ and put the remaining mass on $\hat a$. Then $p$ is a probability distribution, $p_{\hat a}\ge1/K$, and
\begin{align}
\sum_a p_a(v_{a^\star}-v_a)
\le\frac{2K}{\gamma}+\frac{\gamma}{4}\sum_a p_a(\hat v_a-v_a)^2.
\label{eq:gf-single}
\end{align}
\end{lemm}

\begin{proof}
Every action other than $\hat a$ has probability in $(0,1/K]$, so the residual probability is at least $1/K$. Let $e_a=\hat v_a-v_a$, $u_a=\hat v_{\hat a}-\hat v_a\ge0$, $\eta=\gamma/4$, and $p_\star=p_{a^\star}>0$. Subtracting the quadratic term in \eqref{eq:gf-single}, and using $e-\eta e^2\le1/(4\eta)$ for actions other than $a^\star$, gives
\begin{align}
D&:=\sum_a p_a(v_{a^\star}-v_a)-\eta\sum_a p_a e_a^2
\nonumber\\
&\le\sum_{a\ne a^\star}p_a(\hat v_{a^\star}-\hat v_a)
-(1-p_\star)e_{a^\star}-\eta p_\star e_{a^\star}^2
+\frac{1-p_\star}{4\eta}.
\label{eq:gf-single-square}
\end{align}
Completing the square yields $-(1-p_\star)e-\eta p_\star e^2\le(1-p_\star)^2/(4\eta p_\star)$. Hence
\begin{align*}
D
&\le\sum_a p_a u_a-u_{a^\star}+\frac{1-p_\star}{\gamma p_\star}
\le\sum_a p_a u_a-u_{a^\star}+\frac{1}{\gamma p_\star}.
\end{align*}
Also,
\[
\sum_a p_a u_a=\sum_{a\ne\hat a}\frac{u_a}{K+\gamma u_a}\le\frac{K-1}{\gamma}.
\]
If $a^\star=\hat a$, then $u_{a^\star}=0$ and $p_\star\ge1/K$. Otherwise, $p_\star^{-1}=K+\gamma u_{a^\star}$. In either case, $-u_{a^\star}+1/(\gamma p_\star)\le K/\gamma$. Thus $D\le(2K-1)/\gamma\le2K/\gamma$, proving the claim. The argument requires no restriction on the absolute levels of $v$ or $\hat v$.
\end{proof}

\begin{proof}[Proof of Lemma~\ref{lem:gf-anchor}]
Fix $\mathbf u\in\cS_t$. Apply Lemma~\ref{lem:gf-single} with $v_a=f(\x_{t,a})-f(\mathbf u)$ and $\hat v_a=\hat f_t(\x_{t,a})-\hat f_t(\mathbf u)$. Subtracting the fixed partner's value leaves both the estimated maximizer and every estimated gap unchanged, so the resulting distribution is exactly $p_t$. The true regret is $v_{a^\star}-v_a=\Delta_t(\x_{t,a})$, and $\hat v_a-v_a=\varepsilon_t(\x_{t,a},\mathbf u)$. Equation~\eqref{eq:gf-single} gives the bound in the lemma.
\end{proof}

\subsection{From logistic regret to predictable margin error}
\label{app:gf-concentration}

Fix an epoch of length $n$ and failure budget $\delta'\in(0,1)$, and reindex its rounds by $m=1,\ldots,n$. The oracle is fresh at the start of this epoch. All probability bounds below also hold conditional on the history before the epoch. Let $\mathbb E_m$ denote conditional expectation given all information immediately before sampling the current pair, including the current context, baseline information, oracle prediction, and candidate/fallback decision. The label law, after conditioning additionally on the sampled pair, is BTL. Define
\begin{align}
\Lambda&=\log\frac{3n}{\delta'},
& E&=\frac{8\mathcal R_{\log}(n)}{\kappa}+\frac{5\Lambda}{\kappa^2},
\label{eq:gf-local-parameters}\\
d_m&=f(\x_m)-f(\y_m),
&\hat d_m&=\hat f_m(\x_m)-\hat f_m(\y_m),
\nonumber\\
e_m&=\hat d_m-d_m,
&X_m&=e_m^2,
\qquad v_m=\mathbb E_m[X_m].
\label{eq:gf-local-errors}
\end{align}
The true and predicted margins lie in $[-B,B]$, so $0\le X_m\le4B^2$. A useful elementary bound is
\begin{align}
\kappa^{-1}=2+2\cosh B\ge4+B^2\ge4B,
\qquad B\kappa\le\tfrac14.
\label{eq:gf-Bkappa}
\end{align}

\begin{lemm}[\textbf{Uniform predictable error bound}]
\label{lem:gf-errors}
For any predictable comparison-sampling rule, with probability at least $1-2\delta'/3$, simultaneously for every $j\le n$,
\begin{align}
\sum_{m=1}^j v_m\le E.
\label{eq:gf-error-budget}
\end{align}
\end{lemm}

\begin{proof}
Let $\Phi(d)=\log(1+e^d)$, $\pi_m=\sigma(d_m)$, and $\xi_m=e_m(o_m-\pi_m)$. The logistic loss difference decomposes as
\begin{align}
\ell_m(\hat f_m)-\ell_m(f)
&=D_m-\xi_m,
& D_m&:=\Phi(\hat d_m)-\Phi(d_m)-\Phi'(d_m)e_m.
\label{eq:gf-loss-decomposition}
\end{align}
Since the segment between $d_m$ and $\hat d_m$ lies in $[-B,B]$ and $\Phi''=\sigma'\ge\kappa$ there,
\[
D_m=e_m^2\int_0^1(1-u)\Phi''(d_m+ue_m)\,du
\ge\frac{\kappa}{2}X_m.
\]
Realizability and the prefix oracle guarantee therefore imply, for every $j\le n$,
\begin{align}
\frac{\kappa}{2}\sum_{m=1}^jX_m
\le\mathcal R_{\log}(n)+\sum_{m=1}^j\xi_m.
\label{eq:gf-curvature-oracle}
\end{align}

\emph{Label noise.} Conditional on the information before the label, $o_m$ is Bernoulli with mean $\pi_m$, and the two possible values of $\xi_m$ lie in an interval of length $|e_m|$. Conditional Hoeffding's lemma gives
\[
\mathbb E\!\left[\left.
\exp\!\left(a\xi_m-\frac{a^2X_m}{8}\right)
\right|\text{information before the label}\right]\le1
\qquad(a\in\mathbb R).
\]
Taking a further conditional expectation before pair sampling shows that $\exp(a\sum_{m\le j}\xi_m-a^2\sum_{m\le j}X_m/8)$ is a nonnegative supermartingale starting at one. Ville's inequality with $a=2\kappa$ implies, outside an event of probability at most $e^{-\Lambda}$, simultaneously for every $j\le n$,
\[
\sum_{m=1}^j\xi_m
\le\frac{\kappa}{4}\sum_{m=1}^jX_m+\frac{\Lambda}{2\kappa}.
\]
Combining this with \eqref{eq:gf-curvature-oracle} gives
\begin{align}
\sum_{m=1}^jX_m
\le\frac{4\mathcal R_{\log}(n)}{\kappa}
+\frac{2\Lambda}{\kappa^2}.
\label{eq:gf-sampled-error}
\end{align}

\emph{Pair sampling.} Put $c=(1-e^{-1})/(4B^2)$. Convexity of $x\mapsto e^{-x/(4B^2)}$ implies the chord bound $e^{-x/(4B^2)}\le1-cx$ on $[0,4B^2]$. Thus
\[
\mathbb E_m[e^{-X_m/(4B^2)}]\le1-cv_m\le e^{-cv_m}.
\]
It follows that $\exp(c\sum_{m\le j}v_m-\sum_{m\le j}X_m/(4B^2))$ is another nonnegative supermartingale starting at one. A second application of Ville's inequality gives, with additional failure probability at most $e^{-\Lambda}$,
\begin{align}
\sum_{m=1}^jv_m
&\le\frac{1}{1-e^{-1}}\sum_{m=1}^jX_m
+\frac{4B^2\Lambda}{1-e^{-1}}
\nonumber\\
&\le2\sum_{m=1}^jX_m+7B^2\Lambda.
\label{eq:gf-error-transfer}
\end{align}
On the intersection of the two events, \eqref{eq:gf-sampled-error}, \eqref{eq:gf-error-transfer}, and \eqref{eq:gf-Bkappa} yield
\[
\sum_{m=1}^jv_m
\le\frac{8\mathcal R_{\log}(n)}{\kappa}
+\frac{4\Lambda}{\kappa^2}+7B^2\Lambda
\le\frac{8\mathcal R_{\log}(n)}{\kappa}
+\frac{5\Lambda}{\kappa^2}=E.
\]
The total failure probability is at most $2e^{-\Lambda}=2\delta'/(3n)\le2\delta'/3$. No independence between the chosen pairs across rounds is required.
\end{proof}

\subsection{An epoch envelope including fallback comparisons}
\label{app:gf-epoch}

At each round, the branch is chosen before sampling. A candidate displays two independent draws from $p_m$; a fallback displays the baseline and one draw from $p_m$. Define
\begin{align}
r_m^\circ=
\begin{cases}
R_m,&\text{candidate},\\
f(\x_m^\star)-f(\y_m),&\text{fallback}.
\end{cases}
\label{eq:gf-auxiliary}
\end{align}
Thus $0\le r_m^\circ\le B$. This quantity is used only in the analysis.

\begin{lemm}[\textbf{Uniform epoch envelope}]
\label{lem:gf-epoch}
With $E,\Lambda$ from \eqref{eq:gf-local-parameters}, use $\gamma_m=\sqrt{8Km/E}$ and define
\begin{align}
C(n,\delta'):=\frac{3}{\sqrt2}\sqrt{KE}+B\sqrt{2\Lambda}.
\label{eq:gf-local-C}
\end{align}
For any predictable sequence of candidate/fallback decisions, with probability at least $1-\delta'$, simultaneously for every $j\le n$,
\begin{align}
\sum_{m=1}^j r_m^\circ\le C(n,\delta')\sqrt j.
\label{eq:gf-epoch-bound}
\end{align}
\end{lemm}

\begin{proof}
On a candidate round, independence and identical distribution of the two draws give $\mathbb E_m[R_m]=\mathbb E_{a\sim p_m}\Delta_m(\x_{m,a})$. Average Lemma~\ref{lem:gf-anchor} over an independent partner from $p_m$ to obtain
\[
\mathbb E_m[r_m^\circ]
\le\frac{2K}{\gamma_m}+\frac{\gamma_m}{4}v_m.
\]
On a fallback round, apply the same lemma with the baseline as the fixed partner. The squared error of the observed comparison equals the squared error with the pair reversed, so this inequality again holds, now with $r_m^\circ=\Delta_m(\y_m)$. The error $v_m$ is always that of the comparison actually displayed.

On the event of Lemma~\ref{lem:gf-errors}, monotonicity of $\gamma_m$ and $\sum_{m=1}^j m^{-1/2}\le2\sqrt j$ imply
\begin{align}
\sum_{m=1}^j\mathbb E_m[r_m^\circ]
&\le2K\sum_{m=1}^j\frac1{\gamma_m}
+\frac{\gamma_j}{4}\sum_{m=1}^jv_m
\nonumber\\
&\le\sqrt{2KjE}+\frac1{\sqrt2}\sqrt{KjE}
=\frac3{\sqrt2}\sqrt{KjE}.
\label{eq:gf-conditional-regret}
\end{align}
The differences $r_m^\circ-\mathbb E_m[r_m^\circ]$ have conditional mean zero and absolute value at most $B$. For each $j$, Azuma's inequality bounds the probability that their sum exceeds $B\sqrt{2j\Lambda}$ by $e^{-\Lambda}$. A union bound over $j\le n$ gives additional failure probability at most $ne^{-\Lambda}=\delta'/3$. Combining with Lemma~\ref{lem:gf-errors} proves \eqref{eq:gf-epoch-bound}.
\end{proof}

\begin{lemm}[\textbf{Envelope constants}]
\label{lem:gf-constants}
The constant in \eqref{eq:gf-local-C} satisfies $C(n,\delta')\le\tfrac94\sqrt{KE}$.
\end{lemm}

\begin{proof}
Since $E\ge5\Lambda/\kappa^2$, $K\ge2$, and $B\kappa\le1/4$,
\[
\frac{B\sqrt{2\Lambda}}{\sqrt{KE}}
\le B\kappa\sqrt{\frac{2}{5K}}
\le\frac{1}{4\sqrt5}<\frac18.
\]
Use $3/\sqrt2+1/8<9/4$ in \eqref{eq:gf-local-C}.
\end{proof}

\subsection{Combining epochs and controlling growth}
\label{app:gf-anytime}

Epoch $e$ comprises rounds $2^e,\ldots,2^{e+1}-1$. Let $\mathcal E_e$ be its simultaneous event from Lemma~\ref{lem:gf-epoch}, with length $n_e=2^e$ and failure budget $\delta_e=\delta/[(e+1)(e+2)]$. The oracle is restarted, while the credit total is preserved. The conditional guarantees of that lemma apply despite dependence on all previous epochs. Since $\sum_{e\ge0}\delta_e=\delta$,
\begin{align}
\mathbb P(\mathcal E_{\mathrm{IGW}})\ge1-\delta,
\qquad
\mathcal E_{\mathrm{IGW}}:=\bigcap_{e\ge0}\mathcal E_e.
\label{eq:gf-good-event}
\end{align}
Summing the completed-epoch bounds and the current partial-epoch bound gives, on this event,
\begin{align}
\sum_{i=1}^t r_i^\circ\le G_{\mathrm{IGW}}(t)
\qquad\text{for every }t\ge1.
\label{eq:gf-global-auxiliary}
\end{align}
This explains why fallback labels enter the oracle: their errors control their partners' gaps through the same inequality as on candidate rounds.

To make the growth rate explicit, define for $t\ge1$
\begin{align}
\mathcal L(t)&:=\log\frac{3t(1+\log_2t)(2+\log_2t)}{\delta},
&\overline E(t)&:=\frac{8\mathcal R_{\log}(t)}{\kappa}
+\frac{5\mathcal L(t)}{\kappa^2},
\label{eq:gf-growth-functions}\\
c_G&:=\frac{9(2+\sqrt2)}4.
\nonumber
\end{align}
Both $E_e$ and $C_e$ are nondecreasing in $e$. For $e=e(t)$, we have $2^e\le t$, $E_e\le\overline E(t)$, and
\begin{align}
G_{\mathrm{IGW}}(t)
&\le C_e\left(\sum_{j<e}\sqrt{2^j}+\sqrt{m(t)}\right)
\nonumber\\
&\le(2+\sqrt2)C_e\sqrt t
\le c_G\sqrt{Kt\overline E(t)}.
\label{eq:gf-growth}
\end{align}
Here $\sum_{j<e}2^{j/2}\le(1+\sqrt2)2^{e/2}$, and the last step uses Lemma~\ref{lem:gf-constants}. The envelope is nondecreasing, including at epoch boundaries. Since $\mathcal R_{\log}(t)=o(t)$ and $\mathcal L(t)=O(\log(t/\delta)+\log\log(t+2))$, \eqref{eq:gf-growth} implies $G_{\mathrm{IGW}}(t)=o(t)$.

\subsection{Safety, fallback count, and regret}
\label{app:gf-safety}

Let $\mathcal A_t$ and $\mathcal B_t$ denote the candidate and fallback rounds through $t$, and let $N_t=|\mathcal B_t|$. For the uniform bound on the total number of fallbacks, define
\begin{align}
N_{\mathrm{IGW}}
:=\sup\bigl\{n\in\mathbb N_{\ge1}:4\alpha\underline{s}\,n<G_{\mathrm{IGW}}(n)\bigr\},
\label{eq:gf-total-fallback}
\end{align}
with $\sup\emptyset:=0$. This quantity is finite because $G_{\mathrm{IGW}}(n)=o(n)$. It depends on the problem parameters and the oracle-regret function, but not on the stopping horizon, and is not needed by the algorithm. The credit total satisfies
\begin{align}
Q_t=\sum_{i\in\mathcal A_t}\left[\tfrac12-(1-\alpha)s_i^b\right]
+\sum_{i\in\mathcal B_t}\alpha s_i^b.
\label{eq:gf-credit-identity}
\end{align}

\begin{proof}[Proof of Theorem~\ref{thm:general}]
Work on $\mathcal E_{\mathrm{IGW}}$ from \eqref{eq:gf-good-event}.

\emph{Safety.} Put $g_{t,\x}=f(\x_t^\star)-f(\x_t)$ and $g_{t,\y}=f(\x_t^\star)-f(\y_t)$. Because $R_t=(g_{t,\x}+g_{t,\y})/2$,
\begin{align}
\max\{s_t(\x_t),s_t(\y_t)\}
&=\sigma(-\min\{g_{t,\x},g_{t,\y}\})
\ge\sigma(-R_t)\ge\tfrac12-\tfrac14R_t,
\label{eq:gf-score-lower}
\end{align}
where the final inequality follows from $\sigma'\le1/4$. A fallback contains the baseline, so its true surplus is at least $\alpha s_t^b$. Therefore the cumulative true surplus satisfies
\begin{align}
S_t&:=\sum_{i=1}^t\left[\max\{s_i(\x_i),s_i(\y_i)\}-(1-\alpha)s_i^b\right]
\nonumber\\
&\ge Q_t-\tfrac14\sum_{i\in\mathcal A_t}R_i.
\label{eq:gf-surplus-lower}
\end{align}
Let $\tau_t=\max(\mathcal A_t\cup\{0\})$ be the last candidate time, using zero if there has been no candidate. Nonnegativity of $r_i^\circ$ and \eqref{eq:gf-global-auxiliary} give
\[
\sum_{i\in\mathcal A_t}R_i
\le\sum_{i=1}^{\tau_t}r_i^\circ
\le G_{\mathrm{IGW}}(\tau_t),
\]
where empty sums are zero and $G_{\mathrm{IGW}}(0)=0$. Thus $\Lambda_t:=Q_t-G_{\mathrm{IGW}}(\tau_t)/4$ is a lower certificate for $S_t$. Initially $\Lambda_0=0$. On a candidate round, $\tau_t=t$ and the acceptance test ensures $\Lambda_t\ge0$. On a fallback round, $\tau_t=\tau_{t-1}$ and $\Lambda_t=\Lambda_{t-1}+\alpha s_t^b\ge0$. This remains true across oracle restarts because the credits are preserved. Hence $S_t\ge0$ at every prefix.

\emph{Fallbacks.} Since $s_t^b\le1/2$, the candidate credit satisfies $h_t\ge\alpha/2\ge\alpha\underline{s}$. The fallback credit satisfies $\alpha s_t^b\ge\alpha\underline{s}$. Consequently, $Q_{t-1}\ge\alpha\underline{s}(t-1)$, and rejection at time $t$ implies
\begin{align}
\alpha\underline{s}\,t
\le Q_{t-1}+h_t<G_{\mathrm{IGW}}(t)/4.
\label{eq:gf-rejection}
\end{align}
Every fallback time therefore belongs to the finite set in \eqref{eq:gf-total-fallback}. It follows that $N_T\le\min\{T,N_{\mathrm{IGW}}\}$ for every $T$, and the total number of fallbacks is finite.

For an explicit bound in terms of the oracle regret at $T$, suppose $N_T>0$ and let $\tau=\max\mathcal B_T$ be the last fallback through time $T$. Combining \eqref{eq:gf-rejection} with \eqref{eq:gf-growth} yields
\[
4\alpha\underline{s}\,\tau<G_{\mathrm{IGW}}(\tau)
\le c_G\sqrt{K\tau\overline E(\tau)}.
\]
Squaring and dividing by $\tau>0$ gives
\[
\tau<
\frac{c_G^2K\overline E(\tau)}{16\alpha^2\underline{s}^{\,2}}
\le
\frac{c_G^2K\overline E(T)}{16\alpha^2\underline{s}^{\,2}},
\]
where the second inequality uses $\tau\le T$ and monotonicity of $\overline E$. Since $N_T\le\tau$, we obtain
\begin{align}
N_T\le\min\left\{
T,\,
N_{\mathrm{IGW}},\,
\frac{c_G^2K\overline E(T)}{16\alpha^2\underline{s}^{\,2}}
\right\}.
\label{eq:gf-explicit-fallbacks}
\end{align}
The same bound is immediate when $N_T=0$.

\emph{Regret.} On candidates, $R_t=r_t^\circ$. On fallbacks,
\[
R_t=\tfrac12\left[\Delta_{t,b_t}+\Delta_t(\y_t)\right]
\le\tfrac12(\Delta_h+r_t^\circ).
\]
Summing and using \eqref{eq:gf-global-auxiliary} gives
\begin{align}
\Reg(T)
&\le\sum_{t\in\mathcal A_T}r_t^\circ
+\tfrac12\sum_{t\in\mathcal B_T}r_t^\circ
+\tfrac12\Delta_hN_T
\nonumber\\
&\le G_{\mathrm{IGW}}(T)+\tfrac12\Delta_hN_T.
\label{eq:gf-regret-proof}
\end{align}
Equations~\eqref{eq:gf-growth} and \eqref{eq:gf-explicit-fallbacks} therefore imply the explicit bound
\begin{align}
\Reg(T)
\le c_G\sqrt{KT\overline E(T)}
+\frac{c_G^2\Delta_hK\overline E(T)}{32\alpha^2\underline{s}^{\,2}}.
\label{eq:gf-explicit-oracle-rate}
\end{align}
Substituting $\overline E(T)=8\mathcal R_{\log}(T)/\kappa+5\mathcal L(T)/\kappa^2$ from \eqref{eq:gf-growth-functions} and suppressing logarithmic factors in $T$ and $1/\delta$ gives
\[
\Reg(T)
=\tilde\cO\!\left(
\sqrt{KT\left[\frac{\mathcal R_{\log}(T)}{\kappa}+\frac{1}{\kappa^2}\right]}
+\frac{\Delta_hK}{\alpha^2\underline{s}^{\,2}}
\left[\frac{\mathcal R_{\log}(T)}{\kappa}+\frac{1}{\kappa^2}\right]
\right),
\]
which is \eqref{eq:gf-main}. In particular, $\overline E(T)=o(T)$ implies $\Reg(T)=o(T)$ directly from \eqref{eq:gf-explicit-oracle-rate}. We also retain the uniform bound
\begin{align}
\Reg(T)\le G_{\mathrm{IGW}}(T)+\tfrac12\Delta_hN_{\mathrm{IGW}},
\label{eq:gf-uniform-regret}
\end{align}
which is useful when a growth bound on $\mathcal R_{\log}$ allows an explicit horizon-independent estimate of $N_{\mathrm{IGW}}$. All statements hold simultaneously for every $T$ on the same event of probability at least $1-\delta$.
\end{proof}

\subsection{A finite-class oracle and explicit specialization}
\label{app:gf-finite}

Suppose $1\le|\mathcal F|<\infty$ and each function has range diameter at most $B$. Within each epoch initialize $w_1(g)=1/|\mathcal F|$, let $\eta=e^{-B}$, and predict and update according to
\begin{align}
\hat f_m(\x)&=\sum_{g\in\mathcal F}w_m(g)g(\x),
\label{eq:gf-ew-predict}\\
w_{m+1}(g)&=
\frac{w_m(g)e^{-\eta\ell_m(g)}}
{\sum_{h\in\mathcal F}w_m(h)e^{-\eta\ell_m(h)}}.
\label{eq:gf-ew-update}
\end{align}
The common weights produce one predicted utility function, and $|\hat f_m(\x)-\hat f_m(\y)|\le\sum_gw_m(g)B=B$.

\begin{prop}[\textbf{Finite-class logistic oracle}]
\label{prop:gf-finite-oracle}
The oracle in \eqref{eq:gf-ew-predict}--\eqref{eq:gf-ew-update} satisfies Assumption~\ref{asmp:gf-oracle} with $\mathcal R_{\log}(n)=e^B\log|\mathcal F|$ for every $n$.
\end{prop}

\begin{proof}
For $\ell(d,o)=\log(1+e^d)-od$,
\[
\frac{\partial_d^2\ell(d,o)}{(\partial_d\ell(d,o))^2}
=
\begin{cases}
e^d,&o=1,\\
e^{-d},&o=0,
\end{cases}
\qquad |d|\le B.
\]
This ratio is at least $e^{-B}=\eta$, so $d\mapsto e^{-\eta\ell(d,o)}$ is concave on $[-B,B]$. Because the predicted margin is the weighted average of the functions' margins, Jensen's inequality yields
\[
e^{-\eta\ell_m(\hat f_m)}\ge Z_m,
\qquad
Z_m:=\sum_gw_m(g)e^{-\eta\ell_m(g)}.
\]
The weight update gives $\prod_{m=1}^nZ_m=|\mathcal F|^{-1}\sum_g\exp(-\eta\sum_{m=1}^n\ell_m(g))$. Taking logarithms and comparing the sum with any one of its terms therefore gives
\[
\sum_{m=1}^n\ell_m(\hat f_m)
\le\min_{g\in\mathcal F}\sum_{m=1}^n\ell_m(g)
+\frac{\log|\mathcal F|}{\eta}.
\]
The bound is deterministic for every sequence and every prefix, and so permits adaptive contexts and comparisons. The range and timing requirements also hold by construction.
\end{proof}

For an explicit fallback bound, let
\begin{align}
a_0&:=\frac{c_G^2K}{16\alpha^2\underline s^{\,2}}
\left[
\frac{8e^B\log|\mathcal F|}{\kappa}
+\frac{5}{\kappa^2}\log\frac{12}{\delta}
\right],
\label{eq:gf-finite-a0}\\
a_1&:=\frac{15c_G^2K}{16\alpha^2\underline s^{\,2}\kappa^2}.
\label{eq:gf-finite-a1}
\end{align}
For $t\ge1$, $1+\log_2t\le2t$ and $2+\log_2t\le2t$, so $\mathcal L(t)\le\log(12/\delta)+3\log t$. Combining \eqref{eq:gf-growth} with the defining inequality for $N_{\mathrm{IGW}}$ shows that each integer in that set satisfies $t<a_0+a_1\log t$. Concavity of the logarithm gives $\log t\le t/(2a_1)+\log(2a_1)-1$, whence
\begin{align}
N_{\mathrm{IGW}}
\le2a_0+2a_1\log\bigl(\max\{e,2a_1\}\bigr).
\label{eq:gf-finite-count}
\end{align}
Combining \eqref{eq:gf-finite-count} with \eqref{eq:gf-uniform-regret}, and writing $A_{\mathcal F}:=e^B\log|\mathcal F|/\kappa+1/\kappa^2$, gives
\begin{align}
\Reg(T)
=\tilde\cO\!\left(
\sqrt{KA_{\mathcal F}T}
+\frac{\Delta_hKA_{\mathcal F}}{\alpha^2\underline s^{\,2}}
\right).
\label{eq:gf-finite-rate}
\end{align}
Here the suppressed factors are logarithmic in $T$, $1/\delta$, $K$, $A_{\mathcal F}$, and $1/(\alpha\underline s)$. The additive contribution uses \eqref{eq:gf-finite-count} and can be bounded independently of the stopping horizon. For $B=1$, $\kappa$ is a numerical constant and $A_{\mathcal F}=\cO(\log|\mathcal F|+1)$. Computing $p_t$ takes $\cO(K)$ arithmetic after the $K$ score evaluations; the oracle's own computational cost is separate.

\subsection{Relaxing baseline-score knowledge}
\label{app:gf-baseline-bounds}

Suppose $L_t,U_t$ are available before pair sampling and always satisfy $\underline s\le L_t\le U_t\le1/2$. Assume that, with probability at least $1-\delta_b$, they simultaneously cover the baseline scores at every round: $L_t\le s_t^b\le U_t$. Replace the candidate credit and fallback credit, respectively, by
\begin{align}
h_t=\tfrac12-(1-\alpha)U_t,
\qquad
\text{fallback credit}=\alpha L_t.
\label{eq:gf-bound-credits}
\end{align}
Use these credits in the same acceptance test and updates. On the coverage event, a candidate's true surplus is at least $h_t-R_t/4$, and a fallback's true surplus is at least $\alpha L_t$. Thus the same safety argument applies with these credits. Moreover, $h_t\ge\alpha/2$ and $\alpha L_t\ge\alpha\underline s$, so the fallback-count argument and all regret bounds are unchanged. The overall probability becomes at least $1-\delta-\delta_b$, without an independence requirement. Deterministic bounds have $\delta_b=0$; in particular, $(L_t,U_t)=(\underline s,1/2)$ removes exact baseline-score knowledge entirely while retaining the exact guarantees of Theorem~\ref{thm:general}.

\subsection{A known-horizon version without oracle restarts}
\label{app:gf-fixed-horizon}

If a horizon $T$ is given, use one oracle instance, the parameters in \eqref{eq:gf-local-parameters} with $(n,\delta')=(T,\delta)$, $\gamma_t=\sqrt{8Kt/E}$, and the envelope $G_T(t)=C(T,\delta)\sqrt t$. The same test and sampling rules then satisfy, with probability at least $1-\delta$, simultaneously for every $t\le T$, the max-score constraint and
\begin{align}
N_t&\le\min\left\{
t,\frac{C(T,\delta)^2}{16\alpha^2\underline s^{\,2}}
\right\},
\label{eq:gf-fixed-count}\\
\Reg(t)&\le C(T,\delta)\sqrt t+\tfrac12\Delta_hN_t.
\label{eq:gf-fixed-regret}
\end{align}
Indeed, Lemma~\ref{lem:gf-epoch} gives the required envelope on the entire run, and \eqref{eq:gf-rejection} forces every fallback time to be smaller than $C(T,\delta)^2/(16\alpha^2\underline s^{\,2})$. The safety and regret arguments are otherwise identical. Lemma~\ref{lem:gf-constants} yields the explicit terminal bound
\begin{align}
\Reg(T)\le\frac94\sqrt{KTE}
+\frac{81\Delta_hKE}{512\alpha^2\underline s^{\,2}}.
\label{eq:gf-fixed-explicit}
\end{align}
This version uses the known horizon in its parameters and avoids oracle restarts. The epoch construction additionally admits the uniform total-fallback bound \eqref{eq:gf-total-fallback}.

\end{document}